\documentclass[journal]{IEEEtai}

\usepackage[colorlinks,urlcolor=blue,linkcolor=blue,citecolor=blue]{hyperref}

\usepackage{color,array}

\usepackage{graphicx}

\usepackage[utf8]{inputenc}
\usepackage{xcolor}
\usepackage[ruled]{algorithm}  
\usepackage[noend]{algpseudocode}
\usepackage{amsfonts}
\usepackage{amsmath}
\usepackage{amssymb}
\usepackage{booktabs}
\usepackage{bm}
\usepackage{bbm}
\usepackage{url}
\usepackage{makecell}
\usepackage{soul}
\usepackage{tcolorbox}
\usepackage{hyperref}
\usepackage{todonotes}
\usepackage{xspace}
\usepackage{subfigure}
\usepackage{pifont}

\usepackage{color, colortbl}
\usepackage[first=0,last=9]{lcg}
	
\definecolor{Gray}{gray}{0.9}
\definecolor{LightCyan}{rgb}{0.88,1,1}

\usepackage{pifont}
\newcommand{\cmark}{\ding{51}}%
\newcommand{\xmark}{\ding{55}}%

\newtheorem{theorem}{Theorem}

\begin{document}


\title{Defending Against Poisoning Attacks in Federated Learning with Blockchain}

\author{Nanqing Dong, Zhipeng Wang, Jiahao Sun, Michael~Kampffmeyer, William Knottenbelt, and~Eric Xing,~\IEEEmembership{Fellow,~IEEE}
\thanks{The first two authors contributed equally to this work. This work was supported in part by FLock.io under the FLock Research Grant. (Corresponding author: Nanqing Dong.)}
\thanks{N.~Dong is with the Shanghai Artificial Intelligence Laboratory, Shanghai, 200232, China. (email: dongnanqing@pjlab.org.cn)}
\thanks{Z.~Wang and W.~Knottenbelt are with the Department of Computing, Imperial College London, London, SW7 2AZ, UK. (emails: zhipeng.wang20@imperial.ac.uk, w.knottenbelt@imperial.ac.uk)}
\thanks{J.~Sun is with FLock.io, London, WC2H 9JQ, UK. (email: sun@flock.io)}
\thanks{M.~Kampffmeyer is with the Department of Physics and Technology at UiT The Arctic University of Norway, 9019 Troms{\o}, Norway. (email: michael.c.kampffmeyer@uit.no)}
\thanks{E.~Xing is with the Machine Learning Department, School of Computer Science, Carnegie Mellon University, Pittsburgh, PA 15213, USA; and also with Mohamed bin Zayed University of Artificial Intelligence, Masdar City, Abu Dhabi, UAE (email: epxing@cs.cmu.edu)}
}

\markboth{Journal of IEEE Transactions on Artificial Intelligence, Vol. 00, No. 0, Month 2020}
{First A. Author \MakeLowercase{\textit{et al.}}: Bare Demo of IEEEtai.cls for IEEE Journals of IEEE Transactions on Artificial Intelligence}

\makeatletter
\DeclareRobustCommand\onedot{\futurelet\@let@token\@onedot}
\def\@onedot{\ifx\@let@token.\else.\null\fi\xspace}

\def\eg{\emph{e.g\onedot}} \def\Eg{\emph{E.g}\onedot}
\def\ie{\emph{i.e\onedot}} \def\Ie{\emph{I.e}\onedot}
\def\cf{\emph{c.f}\onedot} \def\Cf{\emph{C.f}\onedot}
\def\etc{\emph{etc}\onedot} \def\vs{\emph{vs}\onedot}
\def\wrt{w.r.t\onedot} \def\dof{d.o.f\onedot}
\def\iid{i.i.d\onedot} \def\etal{\emph{et al}\onedot}
\makeatother

\maketitle

\begin{abstract}
    In the era of deep learning, federated learning (FL) presents a promising approach that allows multi-institutional data owners, or clients, to collaboratively train machine learning models without compromising data privacy. However, most existing FL approaches rely on a centralized server for global model aggregation, leading to a single point of failure. This makes the system vulnerable to malicious attacks when dealing with dishonest clients.
    In this work, we address this problem by proposing a secure and reliable FL system based on blockchain and distributed ledger technology.
    Our system incorporates a peer-to-peer voting mechanism and a reward-and-slash mechanism, which are powered by on-chain smart contracts, to detect and deter malicious behaviors. Both theoretical and empirical analyses are presented to demonstrate the effectiveness of the proposed approach, showing that our framework is robust against malicious client-side behaviors.
\end{abstract}

\begin{IEEEImpStatement}
Federated learning has been a promising solution to utilize multi-site data while preserving users' privacy. 
Despite the success of integrating blockchain with federated learning to decentralize global model aggregation, the protection of this integration from clients with malicious intent in federated scenarios remains unclear. 
This paper presents the first formulation of this problem and the proposed stake-based aggregation mechanism shows robustness in detecting malicious behaviors. The results in this work not only pose a new research direction in federated learning, but can also benefit a wide variety of applications such as finance and healthcare.
\end{IEEEImpStatement}


\section{Introduction}
\IEEEPARstart{N}{owadays}, machine learning (ML), or more specifically, deep learning, has transformed a broad spectrum of industries, ranging from finance to healthcare. In current ML paradigms, training data are first collected and curated, and then ML models are optimized by minimizing certain loss criteria on the training data. A common underlying assumption in the learning environment is that the training data can be instantly accessed or easily distributed across computing nodes without communication constraints, \ie~data are \emph{centralized}. 

However, in a system with multiple \emph{clients} (\ie~data holders), to ensure data centralization, clients have to upload local data to a centralized device (\eg~a central server) to conduct the centralized training described above. Despite the success of centralized training in various deep learning applications~\cite{he2016deep,vaswani2017attention,mnih2015human}, there is growing concern about data privacy and security, especially when the local data held by the clients are private or contain sensitive information. Especially, to ensure data governance, strict data regulations have been established~\cite{gdpr,hipaa}. 

\begin{figure}[t]
  \centering
  \includegraphics[width =\columnwidth]{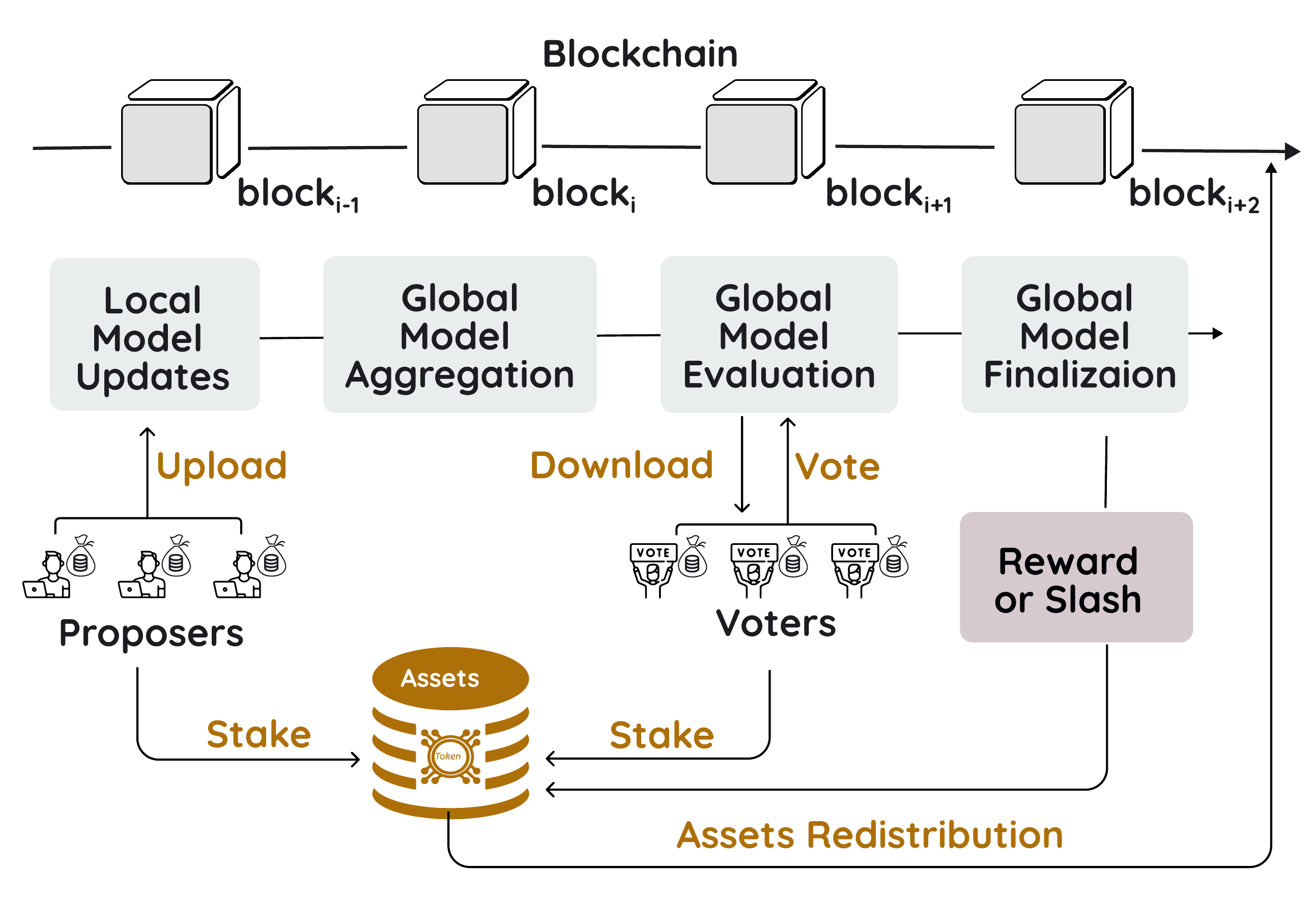}
  \caption{A stake-based aggregation mechanism for FL with blockchain. In each round, the proposers are randomly selected from the participating clients to perform local training and upload local updates to the blockchain. Then, voters download the aggregated local updates from the blockchain, perform local validation, and vote for acceptance or rejection. If the majority of voters vote for accepting the global aggregation, the global model will be updated, and the proposers and the voters who vote for acceptance will be rewarded. Conversely, if the majority of voters vote for rejection, the global model will not be updated, and the proposers and the voters who vote for acceptance will be slashed.}
  \label{fig:overview}
\end{figure}

To address the aforementioned concern, federated learning~(FL) has been proposed~\cite{mcmahan2017communication}. 
In a typical FL system, a central server~\cite{li2014communication} is responsible for aggregating and synchronizing model weights, while a set of clients manipulate multi-site data.
This facilitates data governance, as clients only exchange model weights or gradients with a central server instead of uploading local data to the central server, and has led to FL becoming a standardized solution to utilize multi-site data while preserving privacy.

Though FL perfectly implements \emph{data decentralization}, a trustworthy central server is required in the system. In such a system design, the central server in fact has privileges over clients, as the central server determines the global aggregation and synchronization. If the central server is compromised or manipulated by a malicious party, the clients are vulnerable if the central server intentionally distributes problematic model updates. This can potentially increase the cost of system management and maintenance. Towards avoiding this single point of failure, many efforts have been made to decentralize the central server, and one particularly promising solution is to use a blockchain as decentralized storage~\cite{qu2022blockchain}.

Originally proposed for cryptocurrencies, a blockchain is a distributed ledger that can record the state transition information among multiple parties~\cite{wood2014ethereum,chen2018machine}, without relying on a centralized server. Blockchain technology has gained widespread attention for its potential to revolutionize a variety of industries, such as finance~\cite{wood2014ethereum}, healthcare~\cite{soltanisehat2020technical}, and supply chain management~\cite{queiroz2020blockchain}. By leveraging the decentralized nature of the blockchain, FL can benefit from increased security, privacy, and efficiency, as well as reduced reliance on centralized servers~\cite{li2021blockchain}. Concretely, in FL with blockchain, each client participating in the learning process uploads their local model updates to the blockchain, where they are stored in \emph{blocks}, the metadata of a blockchain system. These blocks are then used to aggregate the local model updates into a global model, which can be downloaded by the clients. The use of blockchain \emph{smart contracts}~\cite{wood2014ethereum}, which are computer programs triggered by blockchain events, ensures that the global aggregation process is performed automatically and transparently, without the need for human intervention or centralized control.

Though integrating blockchain with existing FL systems can partially solve the threat to the central server, it cannot guarantee the quality of uploaded model updates from the clients. That is to say, blockchain-enabled FL systems are still vulnerable to client-side malicious attacks~\cite{bagdasaryan2020how}. In this work, we define malicious behaviors as actions that intentionally decrease the learning performance (\eg~accuracy and convergence) of the global model via poisoning attacks (such as data poisoning~\cite{tolpegin2020data} or model poisoning~\cite{bagdasaryan2020how}). Instead of hacking the central server, the attackers can sabotage the FL systems by manipulating the clients. This work focuses on defending against client-side poisoning attacks. One solution is to combine blockchain-enabled FL with cryptographic protocols, such as Fully Homomorphic Encryption (FHE)~\cite{miao2022privacy} and Secure Multi-Party Computation (SMPC)~\cite{kalapaaking2023blockchain}, to mitigate malicious behaviors from the client side. However, the adoption of these intricate cryptographic protocols introduces significant computational overhead for FL participants, thus impairing the system performance. Besides, the malicious clients can still attack the system without breaching the protocols. It is challenging to address malicious behaviors without substantially compromising the efficiency of a blockchain-based FL system.

We propose a generic framework that can integrate an FL system with a blockchain system and can defend against poisoning attacks without adopting complex cryptographic protocols. The proposed defense mechanism is motivated by \emph{proof-of-stake} (PoS)~\cite{bano2019sok}, a \emph{consensus mechanism} in blockchain, and \emph{The Resistance}~\cite{resistence}, a role-playing board game. PoS has an incentive mechanism that encourages honest behaviors by \emph{rewarding} it and punishes dishonest behaviors via \emph{slashing}. \emph{The Resistance}, on the other hand, has two mismatched competing parties, where the party with a larger size is denoted as the resistance force and the other party is denoted as the spies. In \emph{The Resistance}, there is a voting mechanism where, in each round, each player conducts independent reasoning and votes for a player, and the player with the highest votes will be deemed as a ``spy'' and kicked out of the game. The goal of the resistance force is to vote out all the spies while the spies aim to impersonate the resistance force and survive until the end. Based on these two concepts, this work proposes a novel majority-voting mechanism for global aggregation where each participating client independently validates the quality of aggregated local updates and votes for acceptance of the global update. The aggregation mechanism is stake-based where participating clients stake assets\footnote{In practice, the staked assets can be linked with cryptocurrency or real currency to increase the financial cost of malicious attacks.} or \emph{tokens} (a quantitative measurement of the asset, which can be used to indicate the trustworthiness of the client in our system) for their own actions. There are two types of actions, proposing (uploading local updates) and voting. If the majority vote is to accept the global aggregation, a proposer will be refunded with its staked tokens and a voter who votes for acceptance will not only be refunded but also be rewarded with the staked tokens from the voters who vote for rejection, and vice versa. The overall procedure of the stake-based aggregation mechanism is illustrated in Fig.~\ref{fig:overview}. To the best of our knowledge, this is the first work that integrates the majority voting and incentive mechanisms in the FL and blockchain literature. 

We evaluate the proposed framework on a practical financial problem, namely loan default prediction. We simulate the FL and blockchain environment for the Lending Club Kaggle challenge dataset and ChestX-ray14 dataset~\cite{wang2017chestx}, to conduct experiments in a controllable setting and to provide insights into the problem of interest. We empirically show that an FL system can maintain robust performance under malicious attacks by introducing the proposed stake-based aggregation mechanism.

The contributions of this work are summarized as follows:
\begin{IEEEenumerate}
    \item We formulate the problem of decentralized federated learning with blockchain in the presence of poisoning attacks. 
    \item For the first time, we introduce a novel stake-based aggregation mechanism designed to fortify federated learning systems against poisoning attacks. In contrast to prior solutions, our mechanism boasts the distinct advantage of seamless integration into any blockchain enabled with smart contracts, all without necessitating alterations to the foundational consensus structure of the underlying blockchain. This approach not only enhances security but also simplifies accessibility, rendering it a more user-friendly option for federated learning participants.
    \item We evaluate the robustness of the proposed framework in a simulated environment and provide initial empirical insights into the problem of interest. The findings show evidence that stake-based FL is an under-explored research problem with potential advantages compared with existing FL paradigms in terms of defending against poisoning attacks.
\end{IEEEenumerate}

The rest of the paper is organized as follows. Sec.~\ref{sec:related} reviews the related work on FL and blockchain computing. Sec.~\ref{sec:prob} formulates the problem of interest, defines the key concepts, and lists the necessary assumptions. Sec.~\ref{sec:method} presents the method and theoretical result. Sec.~\ref{sec:exp} and Sec.~\ref{sec:add} provide the experimental details on two different setups. Sec.~\ref{sec:con} concludes this work.

\begin{figure}[t]
  \centering
  \includegraphics[width =\columnwidth]
  {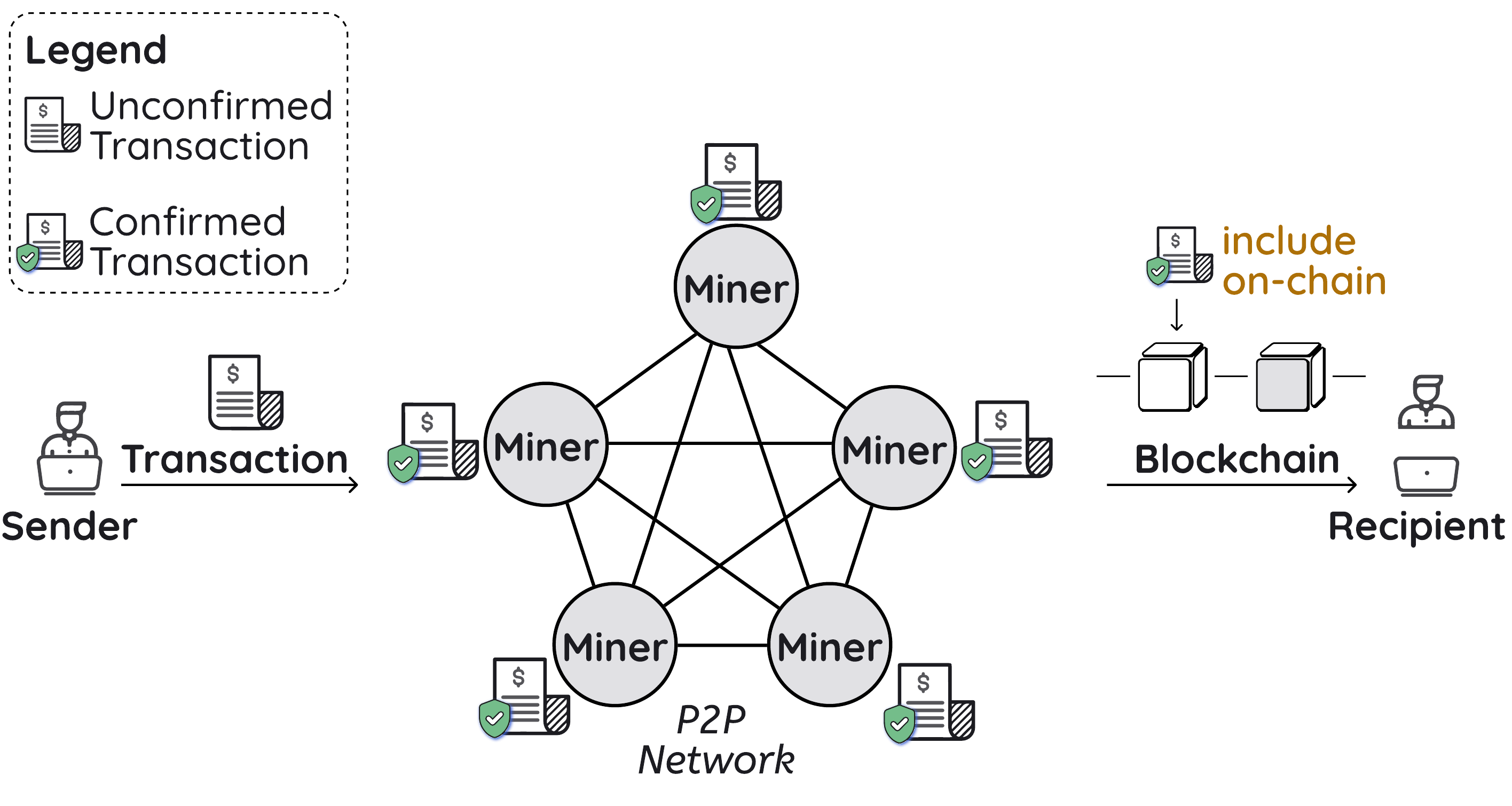}
  \caption{Blockchain workflow overview. The sender broadcasts the issued transaction to the P2P network, which will be confirmed by the miners. The confirmed transaction will be stored on a public blockchain and can be read by the recipient. Blockchain miners typically adopt a consensus mechanism to achieve an agreement on the state of the blockchain.}
  \label{fig:blockchain-overview}
\end{figure}

\begin{table*}[t]
    \centering
    \caption{Comparison of existing blockchain-based FL solutions.}
    \begin{tabular}{c|ccccc}
    \toprule
         Framework & Technique & Incentive Mechanism & Cryptocurrency Incentive & Blockchain Agnostic & Attack Mitigation\\ \hline
         \cite{miao2022privacy} &FHE &\xmark &\xmark &- &\cmark\\
         \cite{kalapaaking2023blockchain} &SMPC &\xmark &\xmark &-&\cmark\\
         \cite{yan2023privacy} &DP &\xmark &\xmark &-&\cmark\\
         \cite{cui2021security} &DP-GAN &\xmark &\xmark &-&\cmark\\
         \cite{toyoda2019mechanism} &incentive mechanism &\cmark &\cmark &- &\xmark\\
         \cite{zhang2021incentive} &reputation + reverse auction  &\cmark &\cmark &- &\xmark\\
         \cite{ma2022federated}& blockchain + reward only &\cmark & \cmark &\xmark &\cmark\\
          \cite{li2021blockchain}& committee consensus &\cmark & \cmark &\xmark &\cmark\\
          \cite{qu2022fl}& committee consensus &\cmark & \cmark &\xmark &\cmark\\
         \midrule
         Ours & majority vote + reward \& slash & \cmark & \cmark &\cmark &\cmark\\
    \bottomrule
    \end{tabular}
    \label{tab:existing-work}
\end{table*}

\section{Related Work}
\label{sec:related}
In this section, we review the recent progress on blockchain-based FL, and highlight the difference between the proposed method and existing studies.

\subsection{Blockchain}
Blockchains refer to distributed ledgers that operate on a global peer-to-peer (P2P) network, as exemplified by popular cryptocurrencies such as Bitcoin~\cite{bitcoin} and Ethereum~\cite{wood2014ethereum}. Users can freely join or leave the blockchain system, without a central authority in place to ensure common agreement on the distributed ledgers. Instead, users rely on consensus protocols~\cite{bano2019sok, garay2020sok}, such as proof-of-work (PoW) or PoS, to achieve agreement in a distributed setting. 

As shown in Fig.~\ref{fig:blockchain-overview}, a blockchain transaction typically involves a sender who transfers digital assets, such as cryptocurrencies, to a recipient. 
The sender authorizes the transaction with a digital signature combining transaction details and their private key. 
The transaction is then broadcasted over a P2P network to \emph{miners}, who are participants in the network responsible for verifying and adding new blocks of transactions to the blockchain. Miners validate and confirm the transaction using consensus protocols, to ensure that the transaction is legitimate and not a duplicate or fraudulent transaction. Once confirmed, the transaction is added to a block, which is then linked to the previous block using cryptographically hash functions~\cite{bonneau2015sok}, forming a chain of blocks (\ie, blockchain). The block is then propagated to all the participants in the network, creating a decentralized, immutable record of the transaction. The combination of cryptography and consensus protocols enhances the security, transparency, and decentralization of transactions, underscoring blockchain's potential across various applications~\cite{gatteschi2018blockchain,soltanisehat2020technical,queiroz2020blockchain,bao2020survey,reyna2018blockchain}.

Another key feature of blockchain technology is the use of smart contracts~\cite{wood2014ethereum}, which are quasi-Turing-complete programs that can be executed within a virtual machine.  When a transaction is initiated, a smart contract is typically used to encode the terms and conditions of the transaction, such as the amount, currency, and time of transfer. The smart contract is then stored on the blockchain network and executed automatically when the predefined conditions are met. 

\subsection{Federated Learning with Blockchain}
Traditional FL faces challenges~\cite{kairouz2021advances, wan2023data, zhang2024decentralized}, such as privacy and security concerns (\eg, poisoning attacks), unreliable communication, and difficulty in reaching a consensus among the parties. Blockchain, on the other hand, provides a decentralized, secure, and transparent platform for data storage and sharing. This makes the use of blockchain for FL a promising direction to potentially address privacy and security concerns by allowing parties to keep their data private while still contributing to the training process. Additionally, blockchain can provide a secure communication channel for FL participants and ensure the integrity of the FL process.

Current blockchain-based FL designs~\cite{zhang2024decentralized, zhu2023blockchain, qu2022blockchain, kim2019blockchained, weng2019deepchain} have been broadly used in diverse fields~\cite{nguyen2021federated,ali2021integration,li2022blockchain}. For example, Ma~\etal~\cite{ma2022federated} propose a blockchain-assisted decentralized FL (BLADE-FL) framework,
to prevent malicious clients from poisoning the learning process. Li~\etal~\cite{li2021blockchain} analyze the impact of lazy clients on the learning performance of BLADE-FL and propose optimization for minimizing the loss function. Cui~\etal~\cite{cui2021security} propose a blockchain-based
decentralized and asynchronous FL framework for anomaly detection in IoTs by using a model named DP-GAN. Qu~\etal~\cite{qu2022fl} introduce a committee-based blockchain consensus algorithm
for decentralized FL to prevent the single point of failure and poisoning attacks in FL.

Despite the potential benefits of combining FL with blockchain, several challenges remain. For instance, FL systems are still vulnerable to client-side malicious attacks~\cite{bagdasaryan2020how} and lack incentive-compatible mechanisms to motivate FL participants to behave honestly during the training process. For instance, Miao~\etal~\cite{miao2022privacy} leverage cosine similarity and blockchain to counteract poisoning attacks and penalize malicious clients. Their designs also rely on cryptographic techniques such as FHE. Kalapaaking~\etal~\cite{kalapaaking2023blockchain} adopt SMPC to enhance the security of blockchain-based FL which can mitigate poisoning attacks for healthcare systems.
Yan~\etal~\cite{yan2023privacy} integrate differential privacy (DP) techniques with blockchain to mitigate the issues of a single point of failure or untrusted aggregation caused by a malicious central server in privacy-preserving federated learning. However, using FHE, SMPC, or DP, instead of incorporating incentive mechanisms, may introduce an additional computation burden for clients or the server. 

Several incentive mechanisms~\cite{toyoda2019mechanism,zhang2021incentive} have recently been proposed to encourage participants and enhance model accuracy in blockchain-based FL. However, it remains unclear how to effectively utilize the blockchain infrastructure and leverage its inherent incentive mechanism (\ie~cryptocurrencies) to incentivize trustworthy FL behaviors and penalize malicious clients. Furthermore, to thwart potential malicious activities, such as poisoning attacks, existing blockchain-based FL solutions~\cite{qu2022blockchain, zhu2023blockchain, ma2022federated,li2021blockchain,qu2022fl} have pointed out that participants can engage in both the FL training process and the block validation in PoS-based blockchains or mining activities in PoW-based blockchains. These design choices not only raise the entry barrier for regular users wishing to partake in blockchain-based FL systems but also add complexity to the fundamental consensus mechanism. 

In this work, we introduce a PoS-based reward-and-slash mechanism for the FL system. We compare our solution with existing work in Table~\ref{tab:existing-work}. Our solution can be seamlessly integrated into any smart-contract-enabled blockchain without requiring modifications to the underlying consensus design (\ie, our design is blockchain agnostic). This approach facilitates the participation of any FL users, making it more accessible and user-friendly.

\section{Problem Formulation}
\label{sec:prob}
This section introduces the problem of interest, the definition of the malicious behaviors considered, and the underlying assumptions in this work. The main definitions and notations adopted in this work are summarized in Table~\ref{tab:def}.

\subsection{Setup}
\label{sec:prob:setup}
There are $K > 1$ clients in a federated system. Let $\mathcal{K} = \{1, 2, \cdots, K\}$ denote the set of all clients. Let $\mathcal{D}_k$ denote the local data stored in client $k$, we have $\mathcal{D}_k \cap \mathcal{D}_l = \emptyset$ for $k \neq l$ and $k, l \in \mathcal{K}$. Each local dataset $\mathcal{D}_k$ can be randomly split into a training set and a test set, which are both private to client $k$. In addition to $K$ clients, a blockchain plays the role of a parameter server~\cite{li2014communication} for global aggregation. Let $f_{\theta}$ be the model of interest. In the parameter server, the parameter set $\theta_0^0$ is randomly initialized at round $0$ and $K$ clients download $\theta_0^0$ from the blockchain as $K$ local copies $\{\theta_k^0\}_{k=1}^K$ for full synchronization. During the federated optimization phase, a set of $\mathcal{K}_p^t$ clients is randomly selected for round $t$. For each $k \in \mathcal{K}_p^t$, the client $k$ updates $\theta_k^{t-1}$ by training on the training set of $\mathcal{D}_k$ independently for a number of local epochs. Then, the blockchain aggregates updated $\{\theta_k^{t}\}_{k \in \mathcal{K}}$ collected from all the $K$ clients to update $\theta_0^{t}$. The $K$ clients then synchronize with the parameter server, \ie~$\theta_k^{t} \leftarrow \theta_0^{t}$. To facilitate data governance, as required in among others the medical domain~\cite{hipaa,gdpr}, we assume that the patient's data (either raw data or encoded data) in a client can not be uploaded to the blockchain or other clients, \ie~only parameters $\{\theta_k\}_{k=0}^K$ and \emph{metadata} (\eg~the statistics of data)~\cite{dong2022federated,dong2022learning} can be exchanged between the blockchain and the clients. It is worth mentioning that this work focuses on the interactions between FL and blockchain, where blockchain computing (or \emph{mining}, in a more fashionable sense) and the application of additional privacy-preserving techniques~\cite{abadi2016deep} are considered orthogonal research directions and thus beyond the scope of this work.

\subsection{Malicious Behaviors}
\label{sec:prob:mal}
The definition of malicious behavior in this work is an action that intentionally decreases the global model performance. There are two types of actions for each client that interact with the federated system, \ie~a client can propose (\ie~be a \emph{proposer}) and vote (\ie~be a \emph{voter}). Proposing is to upload local model or gradient updates to the parameter server, while voting is a peer-review process to validate the ``virtually'' aggregated model updates. The technical details of the two actions are described in Sec.~\ref{sec:method}. There are thus two corresponding malicious behaviors. The first malicious behavior is to propose harmful local model updates and the second one is to vote dishonestly. More specifically, in the second case, a client votes for approval when it is aware that the proposed model updates are poisoned and votes for rejection when there is no evidence that indicates that the proposed model updates are poisoned. It is worth mentioning that the clients themselves might not intentionally attack the FL system as they can be compromised by attackers. For simplicity, we define the clients that have malicious behaviors as \emph{malicious clients} in this work, denoted as $\mathcal{K}_m$. We use $\eta$ to denote the ratio of malicious clients among all clients, \ie~$\eta = \frac{|\mathcal{K}_m|}{\mathcal{K}}$, where $|\cdot|$ is the cardinality of a set.

\subsection{Assumptions}
\label{sec:prob:assume}
There are six important assumptions in this work. 

\begin{IEEEitemize}
    \item \textbf{A1}: The goal of malicious behaviors is to decrease the global model performance. This is also reflected in Sec.~\ref{sec:prob:mal}. Under these assumptions, behaviors that are harmful to the system but do not influence the global model performance are beyond the scope of discussion in this work. An example is eavesdropping, \ie~cloning the model specifications.
    \item \textbf{A2}: All clients are rational. This means that both honest and malicious clients expect to maximize their gain or minimize their loss while achieving their goals.
    \item \textbf{A3}: Following previous studies on blockchain~\cite{li2021blockchain}, we assume that $\eta$ is strictly smaller than $50\%$. This means there are always more honest clients than malicious clients in a federated system.
    \item \textbf{A4}: There is no capacity constraint on the hardware, including computing, communication, and storage, allowing us to solely focus on the algorithmic side of the problem.
    \item \textbf{A5}: The underlying blockchain of the FL system of interest is running securely with a consensus protocol that ensures the validity and integrity of transactions and blocks. While the security of the blockchain is crucial for the overall security of the FL system, addressing the malicious miners falls outside the scope of this study. 
\end{IEEEitemize}

\begin{table}[t]
    \centering
    \caption{Summary of main definitions and notations.}
    \renewcommand\arraystretch{1.25}
    \resizebox{\columnwidth}{!}{
    \begin{tabular}{r||l}
    \toprule
         Notation & Description\\ \hline
         blockchain-based FL& FL framework with a blockchain architecture\\
         FedAVG & Federated Averaging~\cite{mcmahan2017communication}, a FL algorithm\\
         non-IID & Non-independently and identically distributed\\
         $\mathcal{K}$ & Set of participating clients at round $t$\\
         $K_v^t$ & The selected set of voters\\
         $\mathcal{K}_p^t$ & Set of proposers at round $t$\\
         $a^t$ & Majority voting decision at round $t$\\
        $M_k$ & Asset of client $k$\\
        $\gamma_p$ & Staked tokens for proposing\\
        ${pool}_p$ & Pool for storing proposers' stake\\
    \bottomrule
    \end{tabular}
    }
    \label{tab:def}
\end{table}

\section{Method}
\label{sec:method}
In this section, we first introduce the basics of federated aggregation in Sec.~\ref{sec:method:agg}, and describe the local validation and majority voting in Sec.~\ref{sec:method:val} and Sec.~\ref{sec:method:vote}. We then propose a novel incentive mechanism in Sec.~\ref{sec:method:asset}. A theoretical analysis on malicious voting is presented in Sec.~\ref{sec:method:analysis}. We then describe the whole training pipeline in Sec.~\ref{sec:method:train}. Finally, the analysis of computational cost is provided in Sec.~\ref{sec:method:cost}.

\begin{figure*}[t]
  \centering
  \includegraphics[width =\textwidth]
  {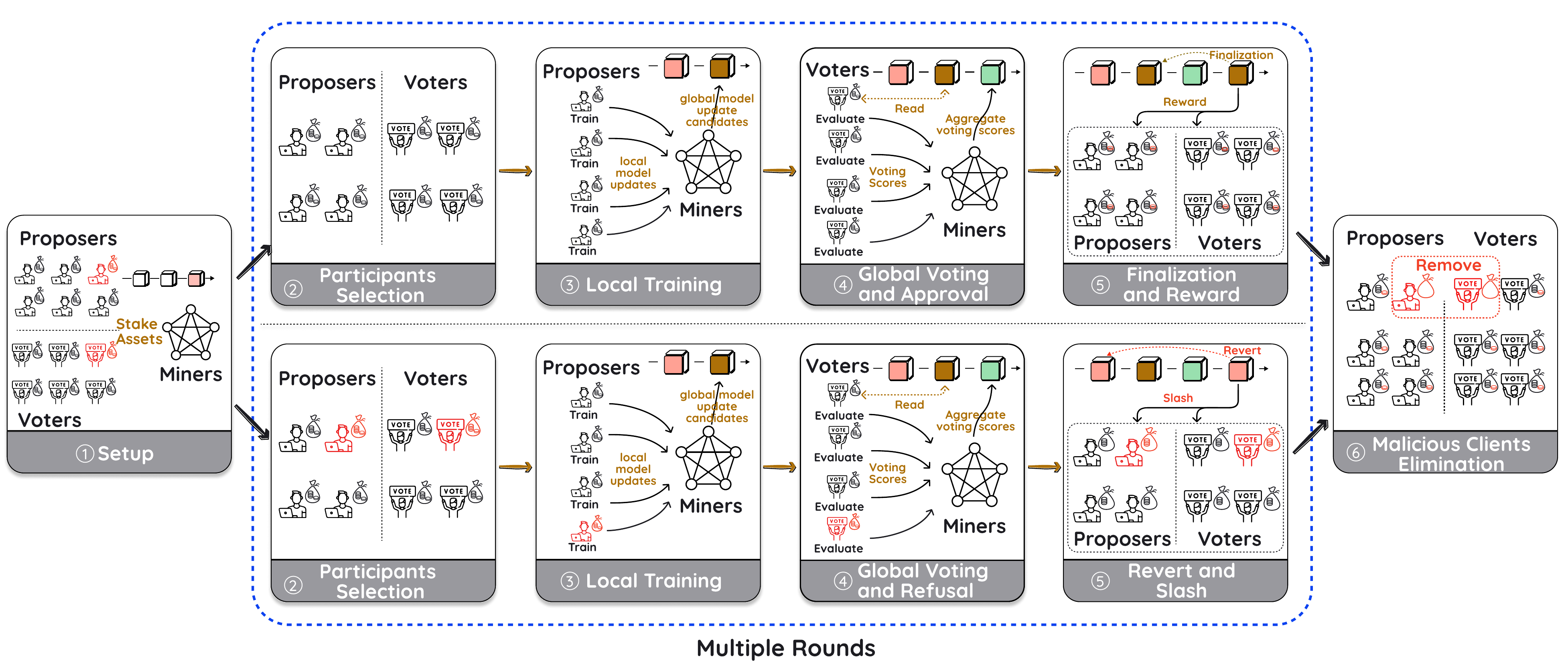}
  \caption{A round-based training process. In the initial state (indexed as \ding{172}), both honest (black) and malicious ({\color{red}red}) clients exist in an FL system. In the final state (indexed as \ding{177}), all malicious clients are expected to be removed from the system. To reach the final state from the initial state, multiple rounds of training are required. Here are two possible scenarios, the proposed aggregation is either approved (the upper branch) or denied (the lower branch) by the voters. In each round (within the dotted {\color{blue}blue} line), a subset of clients are randomly selected as proposers, and another subset of clients are randomly selected as voters. The proposers and voters interact with the blockchain following the order of {\color{orange}orange} arrows (from \ding{173} to \ding{176}).}
  \label{fig:framework}
\end{figure*}

\subsection{Federated Aggregation}
\label{sec:method:agg}
In this work, we illustrate the proposed framework in the context of the seminal FL method, FedAVG~\cite{mcmahan2017communication}. At the end of round $t$, the local models $\{\theta_k^{t}\}_{k=1}^K$ are uploaded and aggregated as a weighted average:
\begin{equation}
    \theta_0^t = \sum_{k=1}^K a_k \theta_k^{t},
    \label{eq:FedAvg}
\end{equation}
where $a_k = \frac{n_k}{N}$. The metadata $n_k = |\mathcal{D}_k|$ is the number of local training examples stored in client $k$ and $N = \sum_{k=1}^K n_k$ is the total number of training examples in the $K$ clients.

\subsection{Local Validation}
\label{sec:method:val}
In contrast to standard FL algorithms, the aggregated global model is not recorded in a block directly. Instead, $\tilde{\theta}_0^t$, a copy of $\theta_0^t$ is downloaded by a randomly selected set of clients, denoted as \emph{voters}, $\mathcal{K}_v^t$. A voter $k$ runs a local inference with $\tilde{\theta}_0^t$ on its local test set and outputs a local validation score. The local validation score $s^{t}_k$ is a scalar, which can be linked with common metrics of ML tasks\footnote{For example, common evaluation metrics include accuracy for classification, mean Intersection over Union (mIOU) for semantic segmentation, and mean average precision (mAP) for object detection.}. If $s^{t}_k$ is not lower than a threshold, the voter votes for accepting this aggregated model; otherwise, the voter votes against it. The threshold can be based on a validation score $s^{t-1}_k$ acquired in the previous round. In the training of ML tasks, the scores can be volatile due to the characteristics of the tasks. Thus, a hyperparameter $\epsilon \in (0, 1)$ is introduced to control the tolerance of performance decrease in a single round. Mathematically, the $k$-th voter has the following score.
\begin{equation}
  v_k^t =
    \begin{cases}
      \phantom{+}1, & s^{t}_k \ge (1 - \epsilon)s^{t-1}_k \\
      -1, & s^{t}_k < (1 - \epsilon)s^{t-1}_k
    \end{cases} 
\end{equation}
It is worth mentioning that the likelihood of the attackers consistently manipulating the scores by fooling all the randomly selected voters (\eg~via adversarial attacks~\cite{goodfellow2014explaining}) diminishes quickly towards zero as the number of epochs increases. According to \textbf{A4}, the majority of voters are honest. It is thus difficult to attack (either via data poisoning or model poisoning) as the validation set of each client is private.

\subsection{Majority Voting}
\label{sec:method:vote}
 The majority voting process for whether to apply the global aggregation operation at round $t$ can be described below. Here, we use a binary variable $a^t$ to denote the decision.
\begin{equation}
  a^t =
    \begin{cases}
      \phantom{+}1, & \sum_{k \in \mathcal{K}_v} v_k > 0\\
      -1, & \sum_{k \in \mathcal{K}_v} v_k \le 0
    \end{cases}   
\end{equation}
If $a^t = 1$, the global aggregation will be finalized and recorded in the block; otherwise, the global aggregation will be discarded.

\subsection{Asset Redistribution}
\label{sec:method:asset}
As there are two independent actions, there are two parallel reward-and-slash designs for proposing and voting. For both actions, the randomly selected proposers and voters are required to stake a fixed sum of tokens before they act. If some of these actors fail to stake (they do not have enough tokens left), they lose their access to the blockchain and are removed from the FL system permanently. Proposers will be rewarded with tokens accumulated in an independent pool (if there are any tokens left in the pool) if the global aggregation is approved and lose their stakes if the global aggregation is rejected. The reward-and-slash design for the proposers is illustrated in Algorithm~\ref{algo:1}. For the voters, the majority party will not only take back their stakes but also be rewarded with the staked tokens lost by the minority party. The reward-and-slash design for the voters is illustrated in Algorithm~\ref{algo:2}. In the following section, Sec.~\ref{sec:method:analysis}, we demonstrate that under the proposed design and assumptions in Sec.~\ref{sec:prob:assume}, malicious voters have no incentive to make dishonest votes. Note, Sec.~\ref{sec:method:asset} highlights the key difference between the proposed voting mechanism and traditional majority voting as the voting is directly linked with the incentive mechanism. 

\begin{algorithm}[th]
    \centering
    \begin{algorithmic}[1]
        \Statex $a^t$: Majority voting decision at round $t$
        \Statex $\mathcal{K}$: Set of participating clients at round $t$
        \Statex $\mathcal{K}_p^t$: Set of proposers at round $t$
        \Statex $M_k$: Asset of client $k$
        \Statex $\gamma_p$: Staked tokens for proposing
        \Statex ${pool}_p$: Pool for storing proposers' stake
        \If {$a^t == -1$} 
            \For{$k \in \mathcal{K}_p^t$}
                \If {$M_k \ge \gamma_p$}
                    \State $M_k \leftarrow M_k - \gamma_p$
                    \State ${pool}_p \leftarrow {pool}_p + \gamma_p$
                \Else
                    \State ${pool}_p \leftarrow {pool}_p + M_k $
                    \State $M_k \leftarrow 0$
                    \State $\mathcal{K}^t \leftarrow \mathcal{K}^t \setminus \{k\}$ 
                \EndIf
            \EndFor
        \Else
            \If {${pool}_p > 0$}
                \For{$k \in \mathcal{K}_p^t$}
                    \State $M_k \leftarrow M_k + \frac{{pool}_p}{|\mathcal{K}_p^t|} $
                \EndFor
                \State $ {pool}_p \leftarrow 0$
            \EndIf 
        \EndIf
    \end{algorithmic}
    \caption{Reward-and-slash design for a set of randomly selected proposers.}
    \label{algo:1}
\end{algorithm}

\begin{algorithm}[th]
    \centering
    \begin{algorithmic}[1]
        \Statex $a^t$: Majority voting decision at round $t$
        \Statex $\mathcal{K}$: Set of participating clients at round $t$
        \Statex $\mathcal{K}_v^t$: Set of voters at round $t$
        \Statex $\mathcal{K}_m^t$: Set of voters at round $t$ with $v_k^t == a^t$
        \Statex $M_k$: Asset of client $k$
        \Statex $\gamma_v$: Staked tokens for voting
        \Statex ${pool}_v$: Pool for storing voters' stake
        \For{$k \in \mathcal{K}_v^t \setminus \mathcal{K}_m^t$}
            \If {$M_k \ge \gamma_v$}
                \State $M_k \leftarrow M_k - \gamma_v$
                \State ${pool}_v \leftarrow {pool}_v + \gamma_p$
            \Else
                \State ${pool}_v \leftarrow {pool}_v + M_k $
                \State $M_k \leftarrow 0$
                \State $\mathcal{K}^t \leftarrow \mathcal{K}^t \setminus \{k\}$ 
            \EndIf
        \EndFor
        \For{$k \in \mathcal{K}_m^t$}
            \State $M_k \leftarrow M_k + \frac{{pool}_v}{|\mathcal{K}_m^t|} $
        \EndFor
        \State ${pool}_v \leftarrow  0$
    \end{algorithmic}
    \caption{Reward-and-slash design for a set of randomly selected voters.}
    \label{algo:2}
\end{algorithm}

\subsection{Theoretical Analysis on Malicious Votes}
\label{sec:method:analysis}
In this section, we theoretically show that malicious voters in the proposed framework have no incentive to make dishonest votes.

\begin{theorem}[Honest Voting Hypothesis]
\label{thm:vote}
When all clients are rational, a malicious client should not make a malicious vote.
\end{theorem}

\begin{proof}
Let $\mathcal{K}_v$ denote a randomly selected set of voters and $n_v = |\mathcal{K}_v|$. For client $k \in \mathcal{K}_v$, let $\gamma_v > 0$ denote the staked tokens for voting, \ie~client $k$ must stake $\gamma_v$ to participate in the voting, otherwise, it will be removed from the system. 

Let us consider a multi-agent scenario, where malicious clients can collude. No matter how the malicious clients cooperate, there are two types of malicious clients. The first type behaves maliciously to achieve the goal of sabotaging the FL training, \ie, lowering the global performance. The second type acts honestly to hide themselves and survive to be able to implement the complex policy to sabotage the FL training at the a later stage. If the malicious clients belong to the second type, they are factually ``honest'' ones.

Let $r$ be the ratio of malicious clients in $\mathcal{K}_v$, there are $r \cdot n_v$ malicious clients in $\mathcal{K}_v$ and $(1 - r) \cdot n_v$ honest clients. If $r \cdot n_v < (1 - r) \cdot n_v$, \ie~$r < 0.5$, each malicious client will lose $\gamma_v$; if $r \cdot n_v > (1 - r) \cdot n_v$, \ie~$r > 0.5$, each malicious client will gain $\frac{(1 - r) \cdot n_v \cdot \gamma_v }{r \cdot n_v} = \frac{1- r}{r} \gamma_v$. The expected return $\mathcal{R}$ of a malicious client will be 
\begin{equation}
\begin{split}
    \mathcal{R} &= \int_{0}^{0.5} -\gamma_v dr + \int_{0.5}^{1} \frac{1- r}{r} \gamma_v dr \\
                &= - 0.5 \gamma_v + ((\ln(1) - 1) - (\ln(0.5) - 0.5)) \gamma_v \\
                &=  -( \ln(0.5) + 1) \gamma_v < 0 
\end{split}
\label{eq:thm}
\end{equation}

Under \textbf{A2}, each client is rational. As $\mathcal{R} < 0$, in the long run, a malicious client will lose all tokens and be removed from the system. So, a given client has no reason to make a dishonest vote resulting in honest votes by all clients.
\end{proof}

Theorem~\ref{thm:vote} will further be empirically validated in Sec.~\ref{sec:exp:thm}. 

If \textbf{A2} holds, we are certain that the malicious voters will reach a consensus internally before they act to win the majority vote. Intuitively, all malicious voters can be considered as a group together. In this case, this "group" will behave exactly as the single malicious client in Theorem~\ref{thm:vote} based on the same reasoning. The proof is omitted.

\subsection{Training}
\label{sec:method:train}
Each round consists of the following steps: proposer selection, local training, global aggregation, local validation, majority voting, token redistribution, and block creation (recording \emph{state}\footnote{For example, the state can record the global model and tokens of each client.} information). The above steps are repeated in multiple rounds until certain stopping criteria are fulfilled. The complete training process is depicted in Fig.~\ref{fig:framework}. The stopping criteria could be a fixed amount of training epochs, which is commonly adopted in ML.

\subsection{Computation Complexity and Cost}
\label{sec:method:cost}
For proposers, the primary computational cost is driven by the local training algorithm, mirroring the structure of traditional FL. For voters, their computation cost stems from evaluating the aggregated global model. In addition to mining blocks, miners engage in further computations by consolidating global model updates, akin to the responsibilities of a centralized aggregator in traditional FL. The overall computational complexity is contingent upon the underlying training network backbone.

The communication and storage costs are outlined as follows. We assume the model size is $M$. During an epoch, the communication cost for a client (\ie, a proposer or a voter) is $O(M)$. And the storage cost on the blockchain is $O(K\cdot M)$, where $K$ is the number of clients.

\section{Experiments}
\label{sec:exp}

\begin{figure*}[htbp]
\centering
\hfill
\subfigure[$\eta$ = 0.1.]{
\includegraphics[width=0.483\columnwidth]{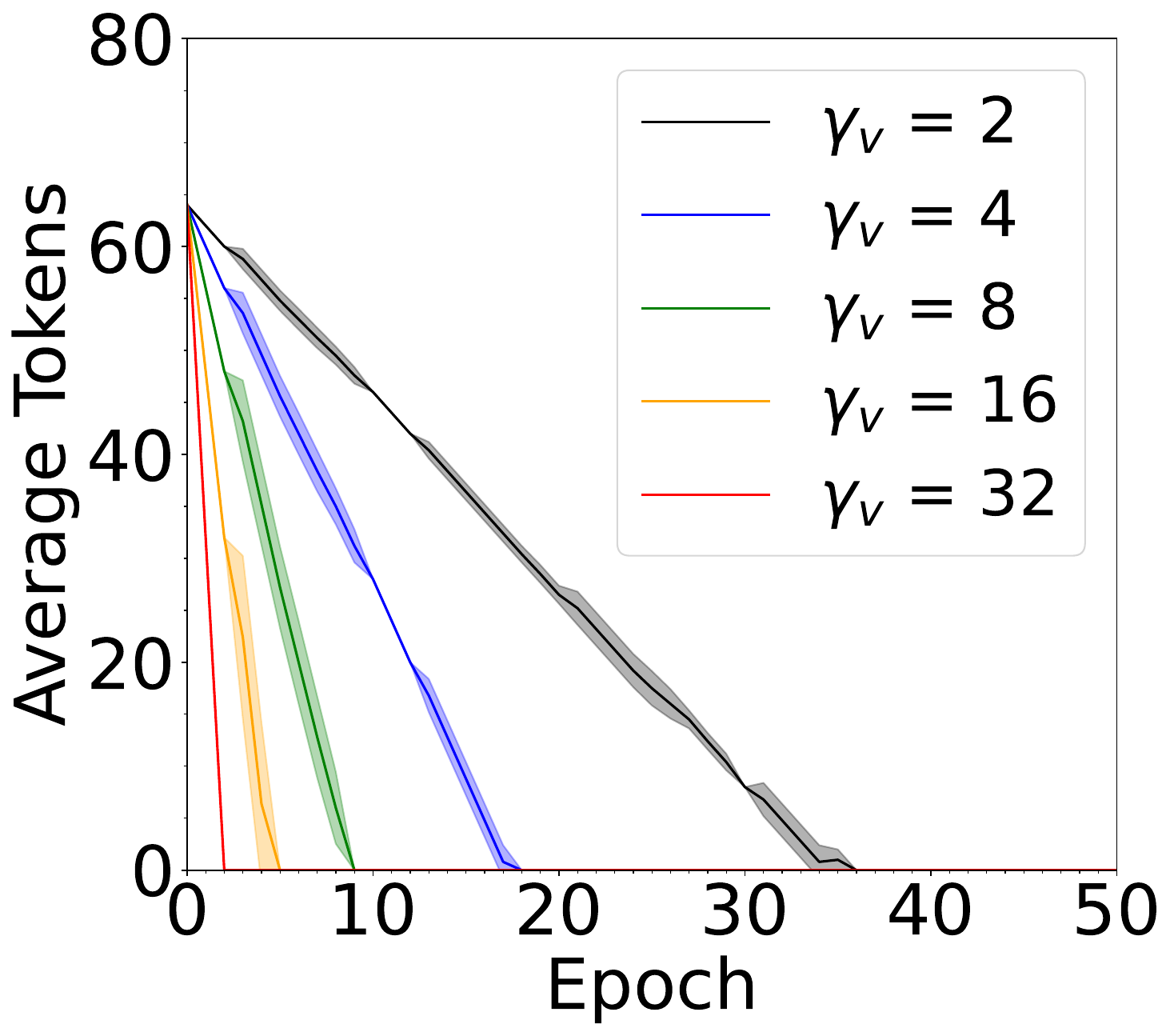}
\label{fig:mal_voters_tokens_with_rate_0.1}
}%
\hfill
\subfigure[$\eta$ = 0.2.]{
\includegraphics[width=0.483\columnwidth]{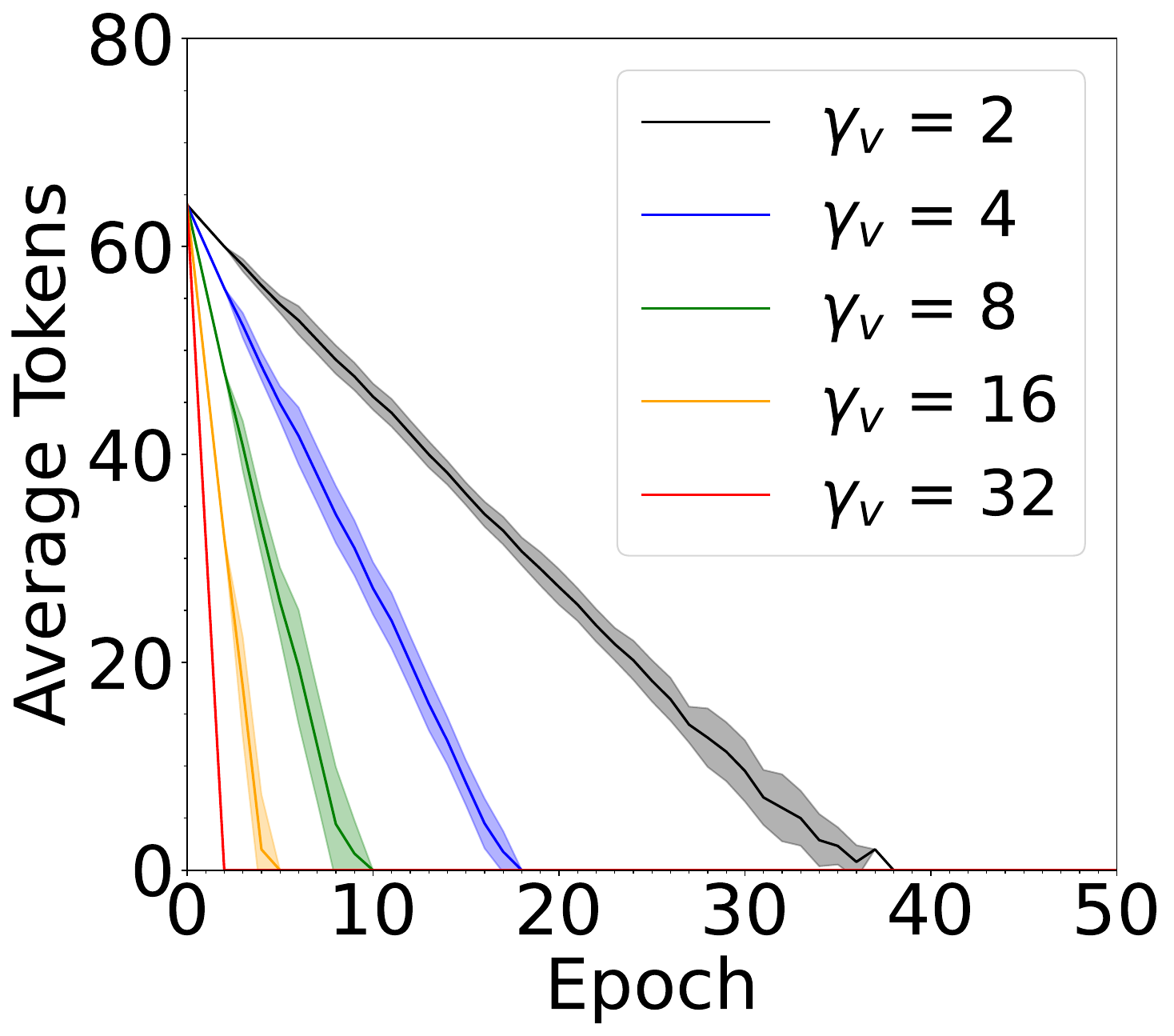}
\label{fig:mal_voters_tokens_with_rate_0.2}
}%
\hfill
\subfigure[$\eta$ = 0.3.]{
\includegraphics[width=0.483\columnwidth]{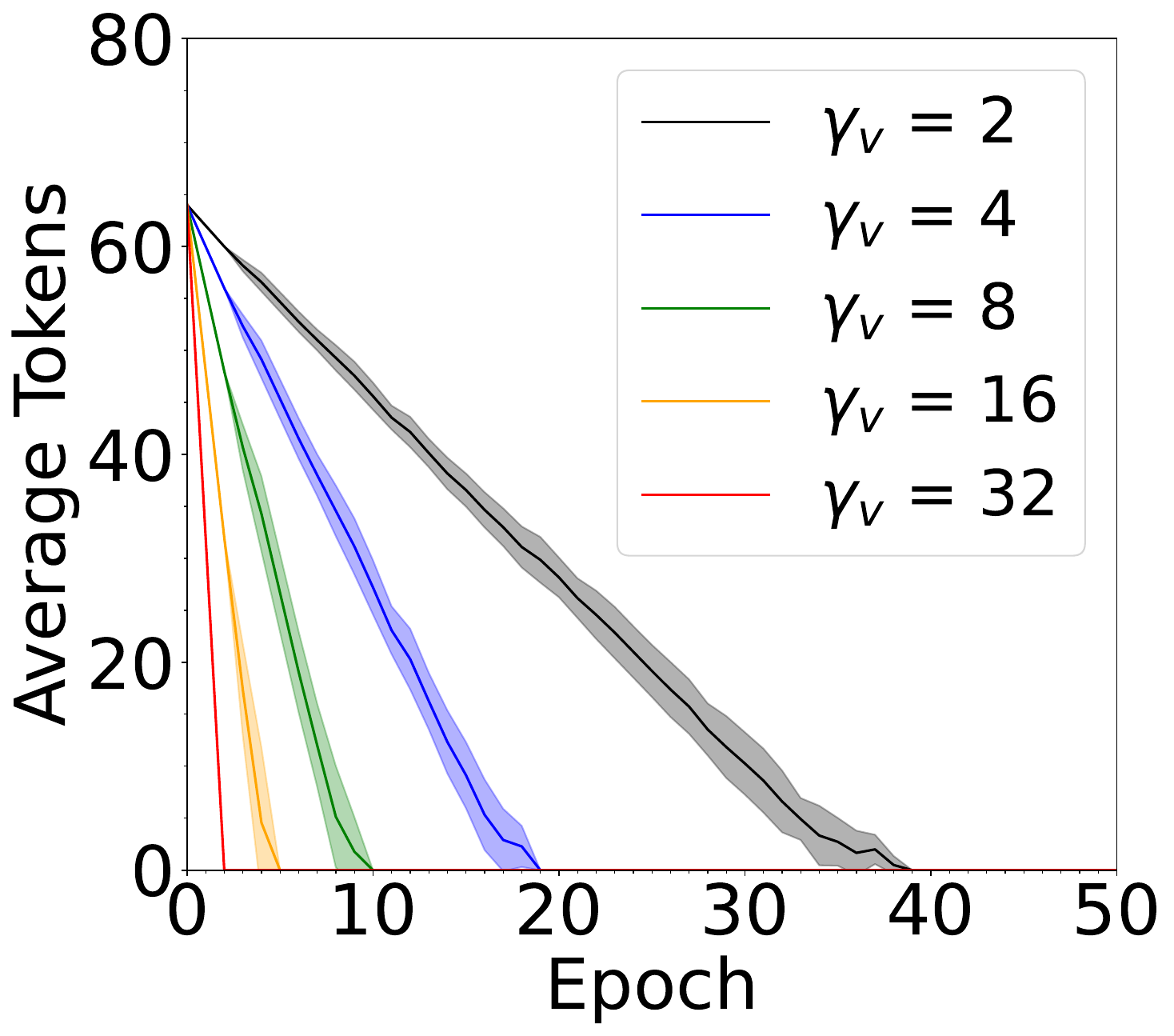}
\label{fig:mal_voters_tokens_with_rate_0.3}
}%
\hfill
\subfigure[$\eta$ = 0.4.]{
\includegraphics[width=0.483\columnwidth]{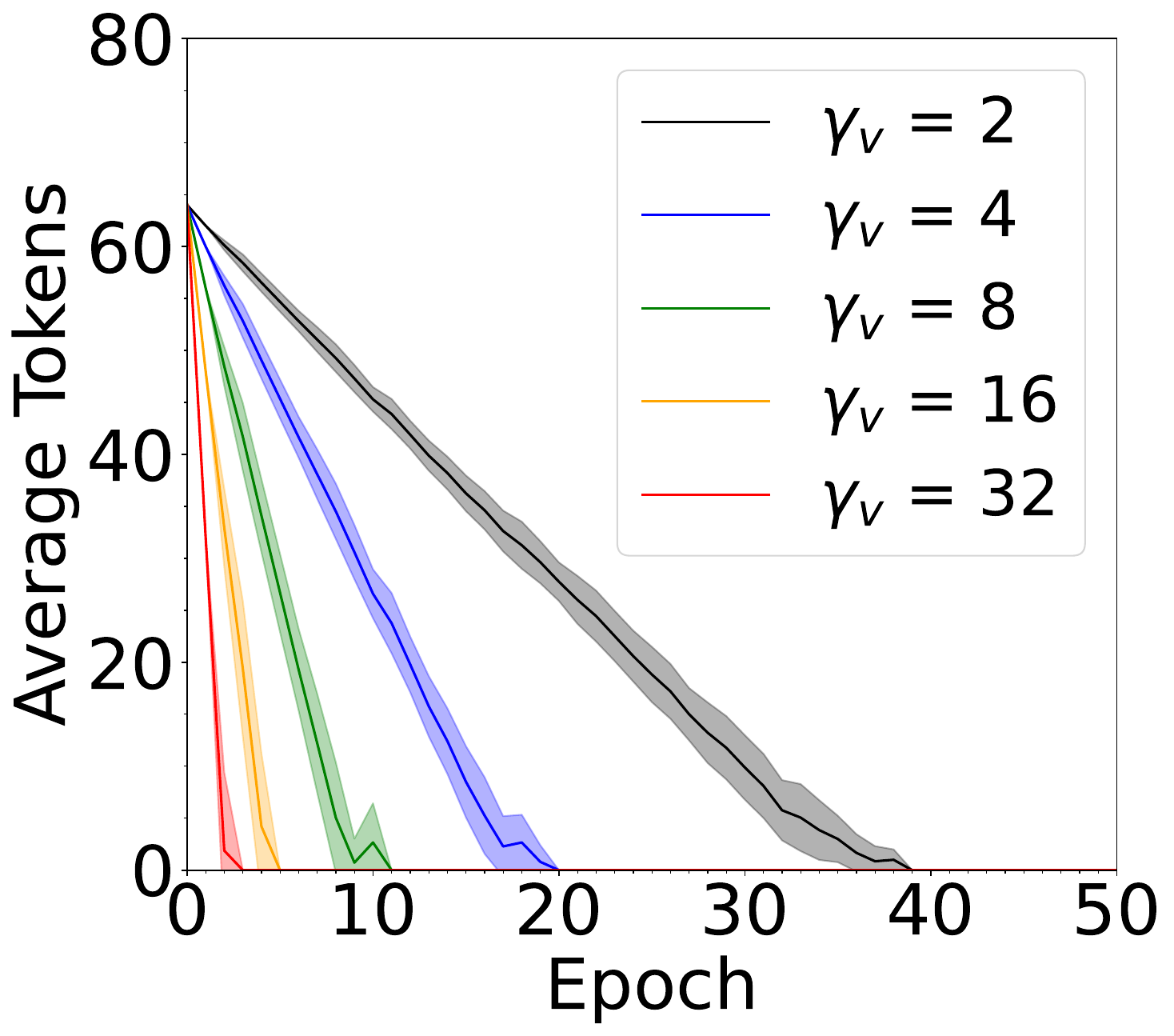}
\label{fig:mal_voters_tokens_with_rate_0.4}
}%
\hfill
\centering
\caption{Token distribution results for malicious voters when all proposers are honest. The malicious voters' tokens decrease quickly as the number of epochs increases and a large $\gamma_p$ leads to a high decreasing rate. This empirically validates our proof of Theorem~\ref{thm:vote}. The solid line denotes the mean over 5 runs with different random seeds and the shaded region denotes 1 standard deviation around the mean.}

\label{fig:mal_voters_token_distribution_sv=4}
\end{figure*}

\begin{table*}[t]
\centering
\caption{Performance comparison under different values of the ratio of malicious clients ($\eta$). The reported numbers of the performance are mean and standard deviation under 5 random seeds.}
\label{tab:training_res}
\begin{tabular}{l|c|c|c|c}\hline
    Model & $\eta = 0.1$ & $\eta = 0.2$ & $\eta = 0.3$ & $\eta = 0.4$ \\\hline
    FedAVG w/ mal &$0.963\pm0.017$ & $0.946\pm0.034$ &$0.801\pm0.222$  &$0.709\pm0.266$ \\\hline
    FedAVG w/ block (Ours) &$0.965\pm0.008$ &$0.969\pm0.003$ &$0.952\pm0.020$ &$0.955\pm0.021$  \\\hline\hline
    FedAVG w/o mal &$0.975\pm0.004$ &$0.975\pm0.004$ &$0.975\pm0.004$ &$0.975\pm0.004$  \\\hline
    \emph{Oracle} &$0.971\pm0.007$ &$0.971\pm0.007$ &$0.971\pm0.007$ & $0.971\pm0.007$\\\hline
\end{tabular}

\end{table*}

\begin{figure*}[t]
\centering
\hfill
\subfigure[$\eta$ = 0.1.]{
\includegraphics[width=0.483\columnwidth]{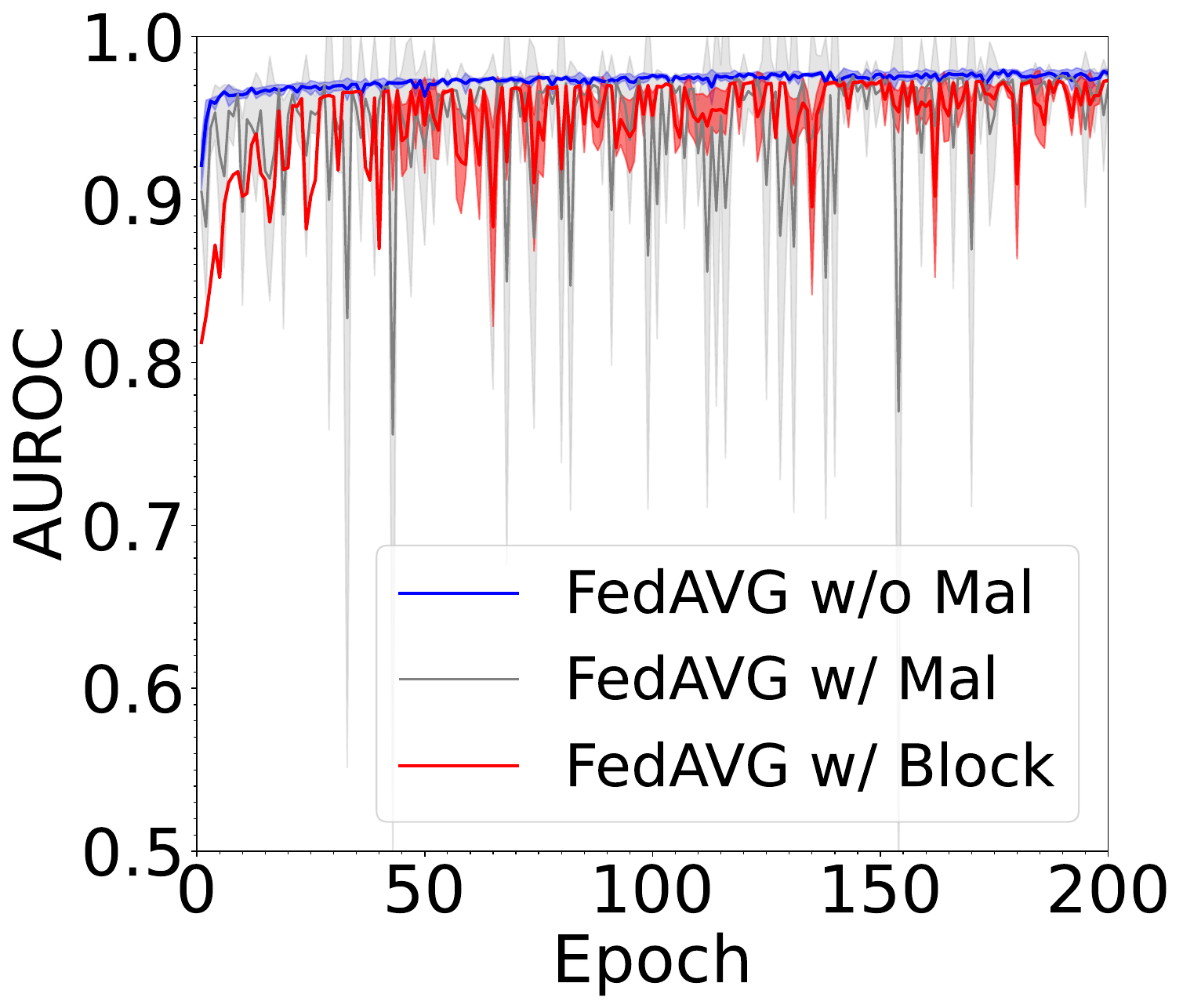}
\label{fig:training_res_with_rate_0.1}
}%
\hfill
\subfigure[$\eta$ = 0.2.]{
\includegraphics[width=0.483\columnwidth]{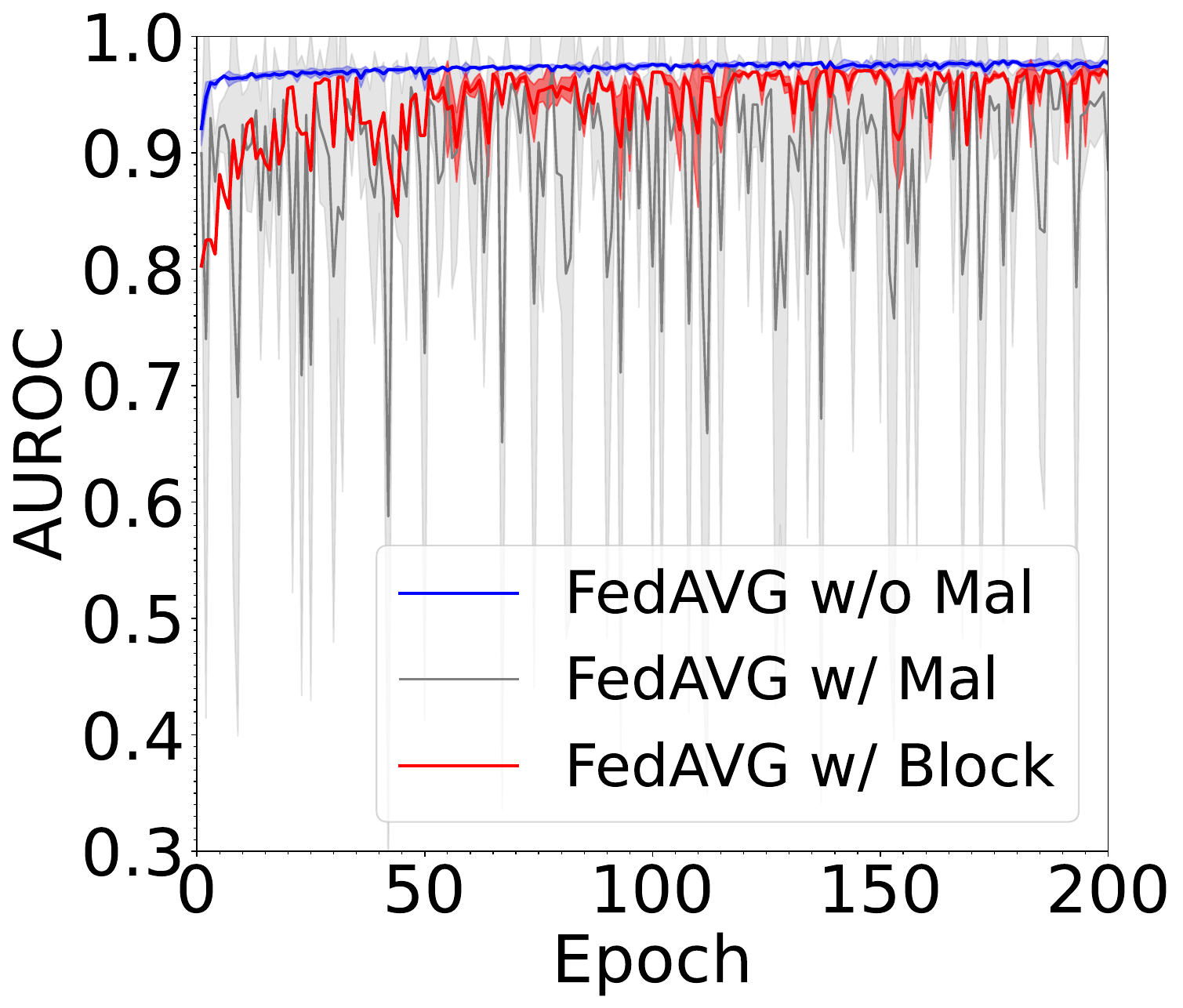}
\label{fig:training_res_with_rate_0.2}
}%
\hfill
\subfigure[$\eta$ = 0.3.]{
\includegraphics[width=0.483\columnwidth]{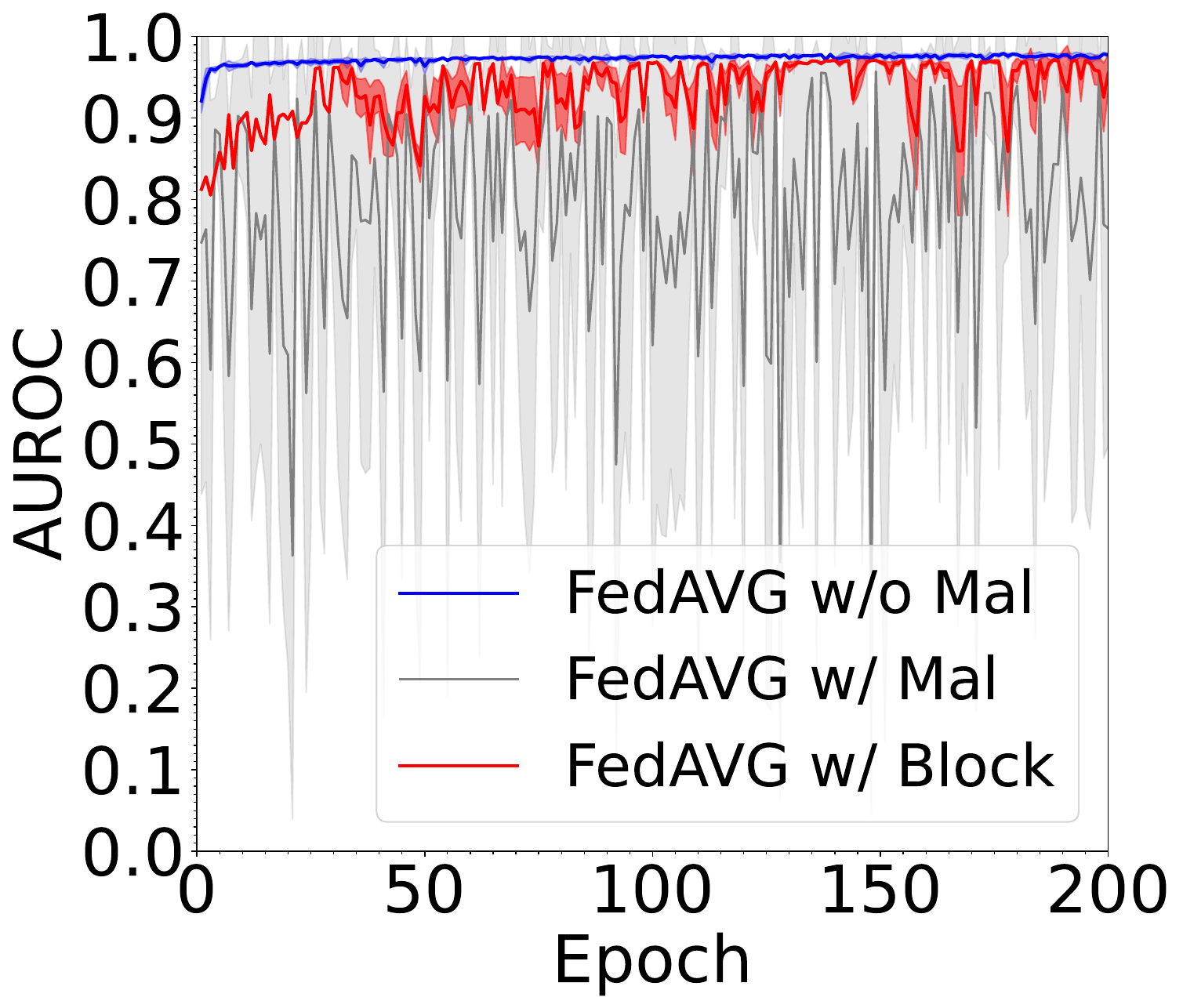}
\label{fig:training_res_with_rate_0.3}
}%
\hfill
\subfigure[$\eta$ = 0.4.]{
\includegraphics[width=0.483\columnwidth]{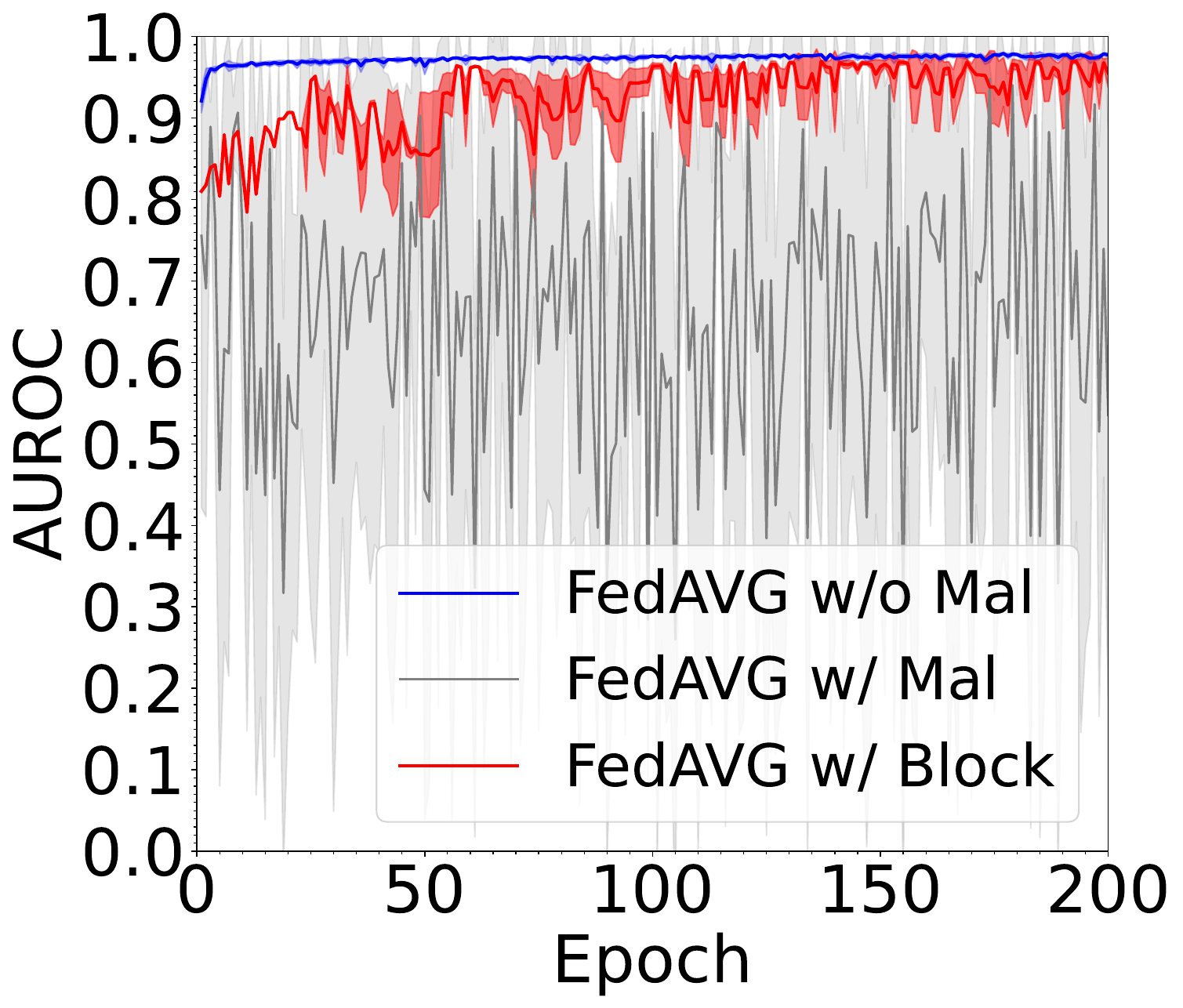}
\label{fig:training_res_with_rate_0.4}
}%
\hfill
\centering
\caption{Federated training under different values of the ratio of malicious clients ($\eta$). Each subfigure shows the training AUROCs when the training time (\ie~the number of epochs) increases.
The solid lines are the mean AUROCs and the shaded regions are 1 standard deviation around the means. We compare the performances of FedAVG with blockchain (i.e., w/ Block), FedAVG with malicious clients (i.e., w/ mal), and FedAVG without malicious clients (i.e., w/o mal). We observe that FedAVG w/ Block significantly outperforms FedAVG w/ mal, while being comparable with FedAVG w/o mal, the performance upper bound under this setup.}
\label{fig:training_res}
\end{figure*}

\subsection{Experimental Setup}
We first evaluate the proposed framework in a simulated environment.

\subsubsection{Data and Task} 
We consider a standard binary classification task, namely loan default prediction. We use the Kaggle Lending Club dataset\footnote{\url{https://www.kaggle.com/datasets/wordsforthewise/lending-club}} to simulate a realistic financial application scenario. 
We pre-process the raw dataset by dropping all entries with missing values. For the labels, we only keep ``Fully Paid'' and ``Charged Off'' to simplify the task as a binary classification task.
We randomly select $80\%$ of the data as the training set and use the rest of the data as the test set. The training set is split into $K$ subsets of equal size and distributed across $K$ clients. Within each client, $20\%$ of the local data are randomly selected as the validation set. \\

\subsubsection{Implementation}
\label{sec:exp:imp}
There are $K = 50$ clients in the system and each client is initialized with 64 tokens. We use a 3-layer multi-layer perceptron (MLP) as the network backbone. Apart from the last layer, each layer of the MLP has 128 hidden nodes. We use a standard Adam~\cite{kingma2015adam} optimizer with fixed learning rate $10^{-3}$ and batch size 128. No data augmentation is applied. We use the binary accuracy as both the local validation score and evaluation metric. 
In our experiments, malicious clients are randomly selected before the training according to the ratio $\eta$. In each training round, if a malicious client is selected as the proposer or voter, it will act maliciously as described in Sec.~\ref{sec:method}.\footnote{If a malicious client acts honestly, then it will be considered as an honest one and makes no harm to the system.} We consider a simple data poisoning attack~\cite{bagdasaryan2020how}, where malicious clients are trained to confuse the model. Specifically, the local models are trained to lower the model performance by using the wrong labels, but maintaining a low weight divergence from the aggregated weights from the last round.
All baselines are implemented in PyTorch 1.12.1~\cite{paszke2019pytorch} on one NVIDIA Tesla T4 GPU. We leverage Ethereum smart contracts to deploy our reward-and-slash design in a private blockchain and simulate the training process using the Python library \texttt{Web3.py}\footnote{\url{https://web3py.readthedocs.io/en/v5/}}. We set $\epsilon = 0.05$ based on empirical experience.\footnote{We notice that too small $\epsilon$ can cause large oscillation, which slows the convergence, and too large $\epsilon$ can facilitate the convergence at the expense of decreased detection performance, \ie~the system fails to remove the majority of malicious clients.}

\subsubsection{Baselines}
\label{sec:exp:base}
So far, there is no such blockchain-based FL baseline suitable for the comparative evaluation of counteracting malicious behaviors in our problem setting. It is worth mentioning that the proposed framework can be integrated with the existing FL method. In our experiments, we use FedAVG as the backbone FL method to illustrate our defending mechanism. We consider 4 baselines. The first one is an \emph{Oracle} approach, a centralized baseline without malicious attacks. The \emph{Oracle} should provide the upper-bound performance of the experiment. The second one is FedAVG without malicious attacks (denoted as FedAVG w/o mal), which is equivalent to FedAVG under $\eta = 0$ and should provide the upper-bound performance for a decentralized environment. The third one is FedAVG under malicious attacks (denoted as FedAVG w/ mal), where $\eta$ of clients are malicious. The fourth one is the proposed method, FedAVG with blockchain under malicious attacks (denoted as FedAVG w/ block). For FL baselines, $10\%$ of clients are randomly selected to perform local training at each epoch. For FedAVG w/ block, we simply use the remaining $90\%$ of the clients as voters.

\begin{figure*}[t]
\centering
\hfill
\subfigure[$\eta$ = 0.1.]{
\includegraphics[width=0.483\columnwidth]{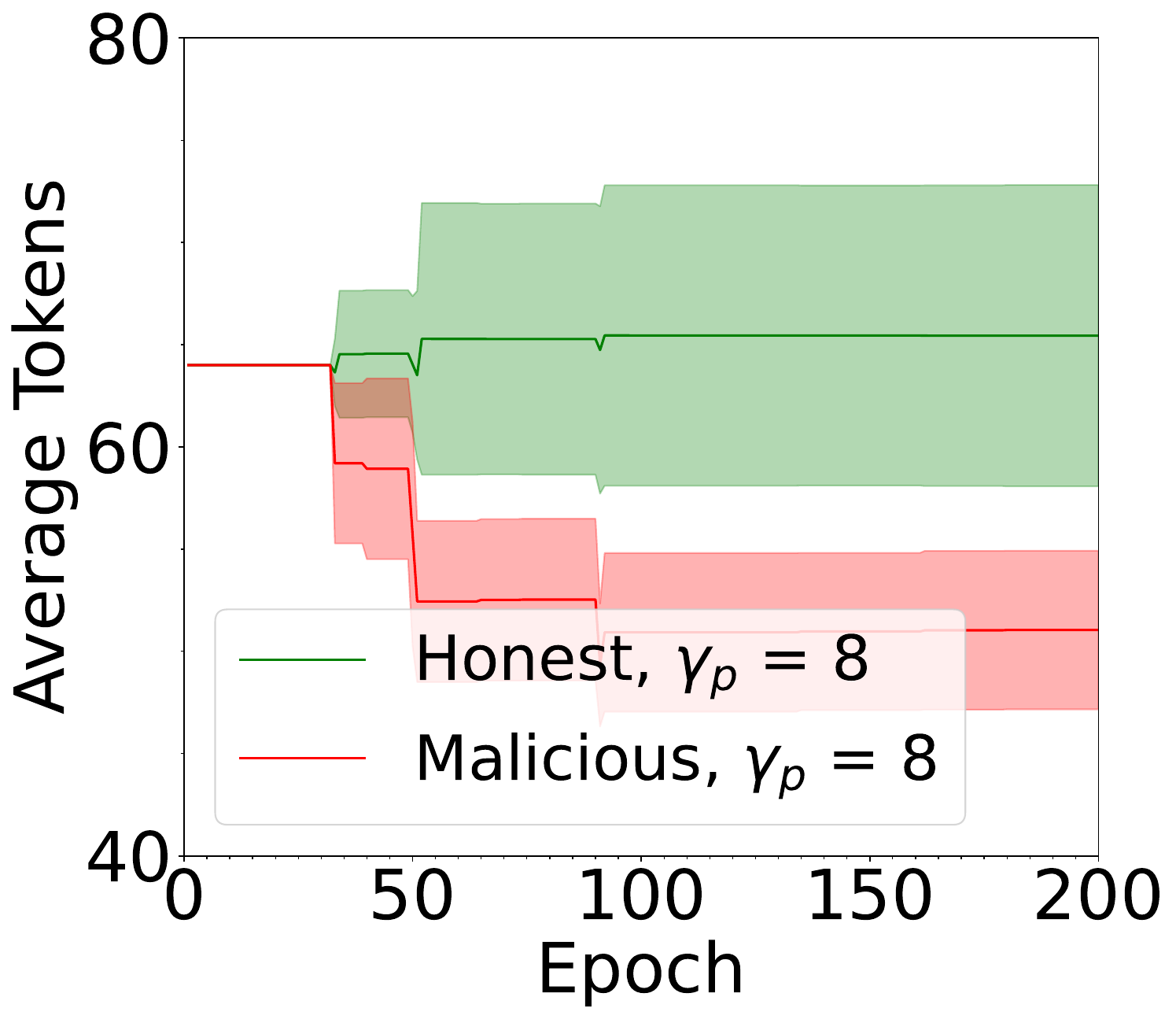}
\label{fig:tokens_with_rate_0.1}
}%
\hfill
\subfigure[$\eta$ = 0.2.]{
\includegraphics[width=0.483\columnwidth]{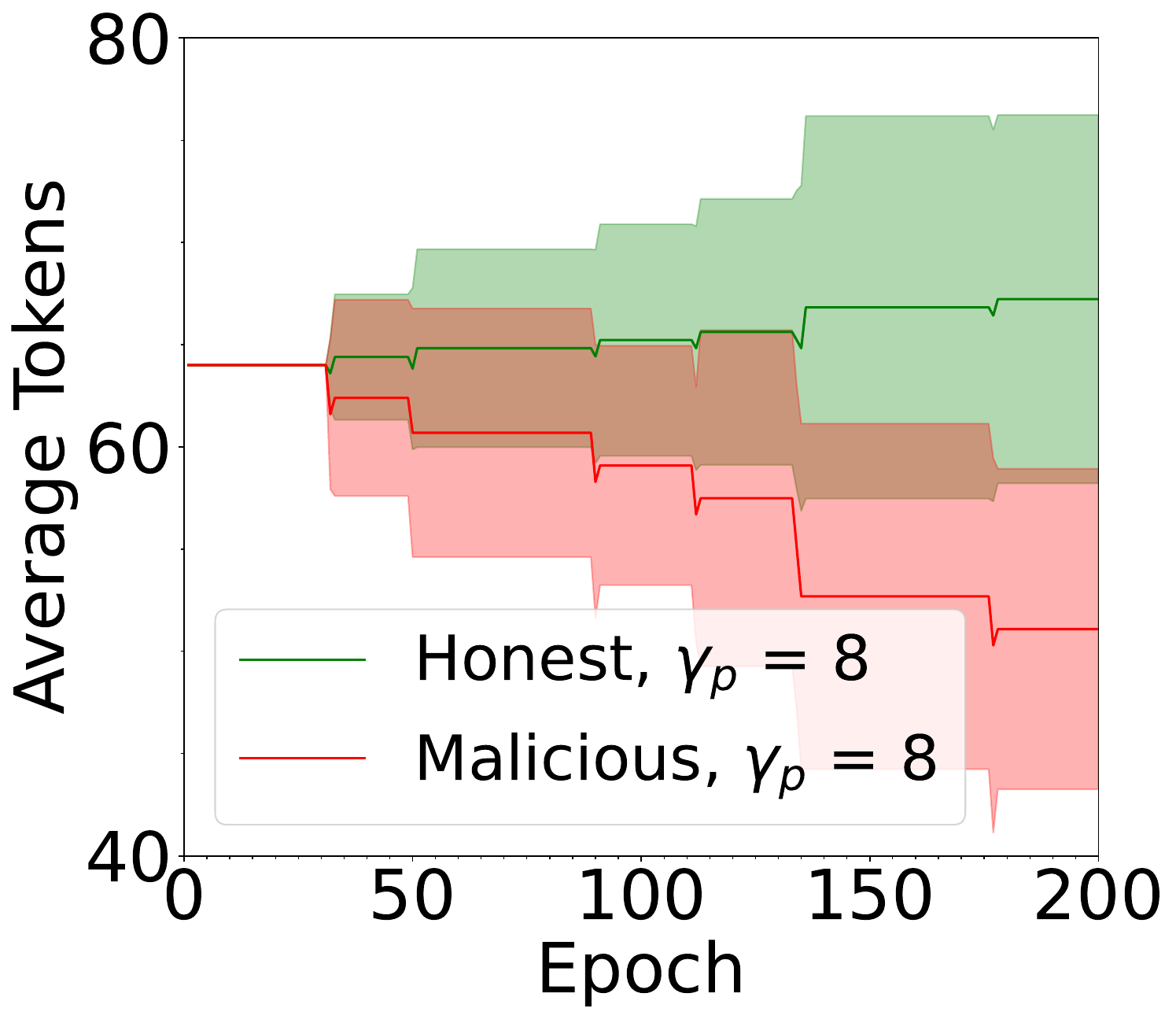}
\label{fig:tokens_with_rate_0.2}
}%
\hfill
\subfigure[$\eta$ = 0.3.]{
\includegraphics[width=0.483\columnwidth]{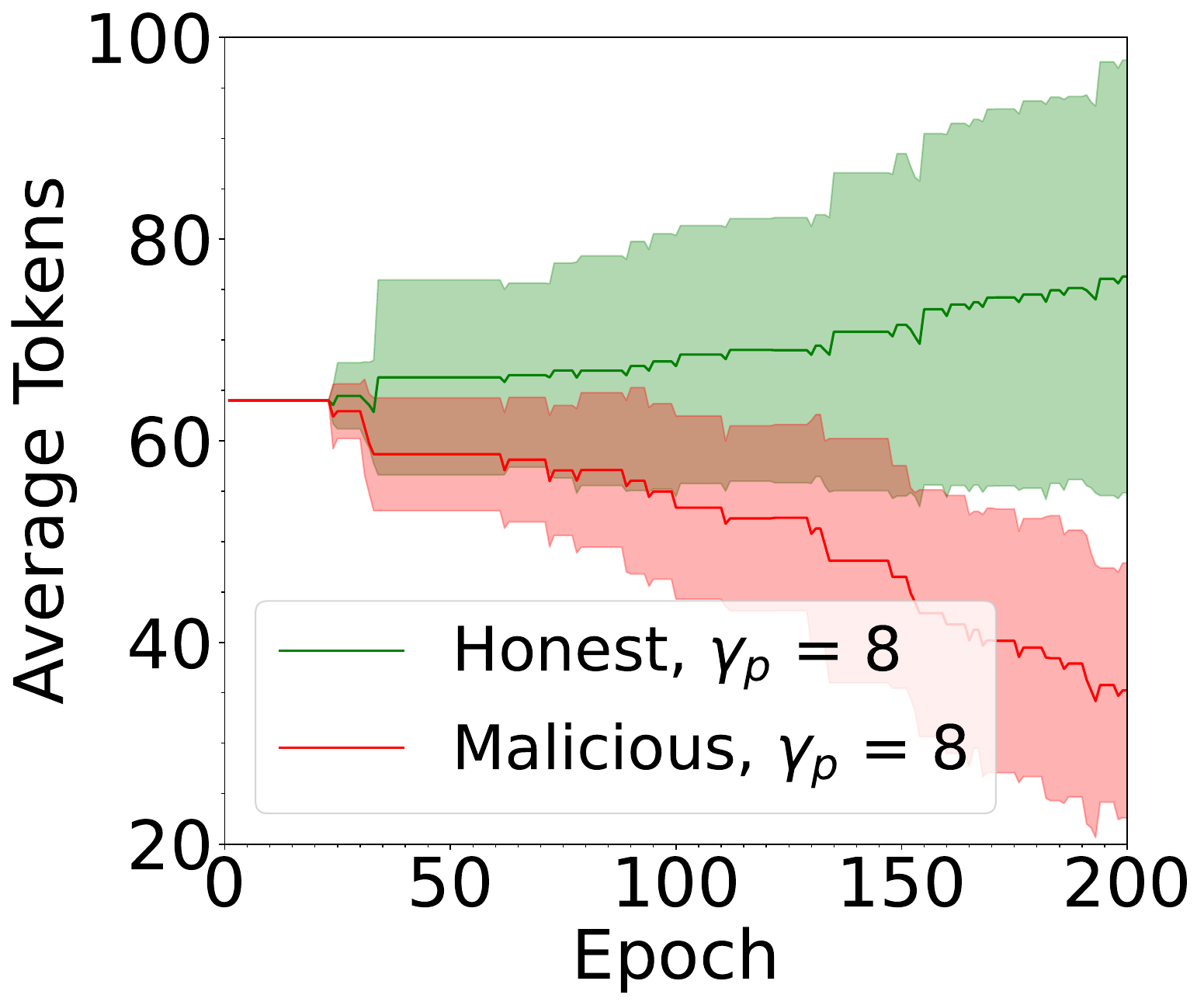}
\label{fig:tokens_with_rate_0.3}
}%
\hfill
\subfigure[$\eta$ = 0.4.]{
\includegraphics[width=0.483\columnwidth]{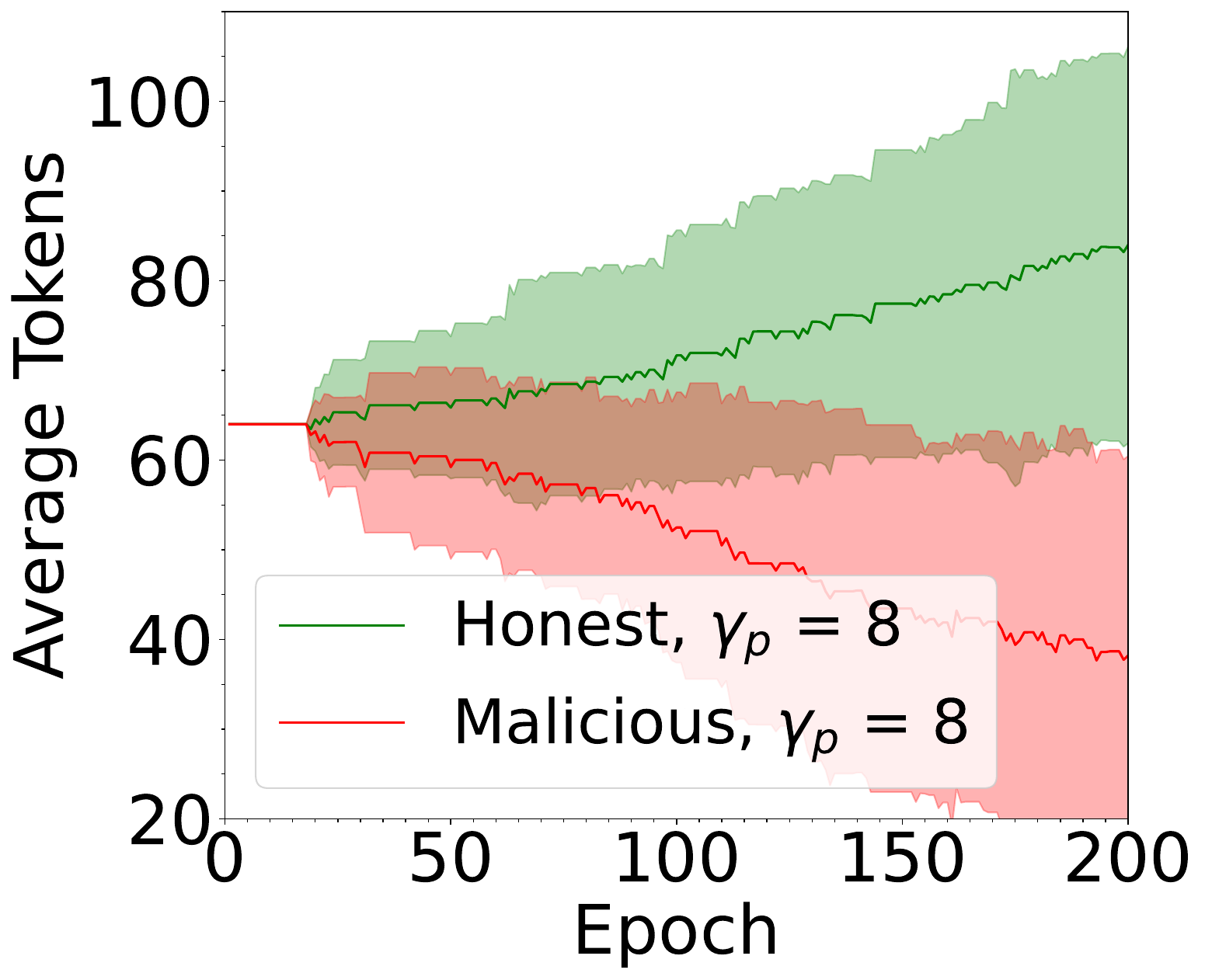}
\label{fig:tokens_with_rate_0.4}
}%
\hfill
\centering
\caption{Token distribution results for clients when setting the parameter for slashing proposers as  $\gamma_p = 8$. The expected average token of malicious proposers fluctuates down during the training process.}

\label{fig:token_distribution_sp=8}
\end{figure*}

\begin{figure*}[h]
\centering
\hfill
\subfigure[$\eta$ = 0.1.]{
\includegraphics[width=0.483\columnwidth]{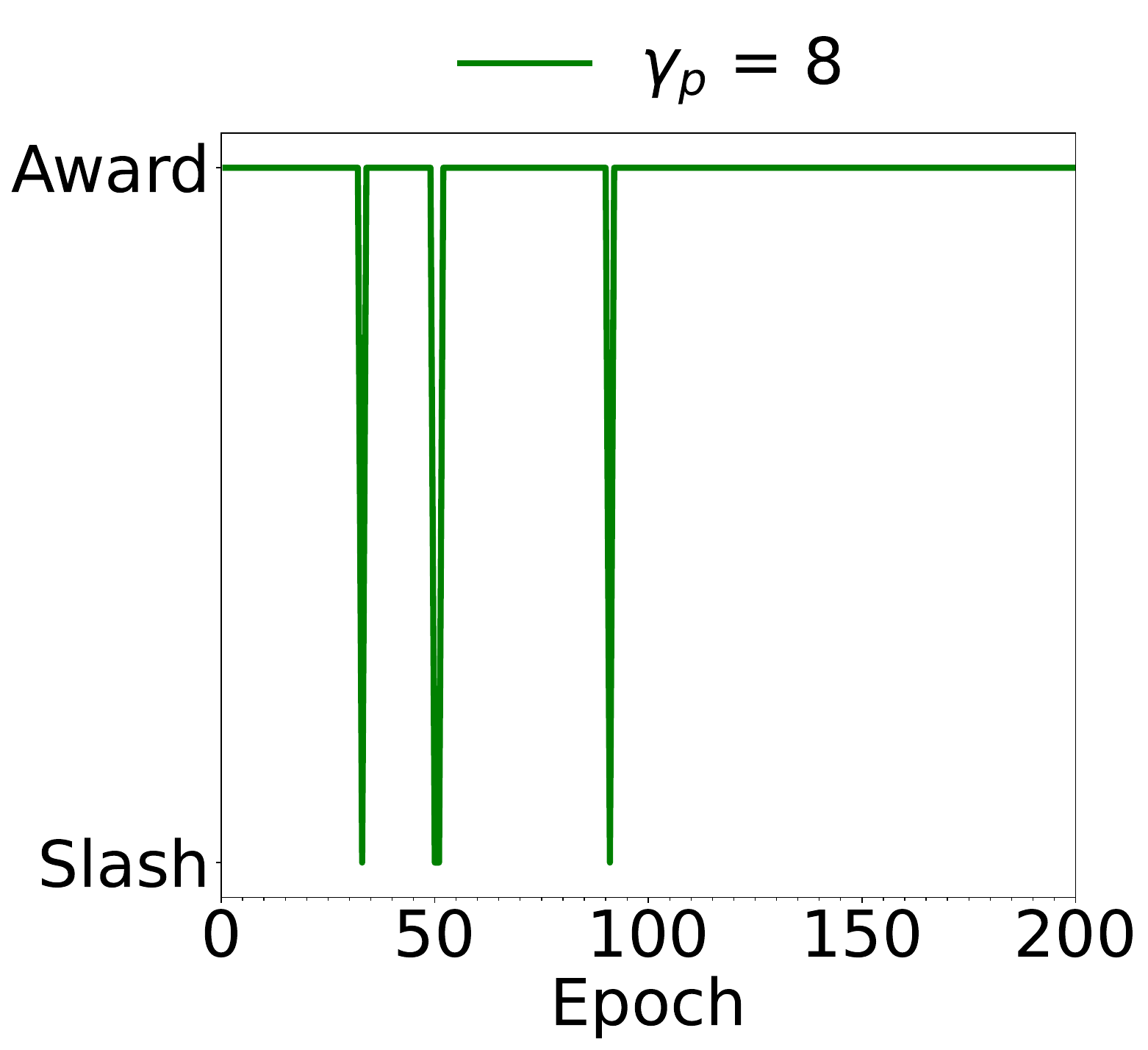}
\label{fig:award_slash_round_with_rate_0.1}
}%
\hfill
\subfigure[$\eta$ = 0.4.]{
\includegraphics[width=0.483\columnwidth]{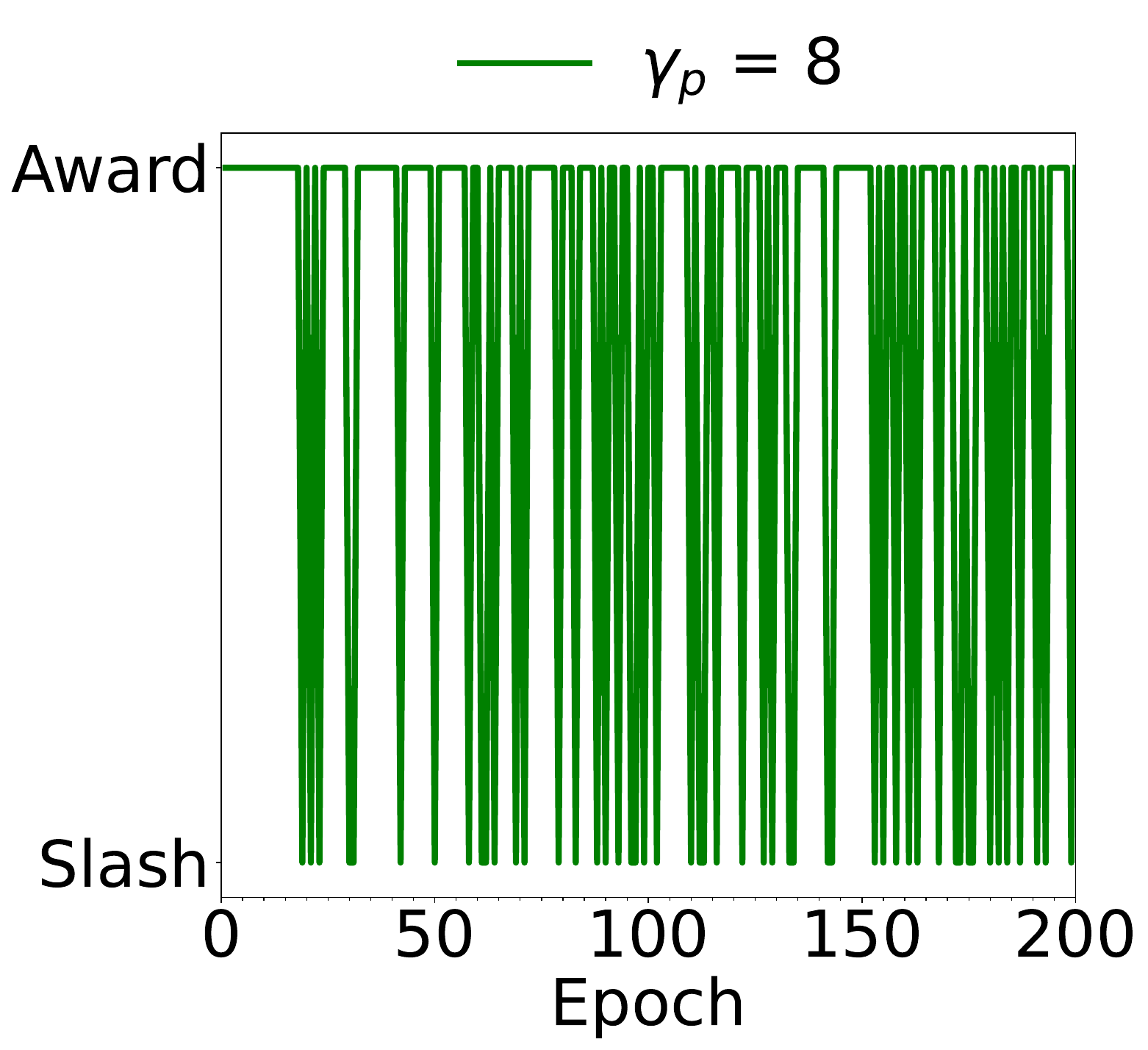}
\label{fig:award_slash_round_with_rate_0.4}
}%
\hfill
\subfigure[$\eta$ = 0.1.]{
\includegraphics[width=0.483\columnwidth]{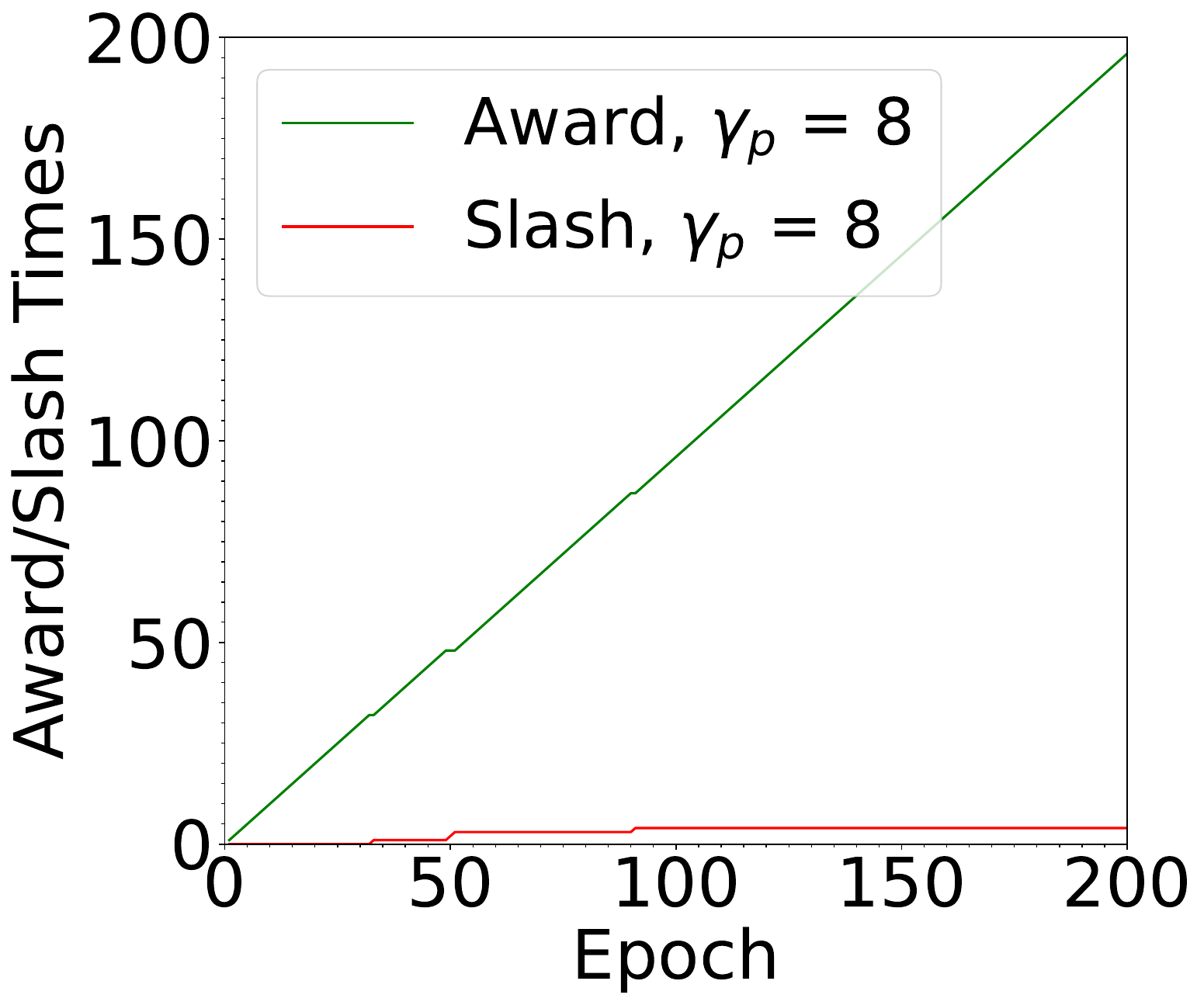}
\label{fig:award_slash_num_with_rate_0.1}
}%
\hfill
\subfigure[$\eta$ = 0.4.]{
\includegraphics[width=0.483\columnwidth]{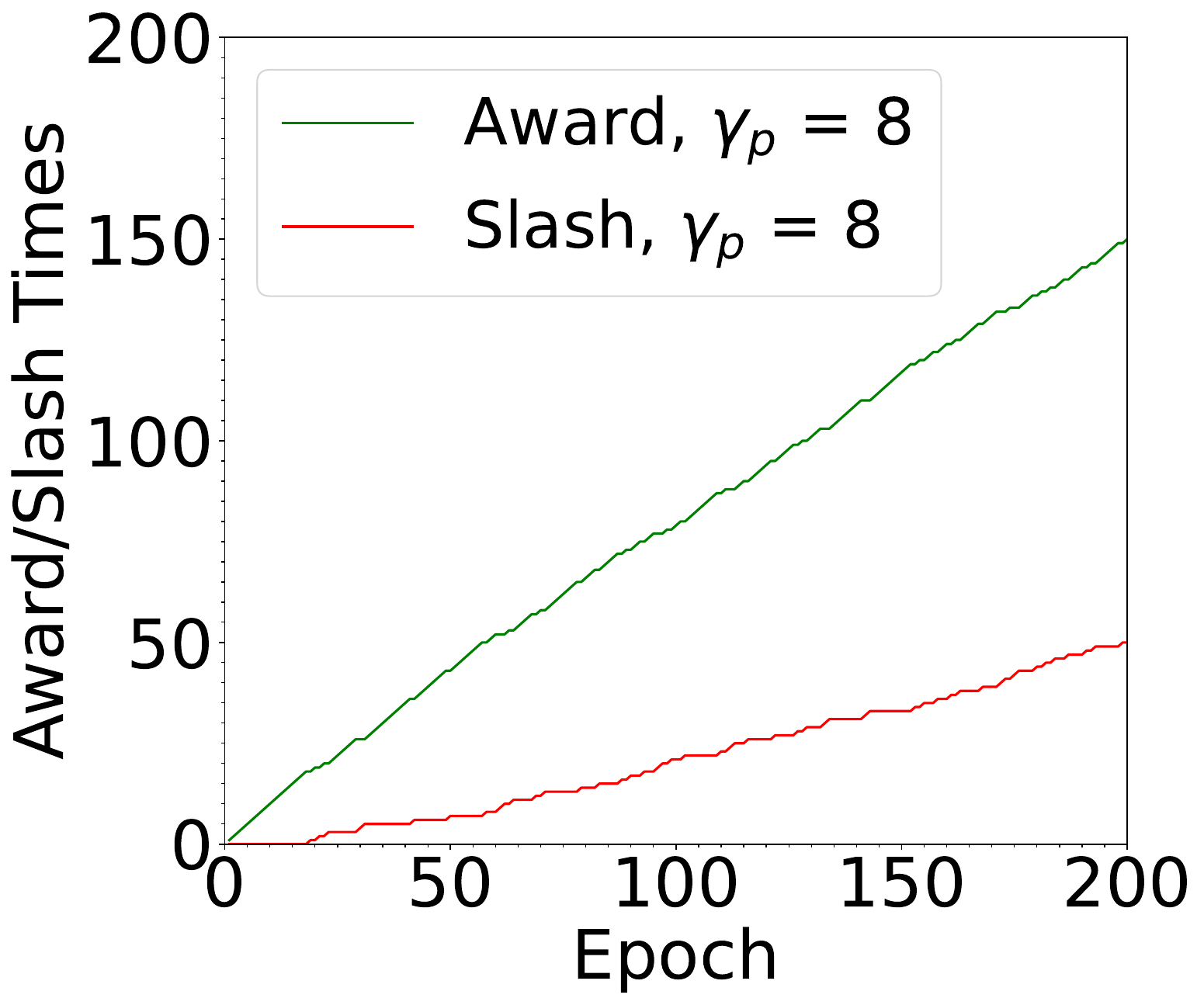}
\label{fig:award_slash_num_with_rate_0.4}
}%
\hfill
\centering
\caption{(a)-(b) Number of award and slash epochs for clients when setting the parameter for slashing proposers as  $\gamma_p = 8$. (c)-(d) The number of slash epochs increases when the number of malicious proposers increases.}
\label{fig:award_slash_num}
\end{figure*}

\begin{figure*}[t]
\centering
\hfill
\subfigure[$\eta$ = 0.1.]{
\includegraphics[width=0.483\columnwidth]{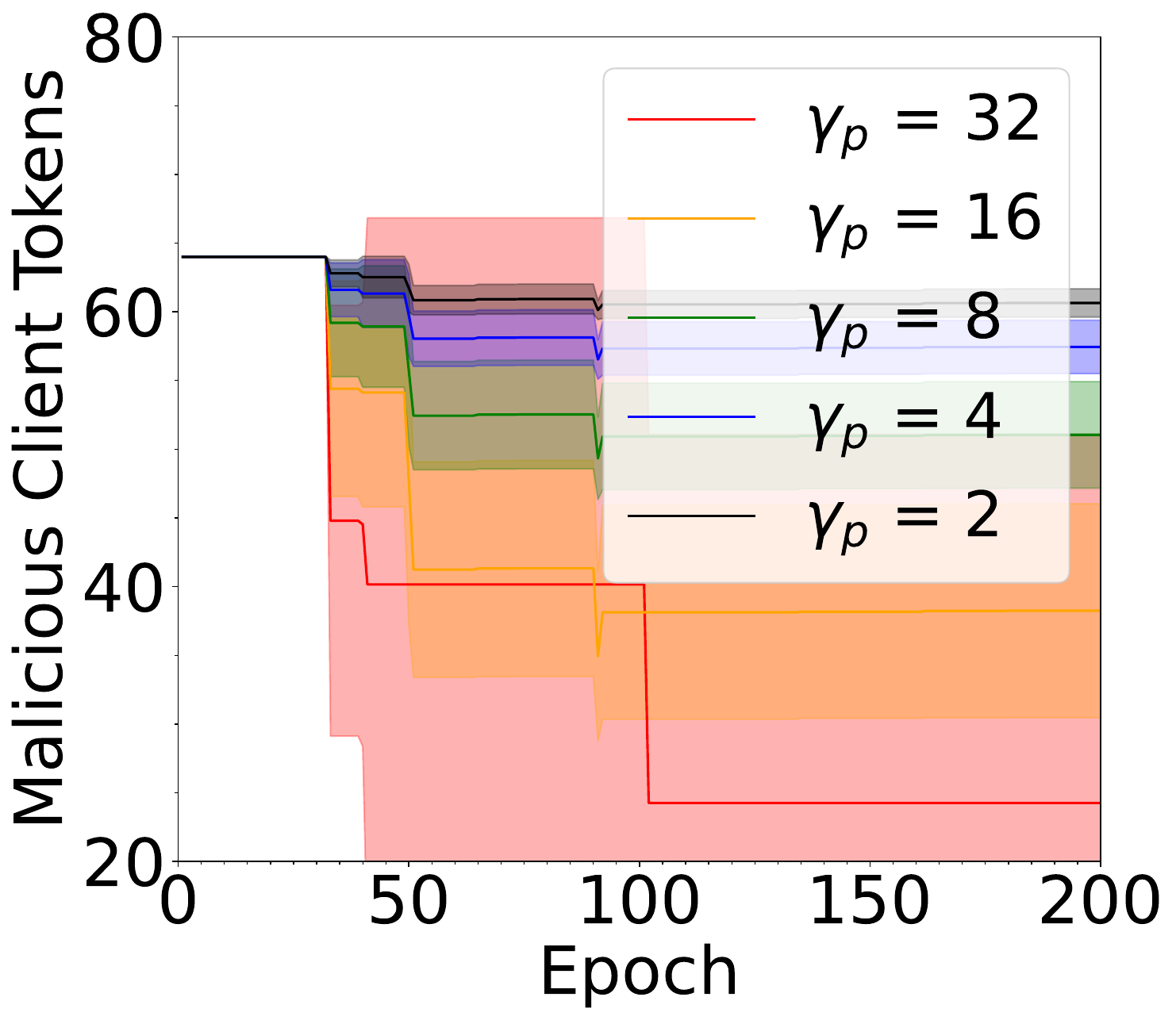}
\label{fig:malicious_node_tokens_with_rate_0.1}
}%
\hfill
\subfigure[$\eta$ = 0.2.]{
\includegraphics[width=0.483\columnwidth]{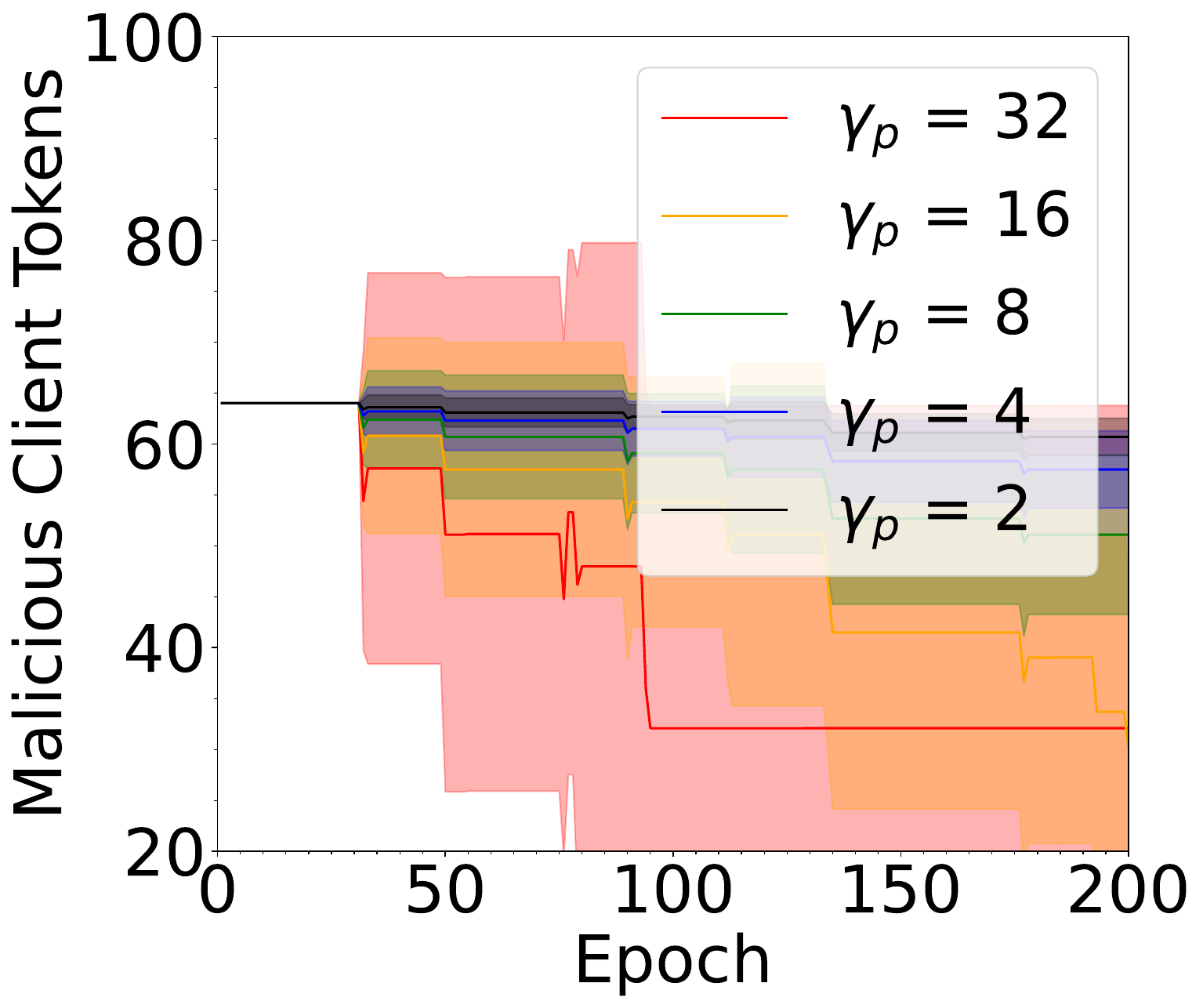}
\label{fig:malicious_node_tokens_with_rate_0.2}
}%
\hfill
\subfigure[$\eta$ = 0.3.]{
\includegraphics[width=0.483\columnwidth]{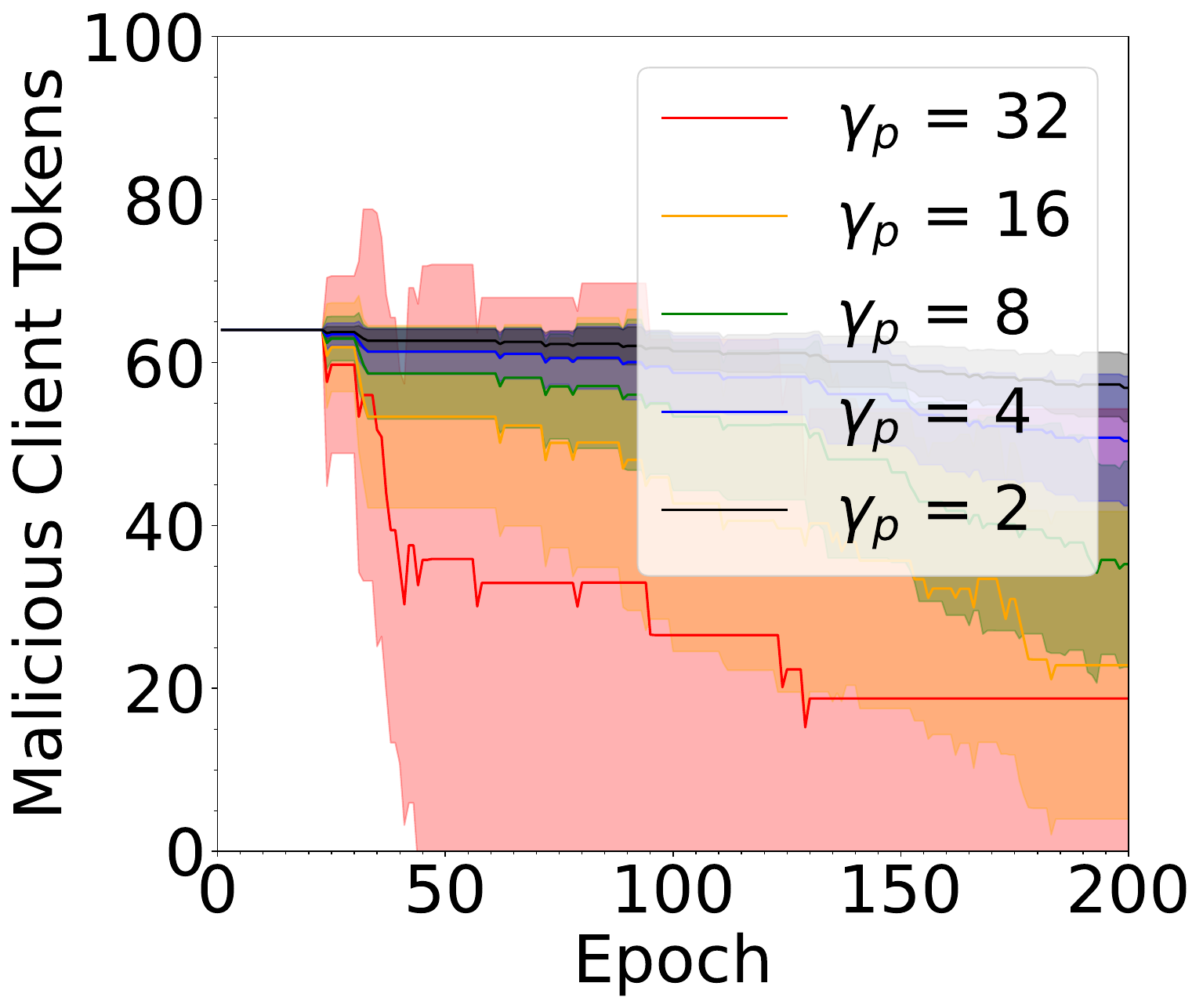}
\label{fig:malicious_node_tokens_with_rate_0.3}
}%
\hfill
\subfigure[$\eta$ = 0.4.]{
\includegraphics[width=0.483\columnwidth]{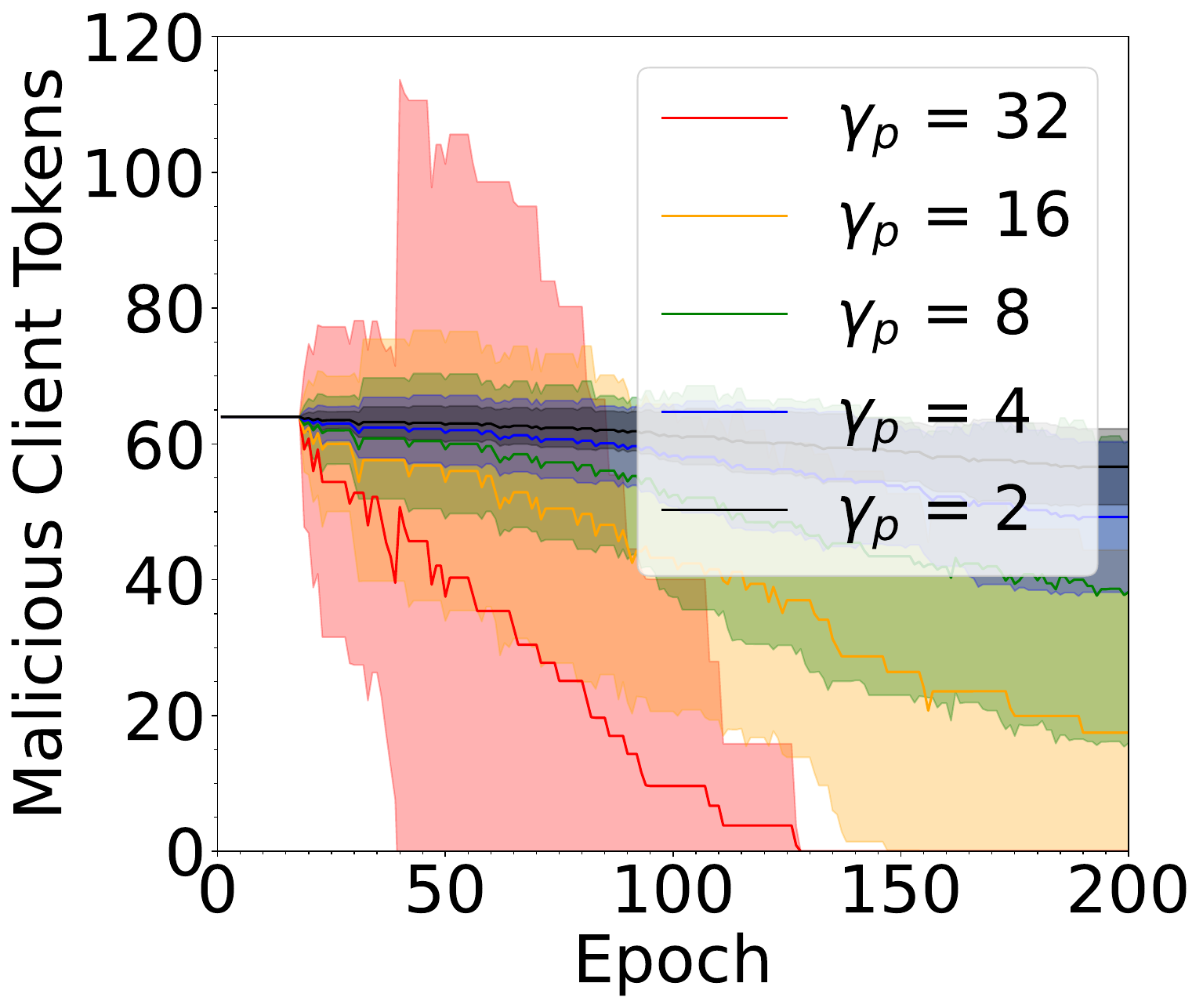}
\label{fig:malicious_node_tokens_with_rate_0.4}
}%
\hfill
\centering
\caption{Token distribution results for malicious clients when choosing $\gamma_p = 2, 4, 8, 16,$ and $ 32$. The expected average token of malicious proposers exhibits a higher rate of decrease when a large value of $\gamma_p$ is selected.}

\label{fig:token_distribution_malicious}
\end{figure*}

\begin{figure*}[t]
\centering
\hfill
\subfigure[$\eta$ = 0.1.]{
\includegraphics[width=0.483\columnwidth]{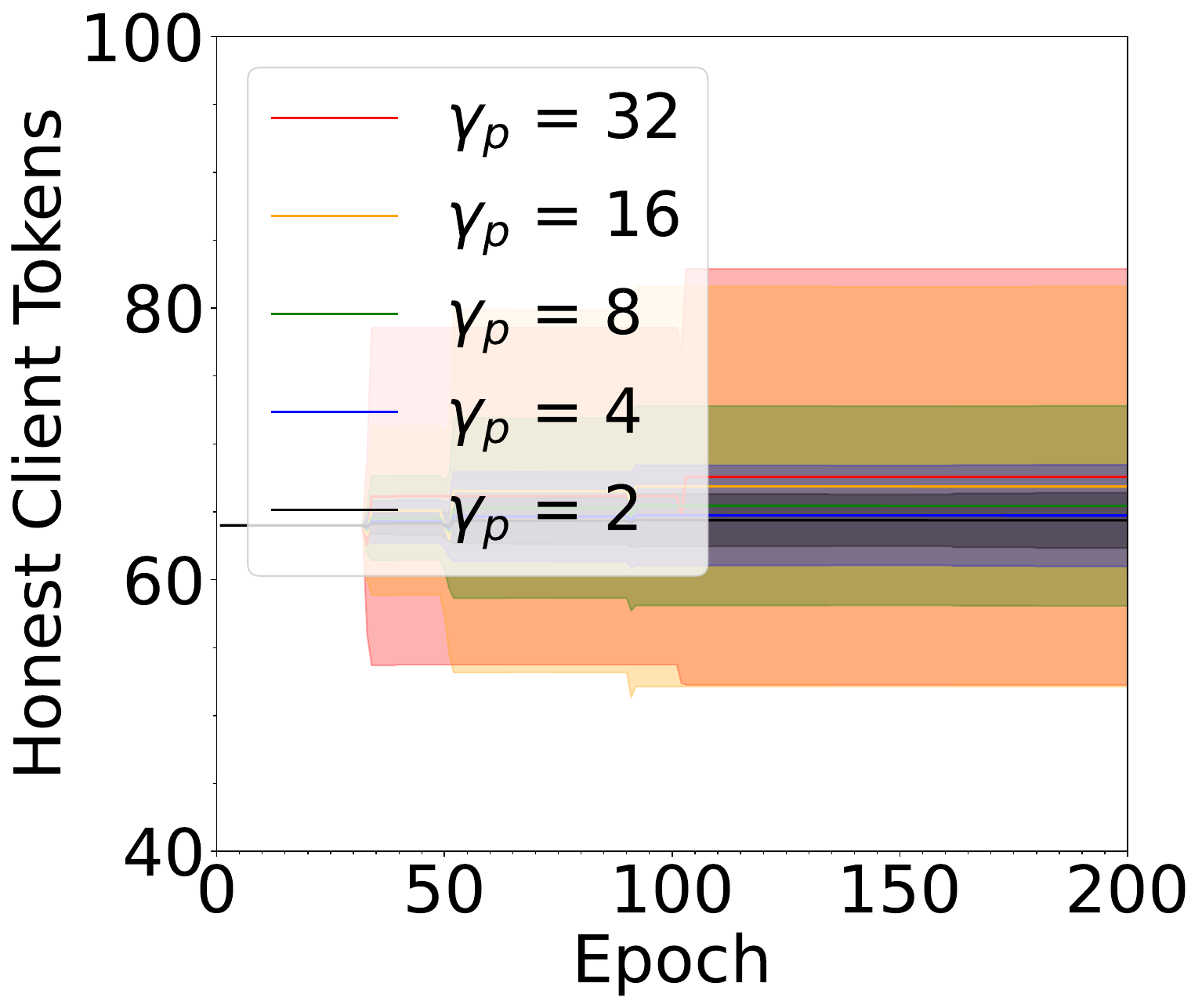}
\label{fig:normal_node_tokens_with_rate_0.1}
}%
\hfill
\subfigure[$\eta$ = 0.2.]{
\includegraphics[width=0.483\columnwidth]{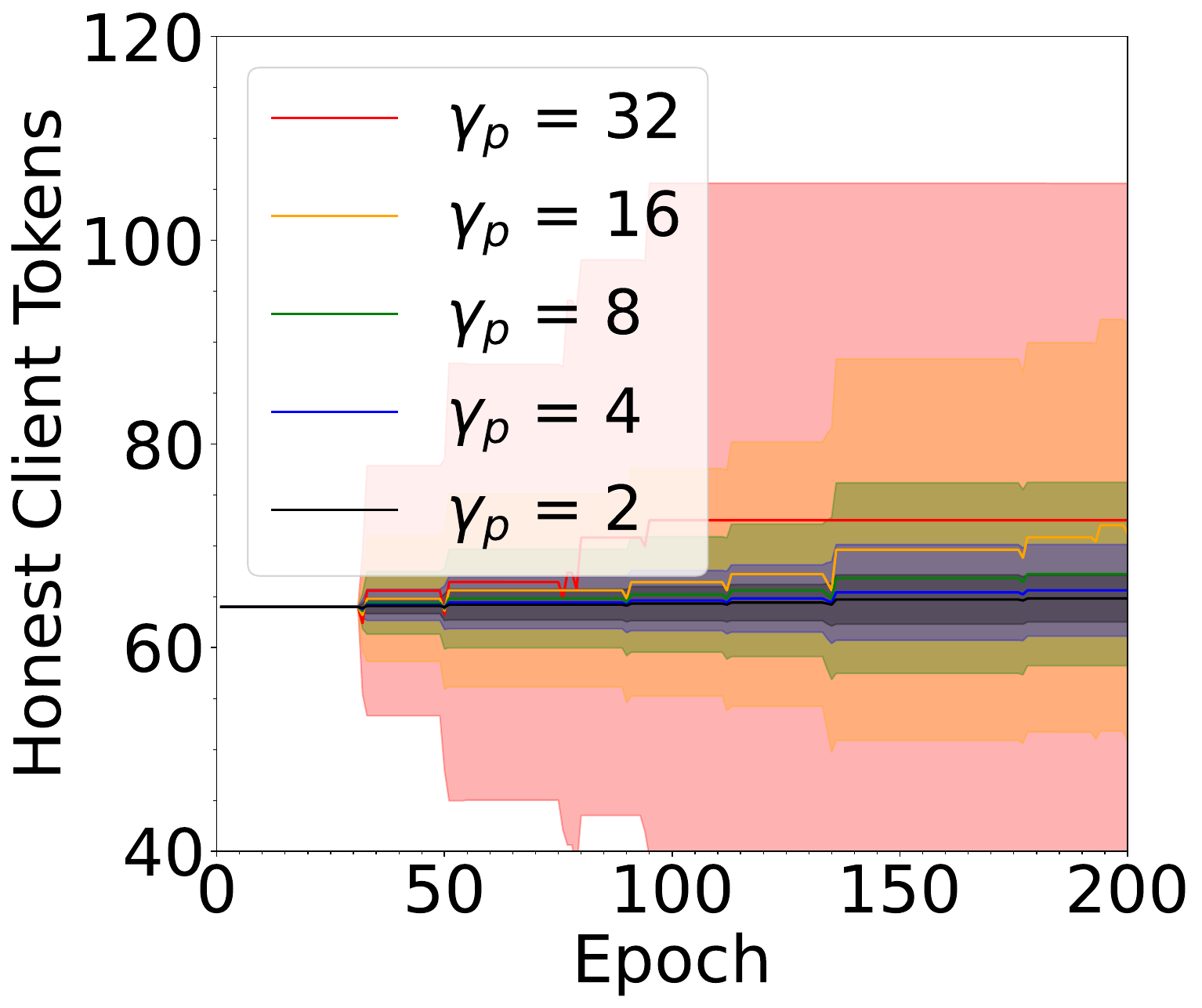}
\label{fig:normal_node_tokens_with_rate_0.2}
}%
\hfill
\subfigure[$\eta$ = 0.3.]{
\includegraphics[width=0.483\columnwidth]{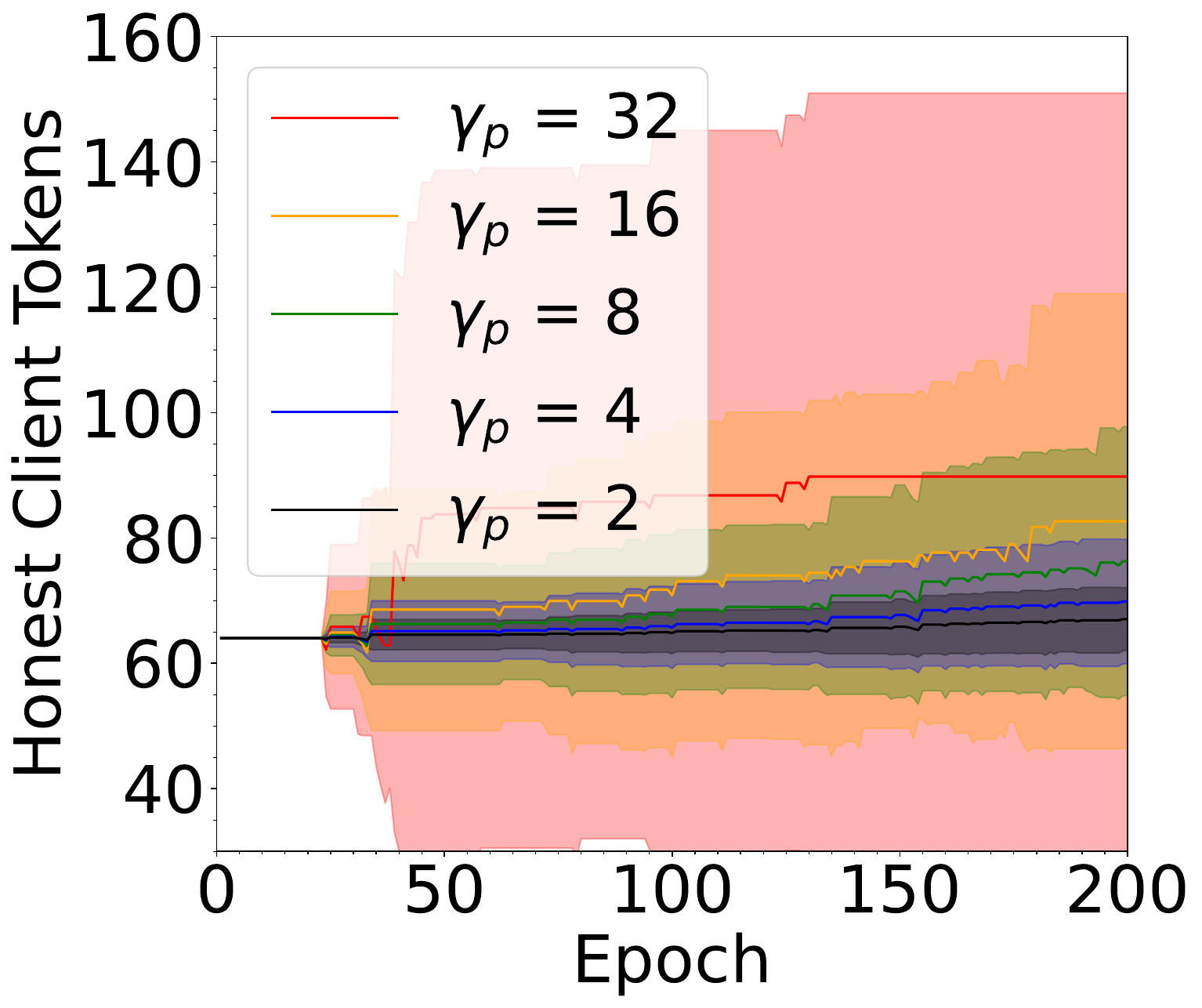}
\label{fig:normal_node_tokens_with_rate_0.3}
}%
\hfill
\subfigure[$\eta$ = 0.4.]{
\includegraphics[width=0.483\columnwidth]{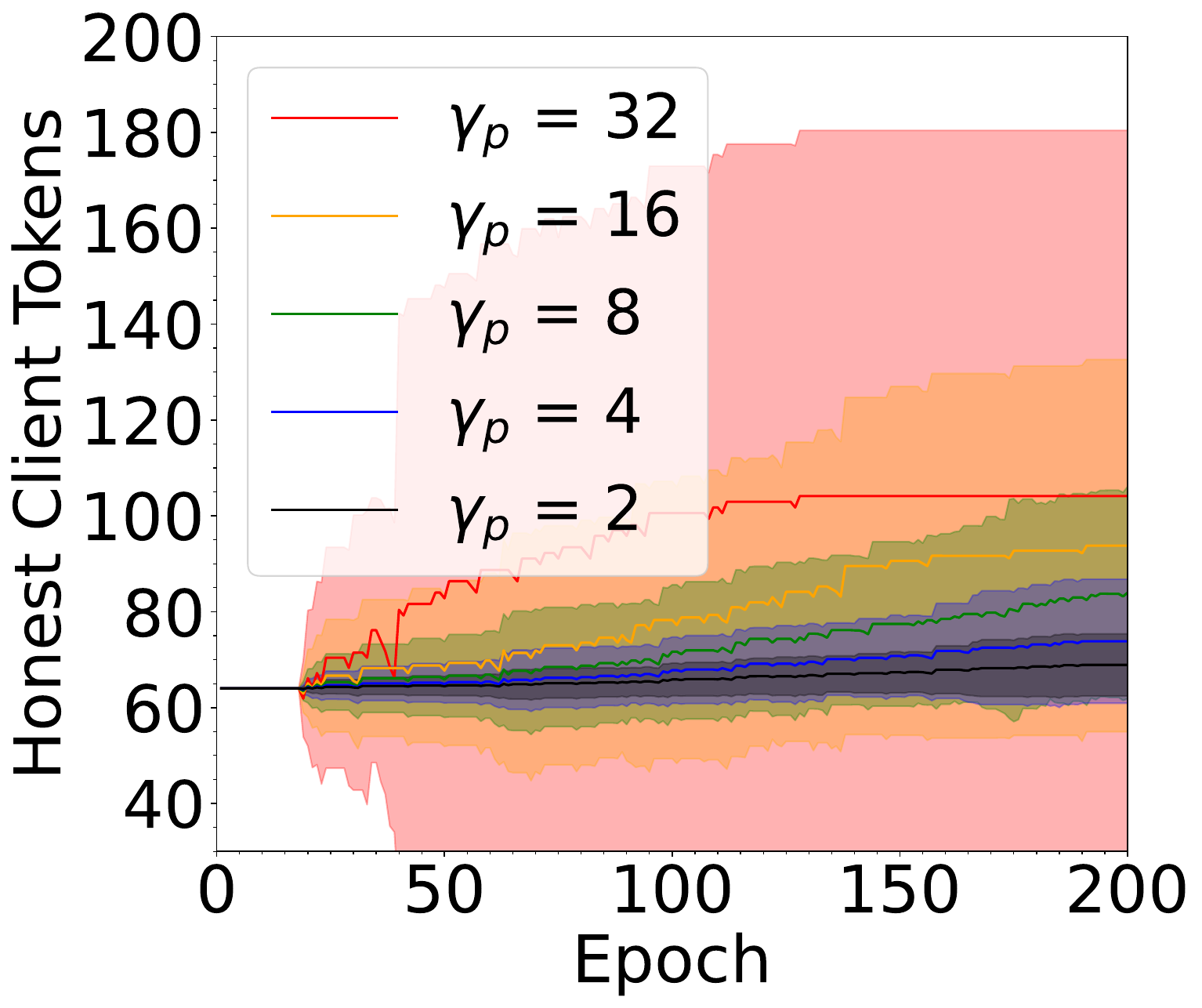}
\label{fig:normal_node_tokens_with_rate_0.4}
}%
\hfill
\centering
\caption{Token distribution results for honest clients when choosing $\gamma_p = 2, 4, 8, 16,$ and $ 32$. The expected average token of honest proposers displays a higher rate of growth when a large value of $\gamma_p$ is selected.}

\label{fig:token_distribution_honest}
\end{figure*}

\subsection{Results}
\subsubsection{Empirical Analysis on Malicious Voters}
\label{sec:exp:thm}
To empirically validate the theoretical result in Sec.~\ref{sec:method:analysis}, we first simulate a hypothetical scenario where there are only honest proposers. As there are more honest proposers than malicious proposers at each round on average, the effect of malicious weights can be seen as slowing the convergence and decreasing the global performance, which will be validated in Sec.~\ref{sec:exp:base}. Here, we further simplify the scenario to focus on the behavior of malicious voters. As shown in Fig.~\ref{fig:mal_voters_token_distribution_sv=4}, given the set of the hyperparameter for slashing voters $\gamma_r = \{2, 4, 8, 16, 32\}$, the malicious voters will be eliminated from the system shortly (\ie~their average tokens decline to $0$ within $\approx 40$ epochs).

\begin{figure*}[htbp]
\centering
\hfill
\subfigure[$\eta$ = 0.1.]{
\includegraphics[width=0.483\columnwidth]{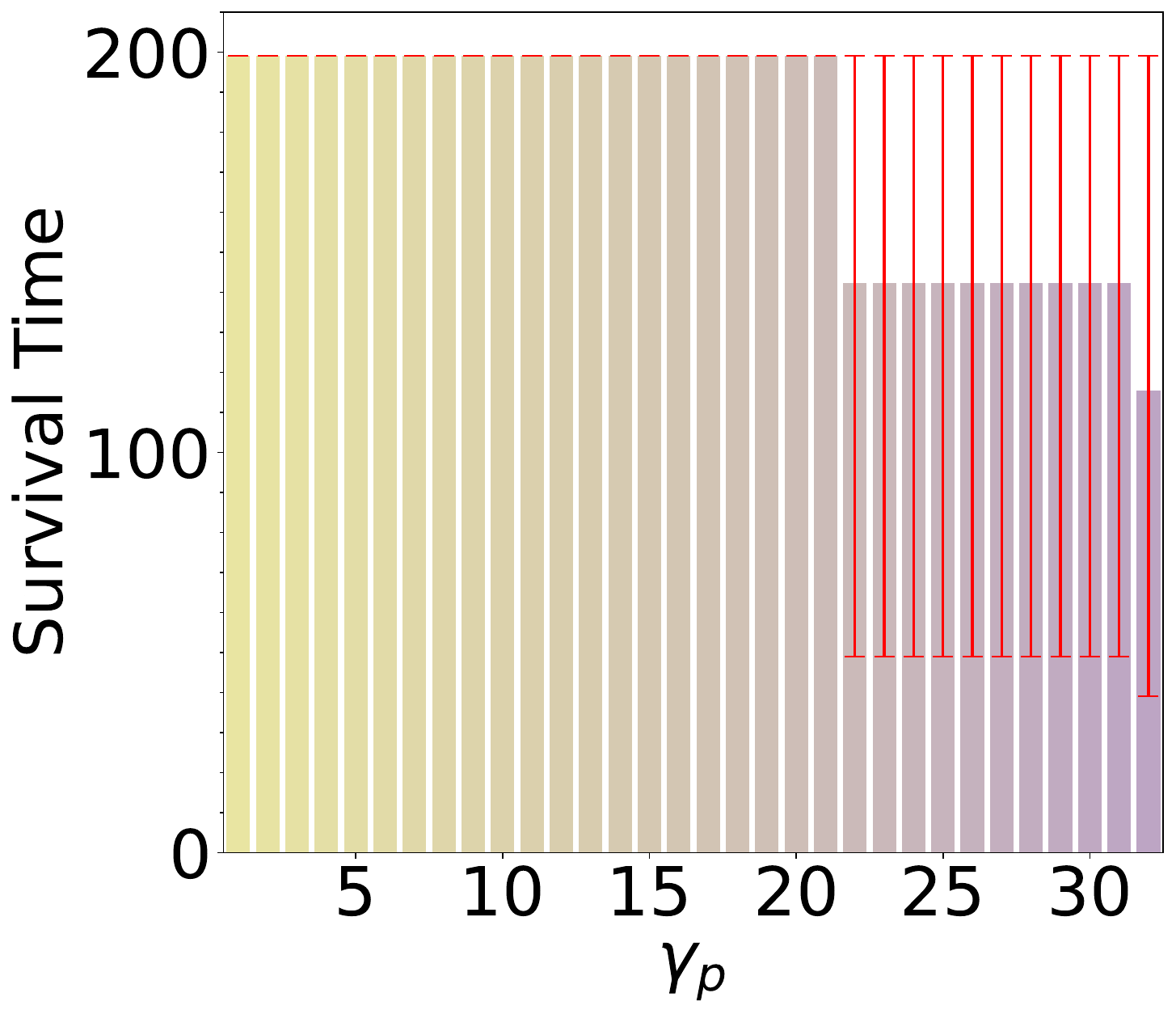}
\label{fig:malicious_client_liveness_with_rate_0.1}
}%
\hfill
\subfigure[$\eta$ = 0.2.]{
\includegraphics[width=0.483\columnwidth]{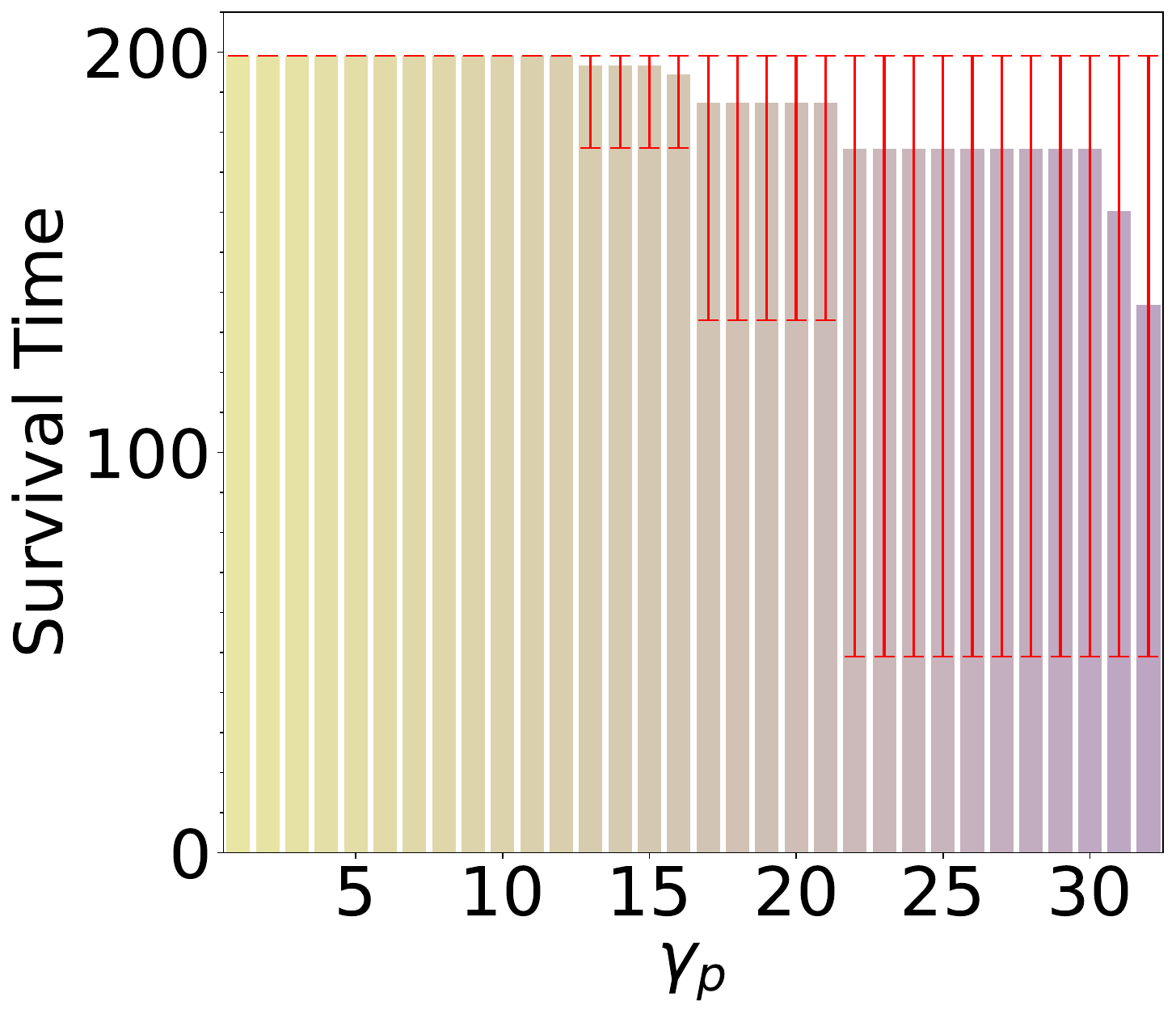}
\label{fig:malicious_client_liveness_with_rate_0.2}
}%
\hfill
\subfigure[$\eta$ = 0.3.]{
\includegraphics[width=0.483\columnwidth]{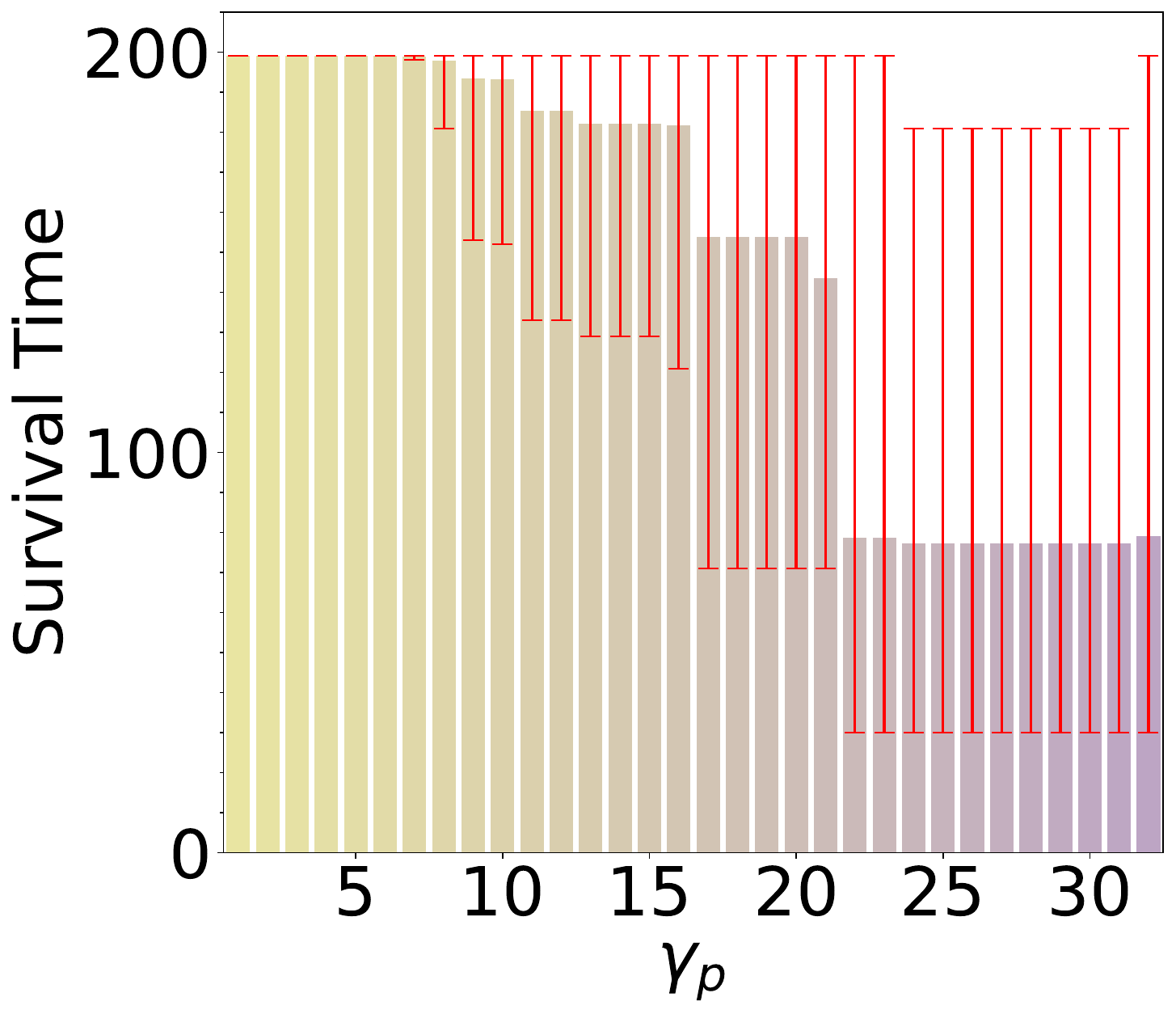}
\label{fig:malicious_client_liveness_with_rate_0.3}
}%
\hfill
\subfigure[$\eta$ = 0.4.]{
\includegraphics[width=0.483\columnwidth]{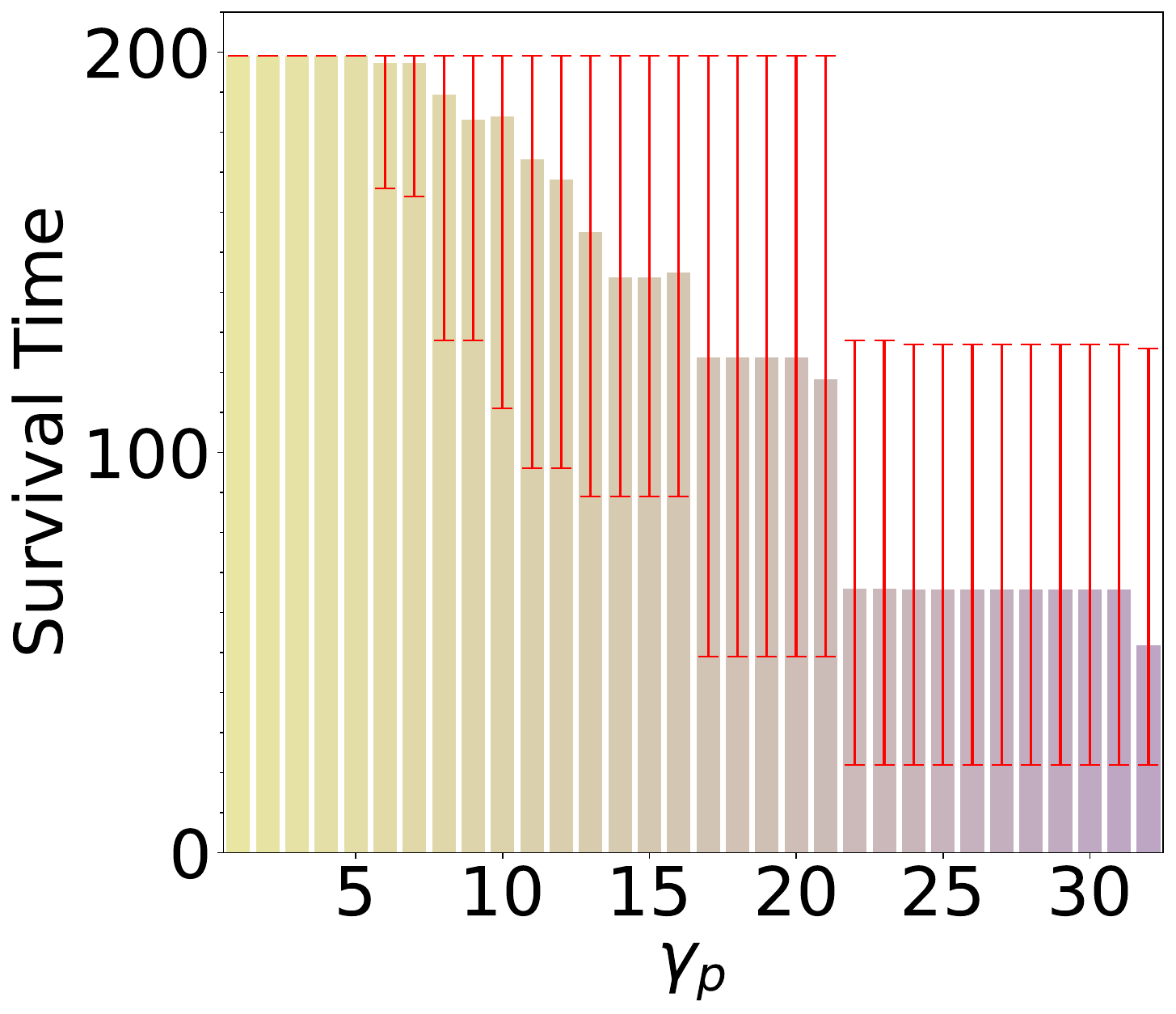}
\label{fig:malicious_client_liveness_with_rate_0.4}
}%
\hfill
\centering
\caption{Malicious proposers survival time with various $\gamma_p$. The expected survival time of malicious proposers declines as $\gamma_p$ increases.}

\label{fig:Malicious_client_liveness}
\end{figure*}

\begin{figure*}[htbp]
\centering
\hfill
\subfigure[$\eta$ = 0.1.]{
\includegraphics[width=0.483\columnwidth]{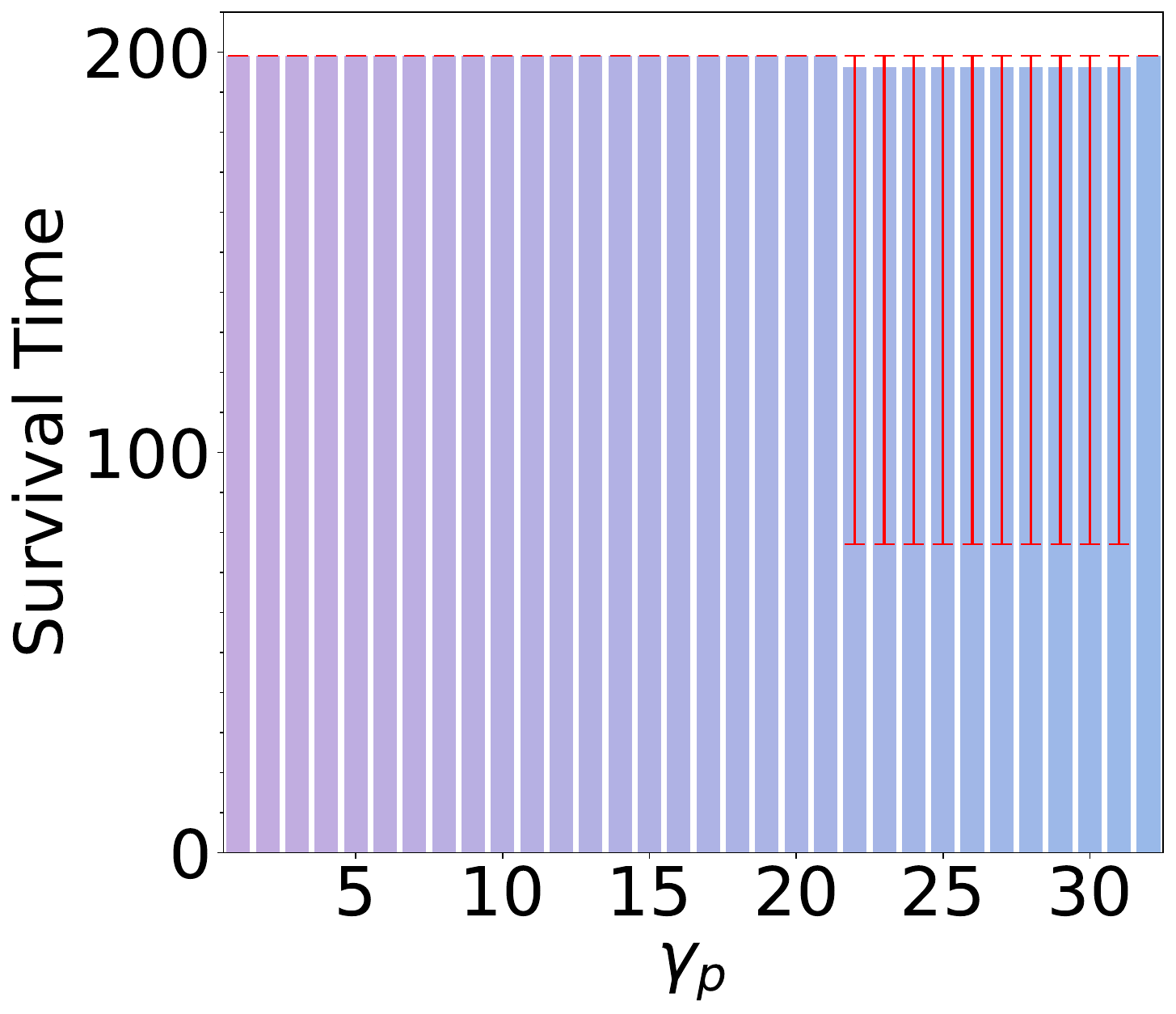}
\label{fig:honest_client_liveness_with_rate_0.1}
}%
\hfill
\subfigure[$\eta$ = 0.2.]{
\includegraphics[width=0.483\columnwidth]{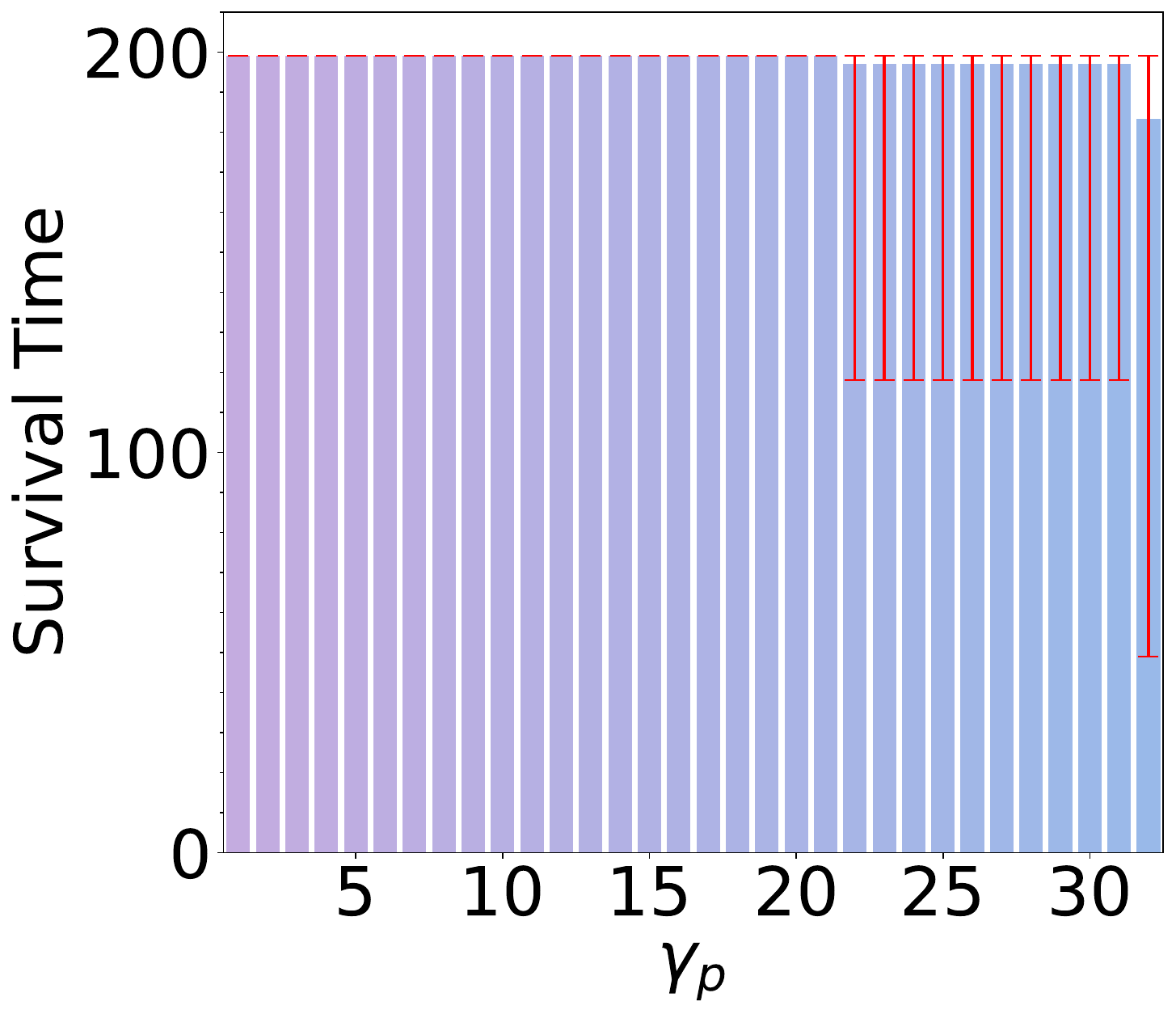}
\label{fig:honest_client_liveness_with_rate_0.2}
}%
\hfill
\subfigure[$\eta$ = 0.3.]{
\includegraphics[width=0.483\columnwidth]{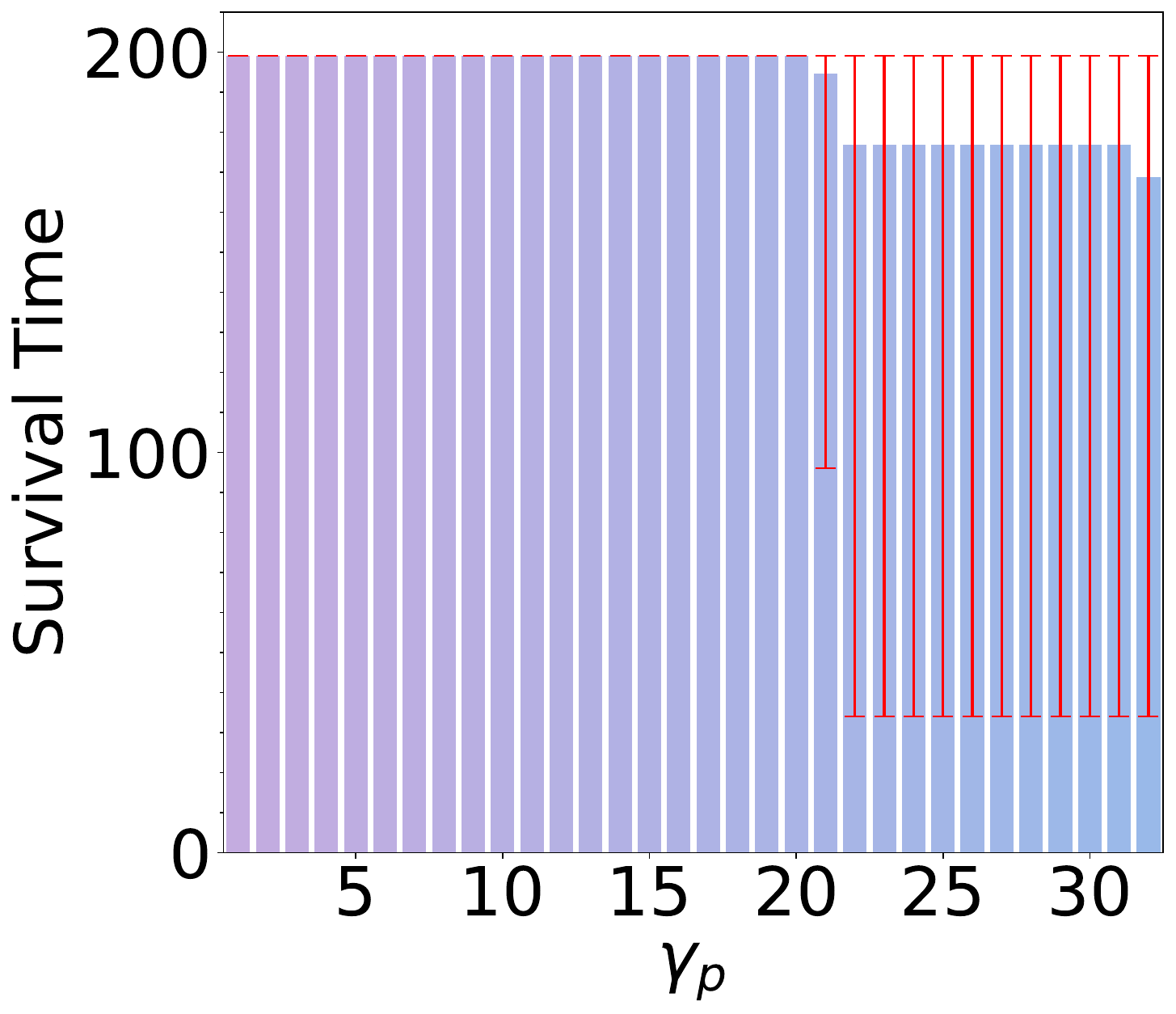}
\label{fig:honest_client_liveness_with_rate_0.3}
}%
\hfill
\subfigure[$\eta$ = 0.4.]{
\includegraphics[width=0.483\columnwidth]{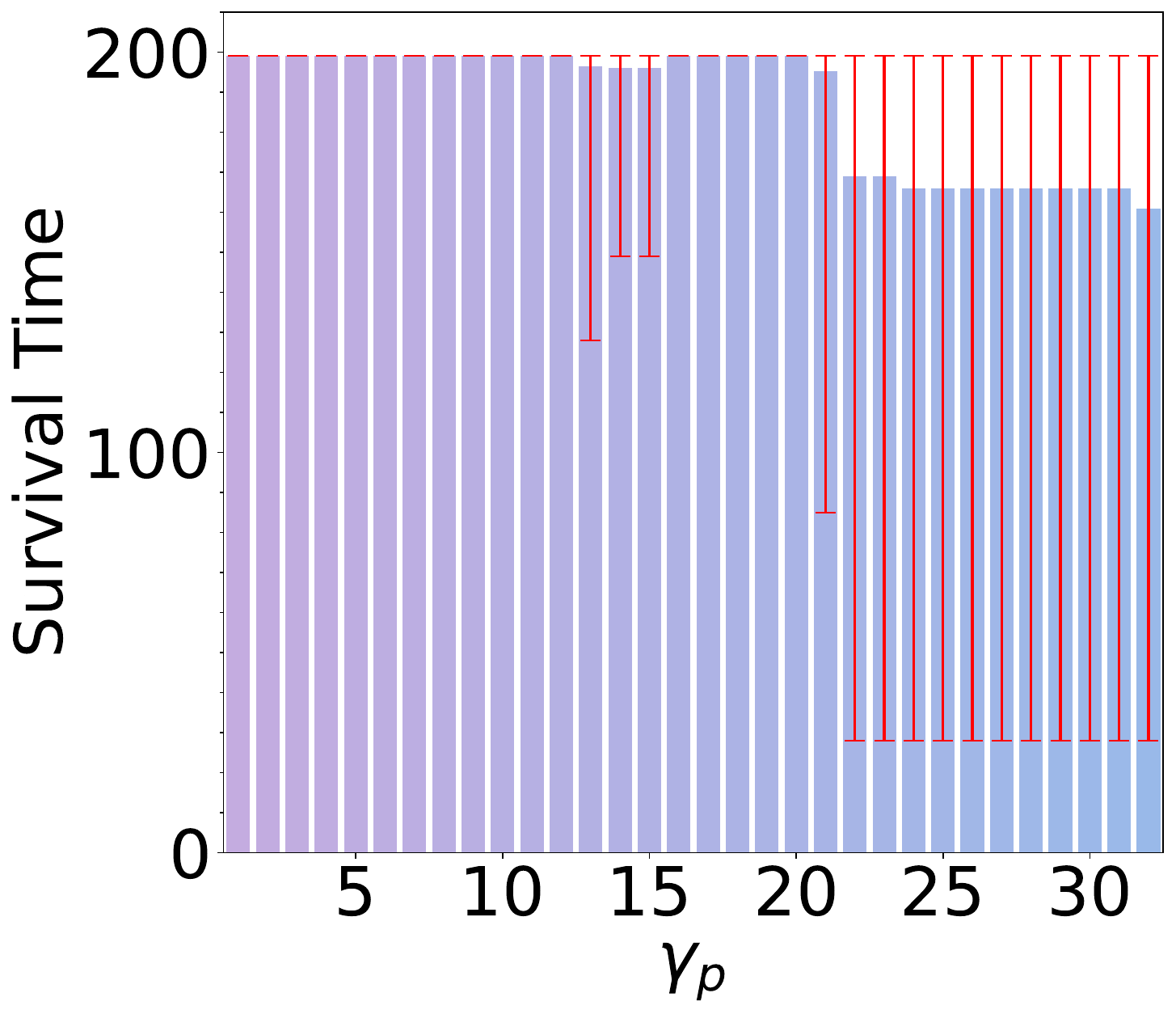}
\label{fig:honest_client_liveness_with_rate_0.4}
}%
\hfill
\centering
\caption{Honest proposers survival time with various $\gamma_p$. A large $\gamma_p$ can also decrease the expected survival time of honest proposers when the malicious rate $\eta$ is large. This is because, in each epoch, the randomly selected proposers will be all slashed when the performance of the aggregated global model does not increase.}

\label{fig:Honest_client_liveness}
\end{figure*}

\begin{figure*}[htbp]
\centering
\hfill
\subfigure[$\eta$ = 0.1.]{
\includegraphics[width=0.483\columnwidth]{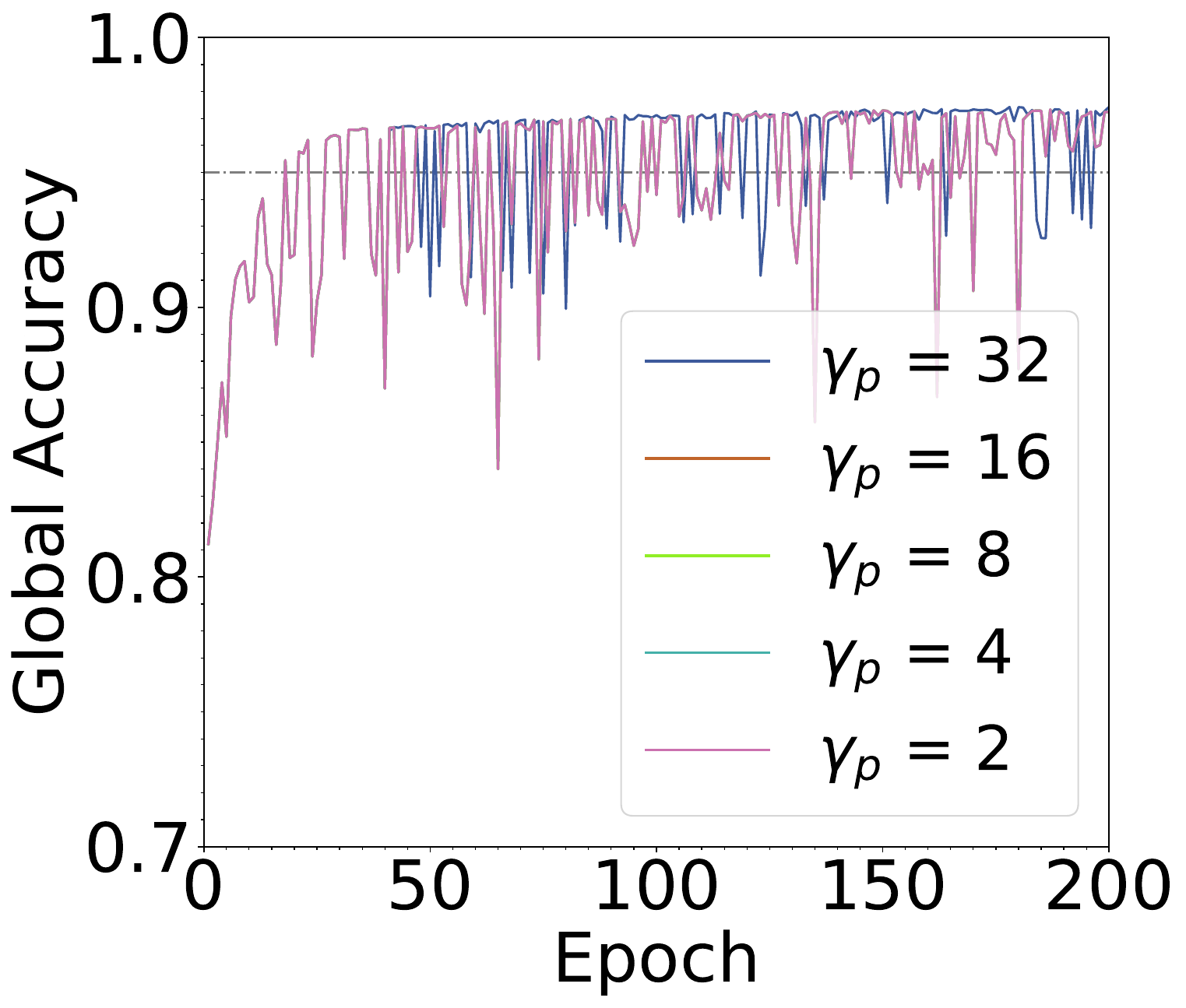}
\label{fig:global_acc_with_rate_0.1}
}%
\hfill
\subfigure[$\eta$ = 0.2.]{
\includegraphics[width=0.483\columnwidth]{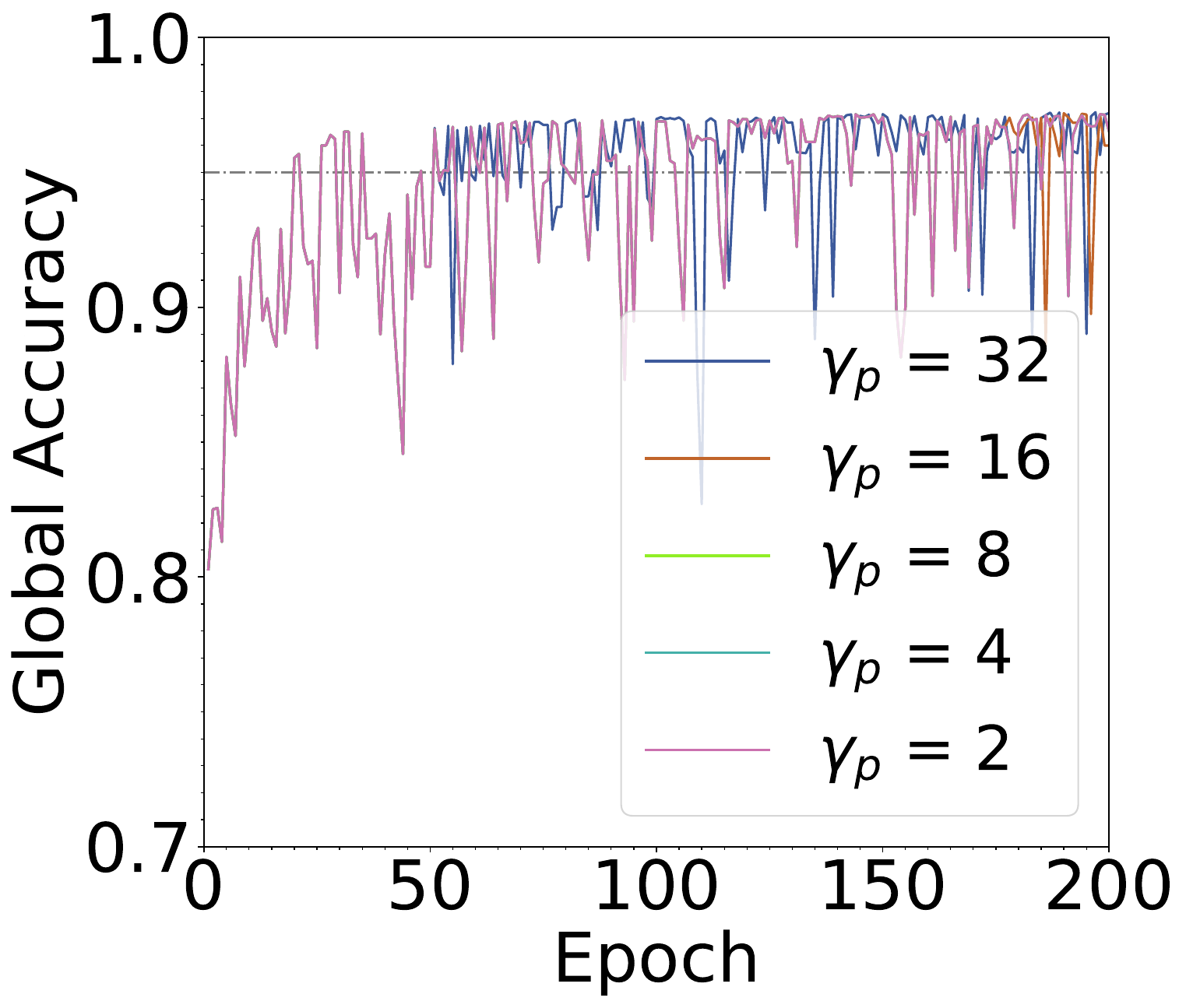}
\label{fig:global_acc_with_rate_0.2}
}%
\hfill
\subfigure[$\eta$ = 0.3.]{
\includegraphics[width=0.483\columnwidth]{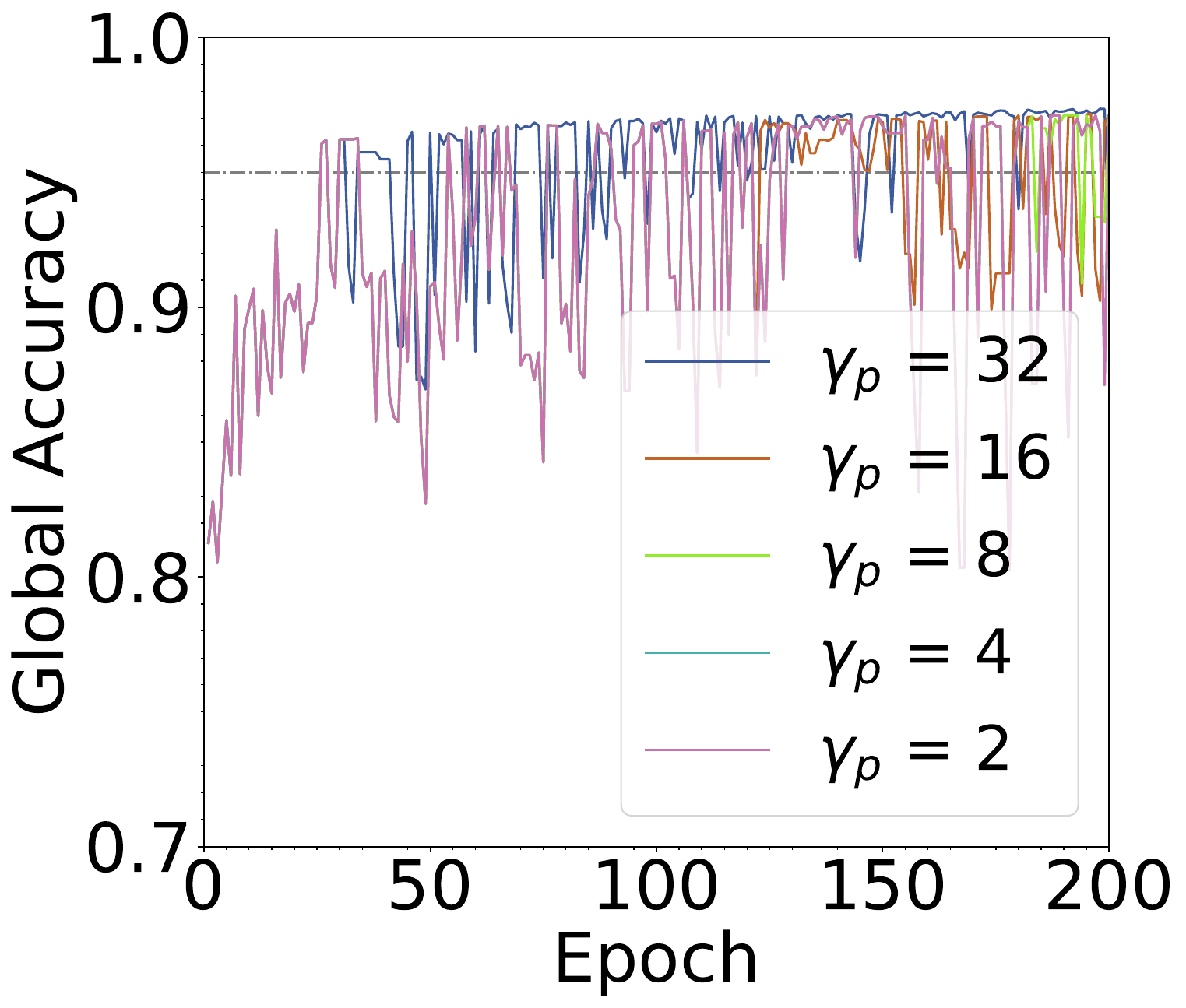}
\label{fig:global_acc_with_rate_0.3}
}%
\hfill
\subfigure[$\eta$ = 0.4.]{
\includegraphics[width=0.483\columnwidth]{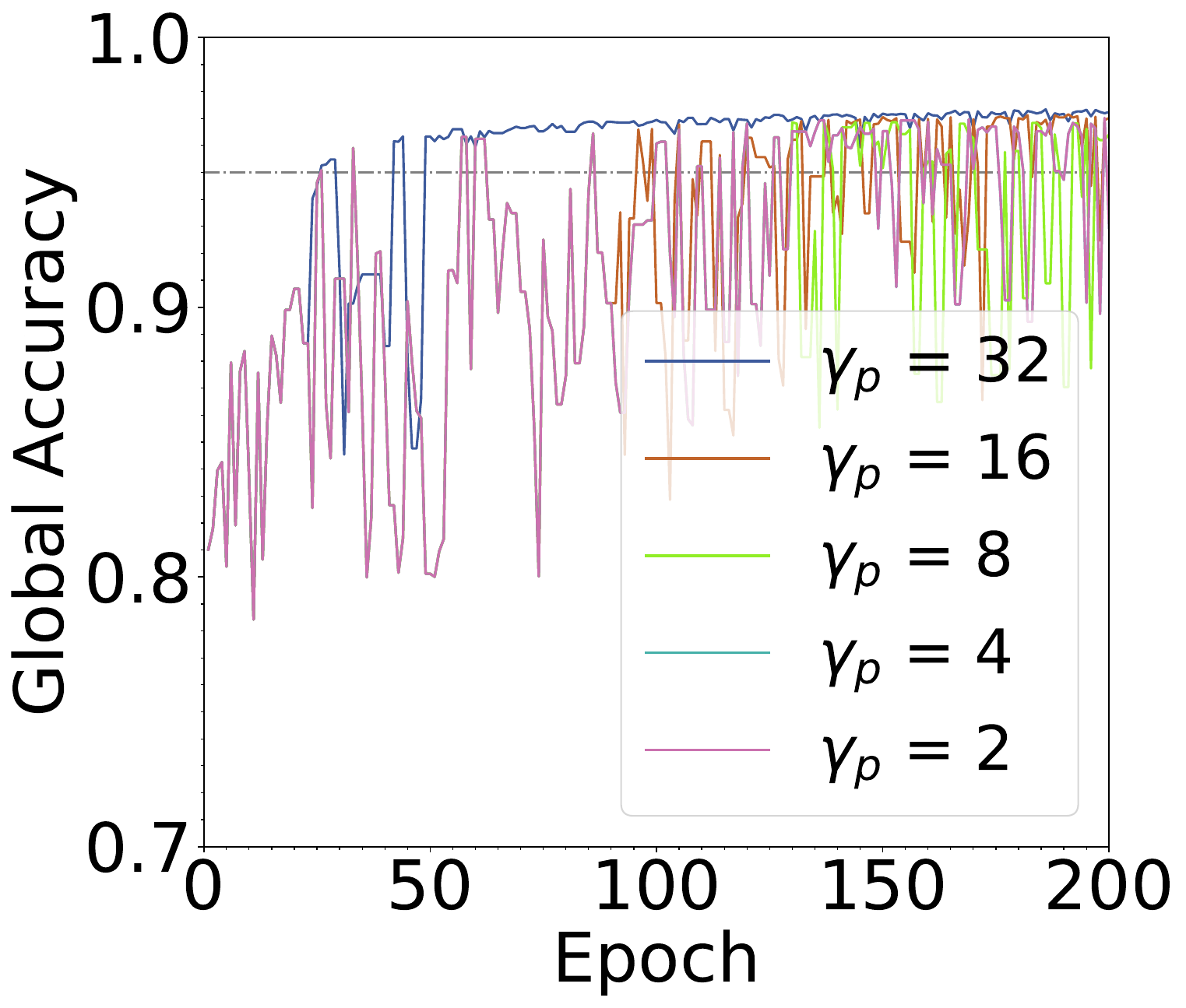}
\label{fig:global_acc_with_rate_0.4}
}%
\hfill
\centering
\caption{Global accuracy results when choosing $\gamma_p = 2, 4, 8, 16,$ and $ 32$. The global accuracy is not sensitive to the value of $\gamma_p$ but the convergence tasks more epochs for larger $\eta$.}
\label{fig:Global_accuracy_sp}
\end{figure*}

\subsubsection{Comparison with Baselines} 
\label{sec:exp:base_result}
Following Theorem~\ref{thm:vote} and Sec.~\ref{sec:exp:thm}, we now are certain that there will be no \textit{de facto} malicious voters. Thus, in the following experiments, we focus on the scenarios where malicious clients only upload harmful weights but make honest votes. We evaluate the proposed framework against the baselines described in Sec.~\ref{sec:exp:base}. We provide the learning curves in Fig.~\ref{fig:training_res} and the accuracy for all four approaches after convergence (the mean accuracy of the last 50 epochs) in Tab.~\ref{tab:training_res}. The performance of FedAVG w/ block is competitive with FedAVG w/o mal (\ie~$\eta = 0$) and consistently outperforms FedAVG w/ mal.  As $\eta$ increases, the performance of FedAVG w/ mal decreases significantly, with a larger standard deviation and increased instability. In contrast, FedAVG w/ block maintains robust performance, with only slightly lower results compared to FedAVG w/o mal.

\subsubsection{Analysis of Token Distributions}  
Fig.~\ref{fig:token_distribution_sp=8} depicts the average tokens remaining in honest and malicious proposers during the FL training process when $\gamma_p = 8$. We observe that honest proposers gradually accumulate more tokens while malicious proposers own fewer tokens as training progresses. Eventually, most malicious proposers lose the eligibility to participate in staking and are removed from the FL system, as their remaining tokens are insufficient. Looking more closely at the case where few malicious proposers are in the system (Fig.~\ref{fig:tokens_with_rate_0.1}), we note that proposals are initially accepted until a revert and slash step is performed around epoch 32 (indicated by the sudden drop in tokens for malicious proposers). As more malicious proposers are in the system (Fig.~\ref{fig:tokens_with_rate_0.4}), revert and slash steps occur more frequently. This is further highlighted by Fig.~\ref{fig:award_slash_round_with_rate_0.1} and~\ref{fig:award_slash_round_with_rate_0.4}, which demonstrates for each epoch if it corresponds to an award or a revert and slash step. Fig.~\ref{fig:award_slash_num_with_rate_0.1} and~\ref{fig:award_slash_num_with_rate_0.4} further provide the cumulative sum over the number of award and slash epochs, illustrating that the majority of the epochs consist of award epochs and that the fraction of revert and slash episodes increases with the rising fraction of malicious proposers. Finally, we depict the average tokens remaining in honest and malicious proposers under various configurations of $\gamma_p$ in Fig.~\ref{fig:token_distribution_malicious} and Fig.~\ref{fig:token_distribution_honest}. We observe the same phenomenon as Fig.~\ref{fig:token_distribution_sp=8}, which aligns with the expectations of our system design and reinforces the effectiveness of our approach.

\subsubsection{Survival Analysis of Clients} As shown in Fig.~\ref{fig:Malicious_client_liveness}, the anticipated survival time of malicious proposers experiences a decrease as $\gamma_p$ increases. This effect can be attributed to the incentive mechanism in place, whereby a higher value of $\gamma_p$ results in a greater penalty for proposers who act maliciously. 
Fig.~\ref{fig:Honest_client_liveness} shows the survival time of honest proposers under different values of $\gamma_p$ and exhibits noteworthy behavior. In cases where the malicious ratio $\eta$ is high, the expected survival time of honest proposers may decrease with a large $\gamma_p$. This is due to the fact that, in each epoch, all randomly selected proposers will be slashed if the performance of the aggregated global model does not show improvement. Therefore, it is worth noting that balancing the token slashing parameter $\gamma_p$ is crucial, because setting an excessively high value can harm honest proposers, whereas a small value can lead to slow convergence (see~Fig.~\ref{fig:Global_accuracy_sp}).

\subsubsection{Sensitivity to Malicious Client Ratio} 
The results presented in Fig.~\ref{fig:training_res} demonstrate the robustness of our proposed method, FedAVG w/ block, against different malicious client ratios, as its performance remains unaffected even under large $\eta$ values. However, it is important to note that the malicious client ratio can impact the token distribution and survival time of clients. Specifically, when there are more malicious clients present in the system, honest clients tend to accumulate more assets on average (\cf~Fig.~\ref{fig:normal_node_tokens_with_rate_0.1} - \ref{fig:normal_node_tokens_with_rate_0.4}). Nevertheless, they also face a higher risk of being slashed during an epoch, which can ultimately shorten their survival time (\cf~Fig.~\ref{fig:honest_client_liveness_with_rate_0.1} - \ref{fig:honest_client_liveness_with_rate_0.4}). 

\subsubsection{Limitations} In this work, as the experimental results aim to evaluate the robustness of the proposed framework, several practical challenges are simplified, \eg~staleness~\cite{dai2019toward}, storage, and privacy~\cite{abadi2016deep}. Further, the proposed method requires more computational power than traditional methods due to mining (blockchain computing) and voting. Finally, large models have gained in popularity in practical applications, \eg~ViT~\cite{dosovitskiy2021image} and GPT-3~\cite{brown2020language}. This raises the question of how to efficiently handle on-chain aggregation for large models. Future work thus will aim to address these limitations to facilitate the research and development of FL with blockchain.

\subsection{Cost Analysis}
In our experiments, we adopt the Kaggle Lending Club dataset\footnote{\url{https://www.kaggle.com/datasets/wordsforthewise/lending-club}} to simulate a realistic financial application scenario. The total client number is $K = 50$. For each client, the generated model update file is with a size of $587$KB. Therefore, the communication cost incurred by a proposer or a voter is $587$KB and the storage cost on the blockchain is $587 \times 50 $KB $= 28.66$MB.

\section{Additional Experiments}
\label{sec:add}
In Sec.~\ref{sec:exp}, we examine the performance of the proposed system on a simple learning task on IID data. In this section, we further evaluate the robustness of the proposed system with a more complex task on non-IID data.

\begin{figure*}[t]
\centering
\hfill
\subfigure[$\eta = 0.1$]{
    \includegraphics[width=0.473\columnwidth]{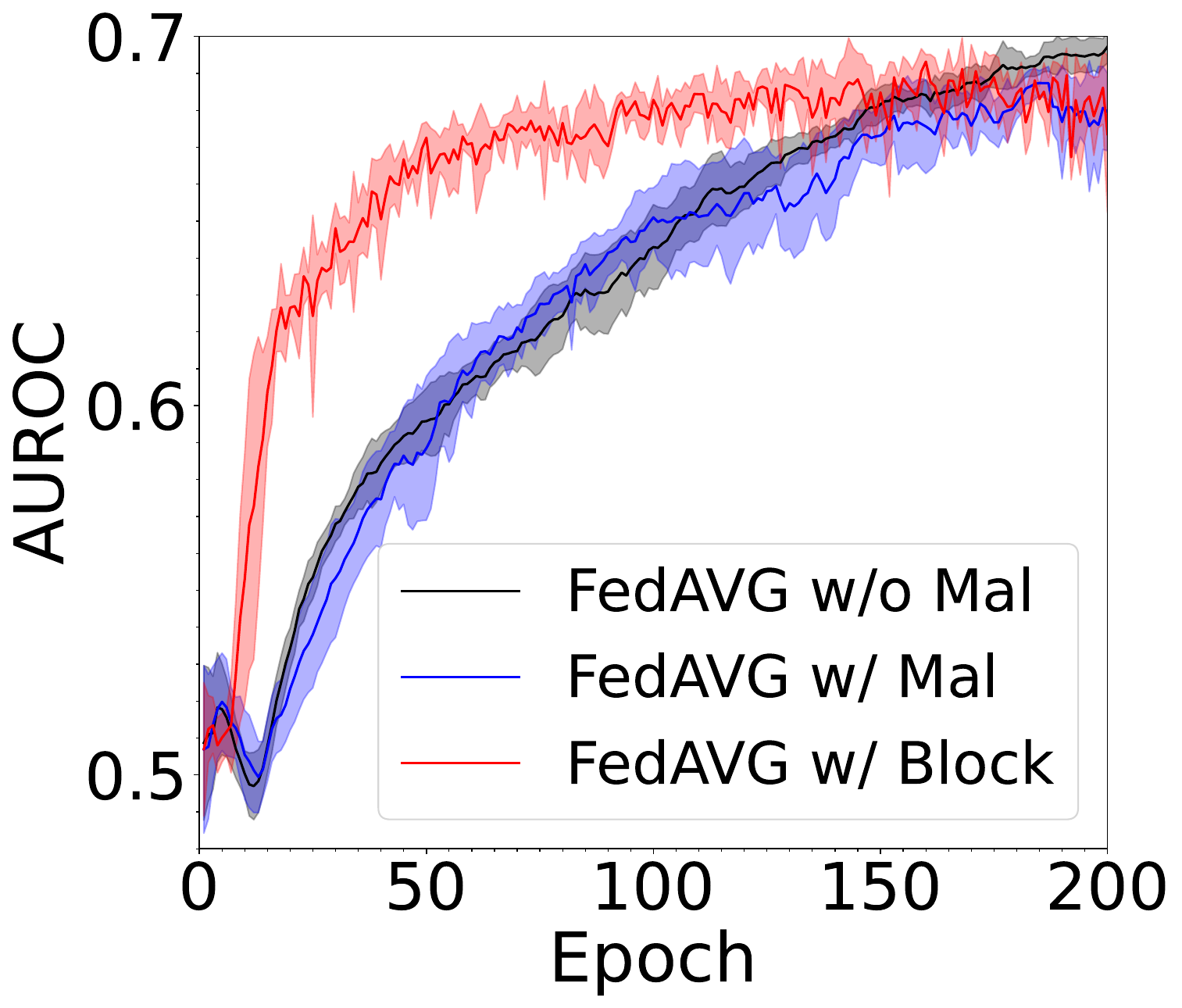}
}
\hfill
\subfigure[$\eta = 0.2$]{
    \includegraphics[width=0.473\columnwidth]{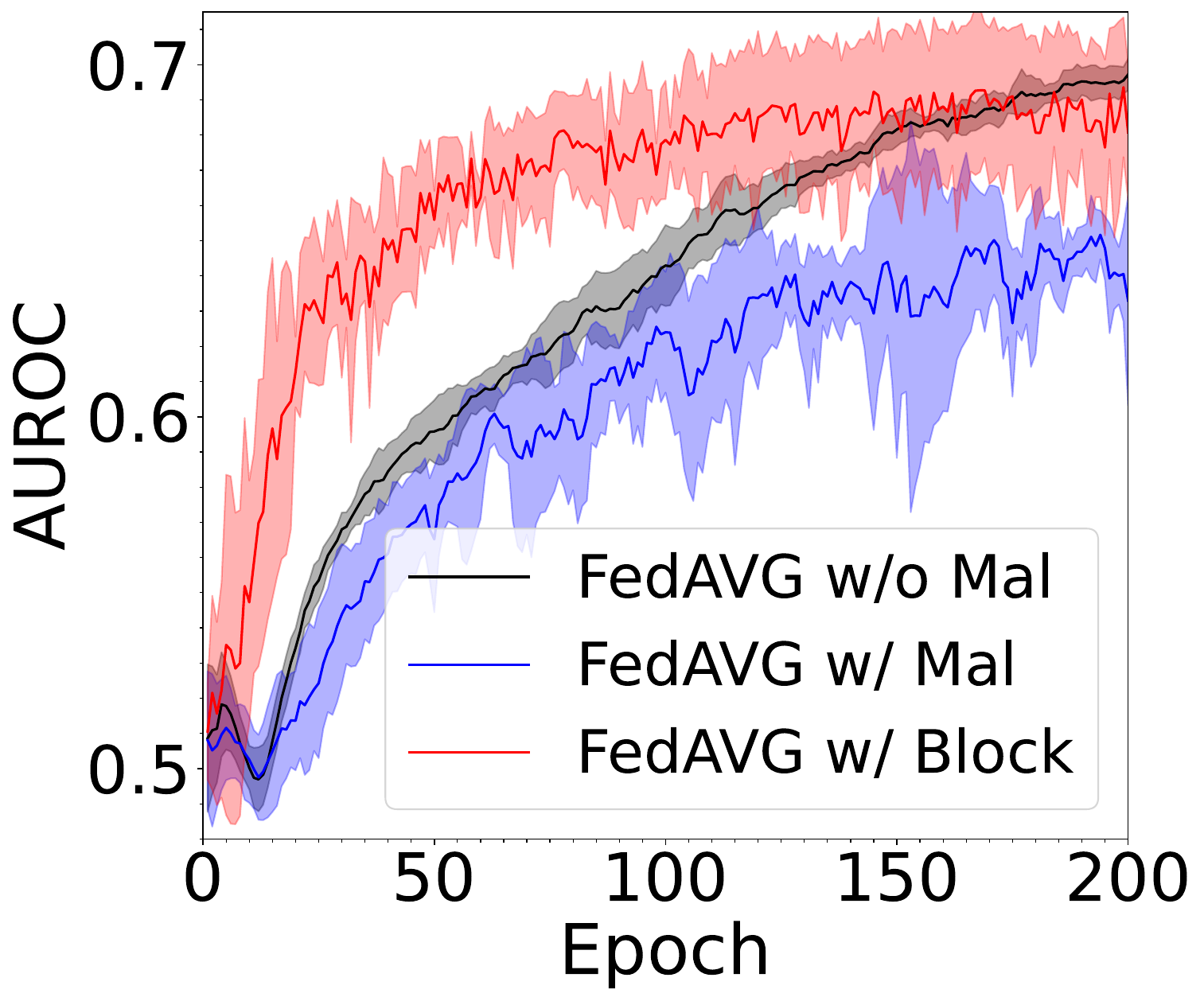}
}
\hfill
\subfigure[$\eta = 0.3$]{
    \includegraphics[width=0.473\columnwidth]{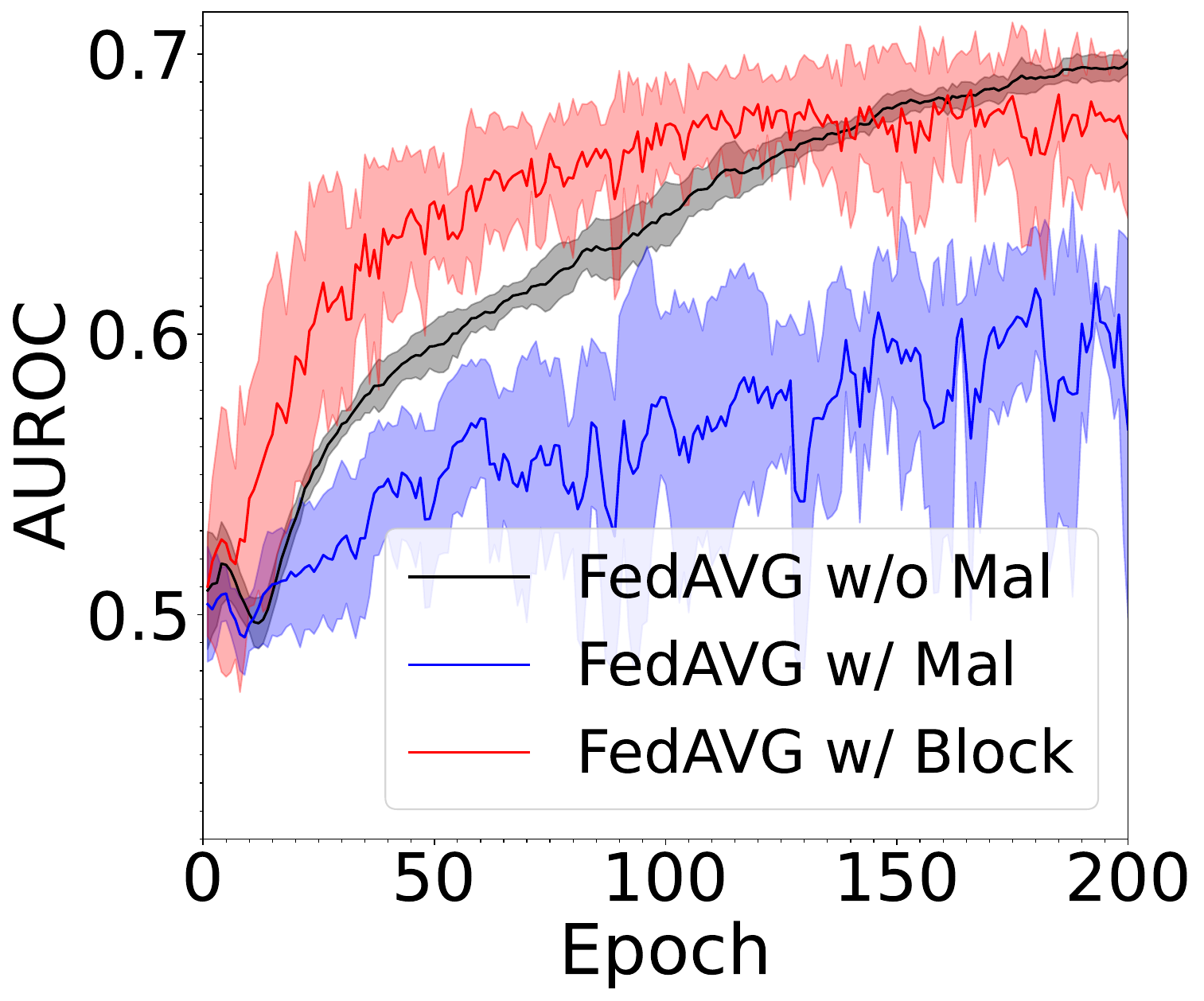}
}
\hfill
\subfigure[$\eta = 0.4$]{
    \includegraphics[width=0.473\columnwidth]{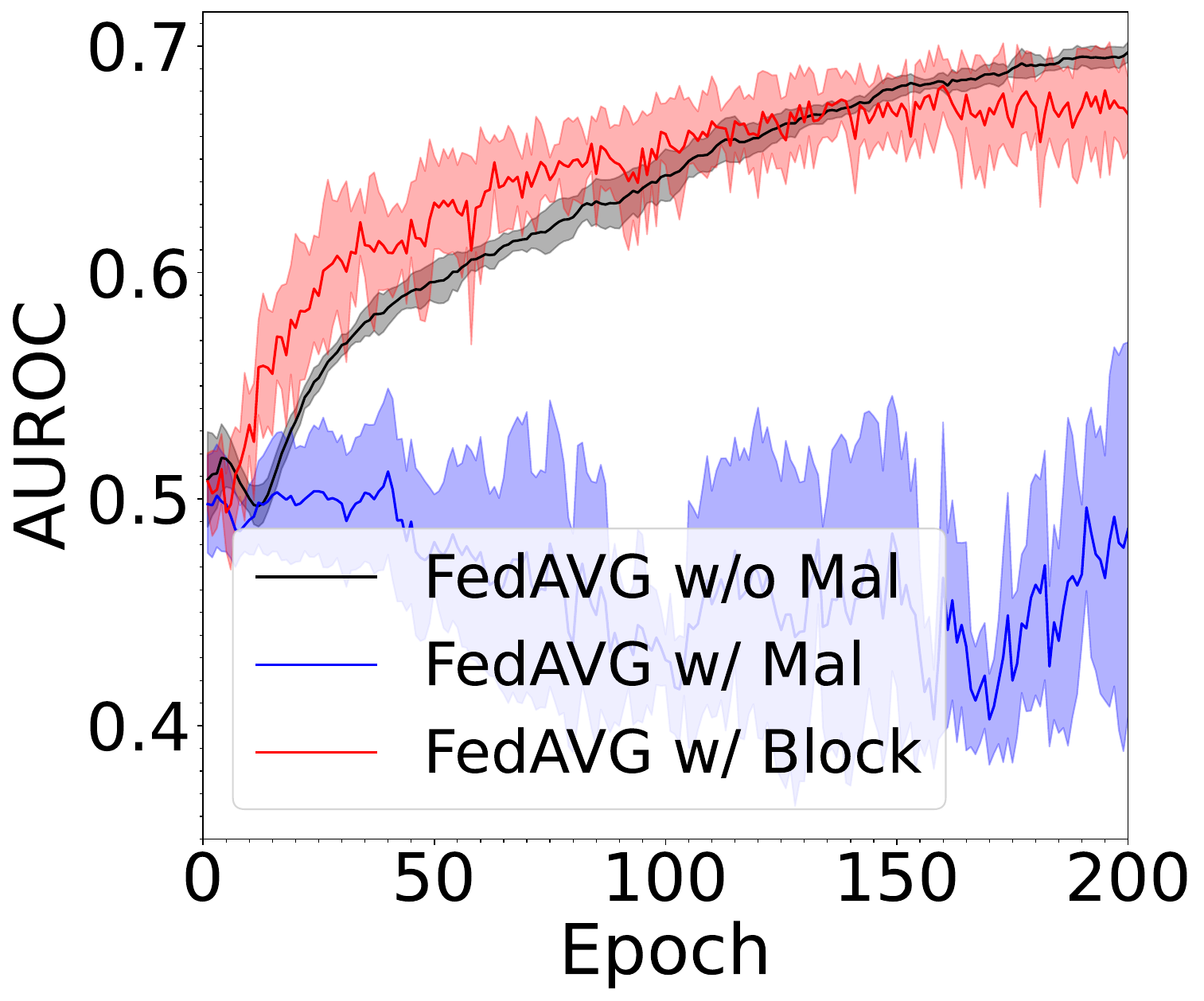}
}
\caption{Federated training under different values of ratio of malicious clients ($\eta$). The solid lines are the mean AUROCs and the shaded regions are 1 standard deviation around the means.}
\label{fig:med:train}
\end{figure*}

\begin{table}[t]
\centering
\caption{Performance comparison under different values of ratio of malicious clients ($\eta$). The reported numbers are mean and standard deviation under 5 random seeds.}
\resizebox{1.0\columnwidth}{!}{
\begin{tabular}{l|c|c|c|c}\hline
    Model & $\eta = 0.1$ & $\eta = 0.2$ & $\eta = 0.3$ & $\eta = 0.4$ \\\hline
    FA w/ M &$0.677\pm0.014$ &$0.640\pm0.012$ &$0.594\pm0.030$ &$0.483\pm0.080$ \\\hline
    FA w/ B &$0.683\pm0.008$ &$0.688\pm0.024$ &$0.676\pm0.023$ &$0.673\pm0.021$ \\\hline\hline
    FA w/o M &$0.695\pm0.004$ &$0.695\pm0.004$ &$0.695\pm0.004$ &$0.695\pm0.004$ \\\hline
    \emph{Oracle} &$0.733\pm0.014$ &$0.733\pm0.014$ &$0.733\pm0.014$ & $0.733\pm0.014$\\\hline
\end{tabular}
}
\label{tab:med:mean}
\end{table}

\subsection{Experimental Setup}
\subsubsection{Data and Task} 
We consider a standard multi-label classification (MLC) task~\cite{rajpurkar2017chexnet} to simulate a realistic clinical application scenario. We use the ChestX-ray14\footnote{\url{https://nihcc.app.box.com/v/ChestXray-NIHCC}} dataset~\cite{wang2017chestx} and leverage the first $6{\times}10^4$ chest X-ray images (CXRs) as our non-IID dataset. We use $80\%$ of the data as the training set and use the rest of the data as the test set. The training set is split into $K$ subsets of equal size and distributed across $K$ clients in a non-IID fashion. Within each client, $20\%$ of local data are randomly selected as the validation set. Because ChestX-ray14 has a long-tailed label distribution, we choose the 10 most common diseases to ensure that each client can contain labels for all diseases of interest.

\subsubsection{Implementation} There are $K = 50$ clients in the system and each client is initialized with 64 tokens. Following~\cite{rajpurkar2017chexnet}, we use DenseNet121~\cite{huang2017densely} as the network backbone. We use a standard Adam~\cite{kingma2015adam} optimizer with fixed learning rate $10^{-3}$ and batch size 256. We process each CXR with instance normalization~\cite{dai2018scan} and no data augmentation is applied. We use the mean area under the receiver operating characteristic curve over the 10 diseases as both the local validation score and evaluation metric. We set $\epsilon = 0.05$ and $\gamma = 32$ based on empirical experience. The rest of the implementation details follow Sec.~\ref{sec:exp:imp}.

\subsubsection{Baselines} We consider the same four baselines as in Sec.~\ref{sec:exp:base}.

\subsection{Results}

We run each baseline with 5 random seeds and report both mean and standard deviation under different values of $\eta$. The training results are visualized in Fig.~\ref{fig:med:train}.

\subsubsection{Performance Comparison}  
FedAVG w/ block shows competitive performance with FedAVG w/o mal (\ie~$\eta = 0$) and outperforms FedAVG w/ mal consistently. In addition to the learning curves in Fig.~\ref{fig:med:train}, we also report the mean AUROCs for all 4 methods after they fully converge in Tab.~\ref{tab:med:mean}. When $\eta$ increases, the performance of FedAVG w/ mal drops significantly and becomes more unstable (\ie~larger standard deviation). FedAVG w/ block remains robust performance and is only slightly lower than FedAVG w/o mal.

\subsubsection{Convergence Analysis} 
It can be shown that FedAVG w/ block converges faster than both FedAVG w/o mal and FedAVG w/ mal under various values of $\eta$. We hypothesize that the proposed global aggregation mechanism can facilitate federated optimization. Intuitively, this can be explained with gradient descent. FedAVG averages gradients optimized for different directions at different clients, which might not be an optimal global gradient. Under malicious attacks, gradients from malicious clients are intentionally optimized away from the optimal direction, which slows down the training process of FedAVG. However, the proposed consensus mechanism mitigates this issue as it only aggregates when consensus is achieved.

\begin{figure}[t]
    \centering
    \includegraphics[width=\columnwidth]{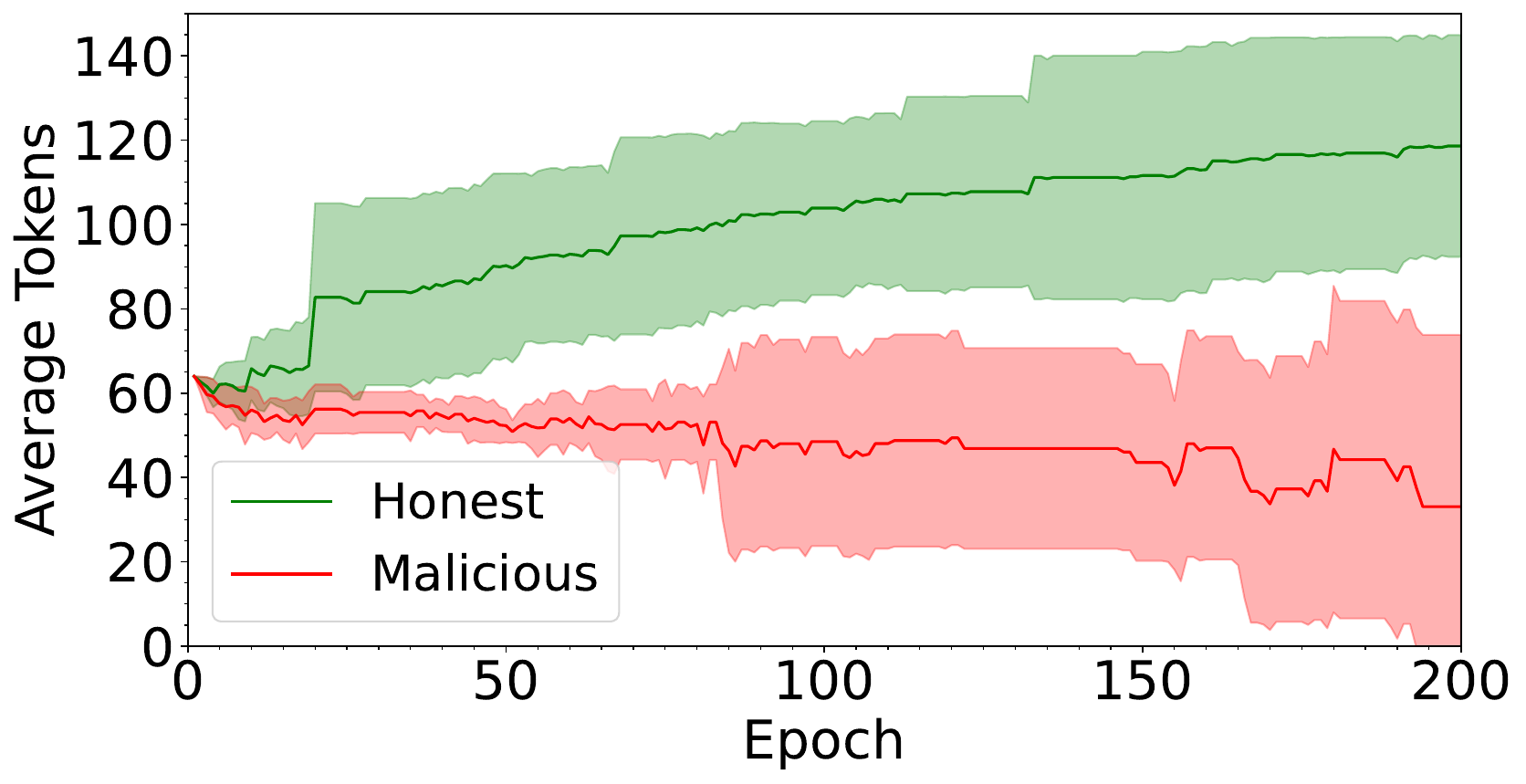}
    \caption{Tokens of honest and malicious clients when $\eta = 0.3$ and $\gamma = 32$.}
    \label{fig:med:token}
\end{figure}

\subsubsection{Token Analysis} 
As we use tokens to filter out malicious clients, we plot the average tokens left in honest and malicious clients (\eg~Fig.~\ref{fig:med:token}). After enough training epochs, honest clients will have more tokens and malicious clients will have fewer tokens. At the end of training, almost all malicious clients do not have enough tokens to stake, \ie~they are removed from the FL system. It is worth mentioning that $\gamma$ has only trivial effect on the learning performance but large $\gamma$ can overkill honest clients and small $\gamma$ can cause slow convergence.

\subsubsection{Impact of Non-IID Data}
Due to the non-IID nature of the medical task, the task setup in this section is more complex than the binary classification task in the previous section.
In contrast to Sec.~\ref{sec:exp}, there are two important findings. First, the proposed method is robust under the non-IID setup. Second, surprisingly, while the task is more difficult, the performance gain between the proposed method and the baselines becomes larger.

\section{Conclusion}
\label{sec:con}
In this work, we explore an under-explored research direction, namely using FL and blockchain to defend against poisoning attacks. The defense mechanism is twofold. We use on-chain smart contracts to replace the traditional central server and propose a stake-based majority voting mechanism to detect client-side malicious behaviors. We not only provide a solution to the problem of interest, but also show the robustness of the proposed method and provide the first empirical understanding of the problem. Last but not least, the results of this work suggest that the integration of FL and blockchain is an emerging solution to trustworthy ML. We believe that blockchain can not only play an important role in decentralization and the incentivization of participants for real-world FL applications in fields such as finance and medicine, but also can be leveraged to defend against poisoning attacks.

 \section*{Acknowledgment}
The authors would like to thank Shuoying Zhang from FLock.io and Shuhao Zheng from the School of Computer Science, McGill University for the helpful discussion. The authors would also like to thank Yizhe Wen and Xinyang Wang from FLock.io for the experimental simulation.

\bibliographystyle{IEEEtran}
\bibliography{ref}

\end{document}